\documentclass[conference]{IEEEtran}

 \PassOptionsToPackage{numbers, compress}{natbib}
\usepackage[pdftex]{graphicx}
\usepackage[utf8]{inputenc} %
\usepackage[T1]{fontenc}    %
\usepackage{hyperref}       %
\usepackage{url}            %
\usepackage{booktabs}       %
\usepackage{amsfonts}       %
\usepackage{nicefrac}       %
\usepackage{microtype}      %
\usepackage[table]{xcolor}
\usepackage{float}
\usepackage[listings,skins,breakable]{tcolorbox}
\definecolor{aliceblue}{rgb}{0.94, 0.97, 1.0}

\usepackage{amsmath}
\usepackage{amssymb}
\usepackage{mathtools}
\usepackage{amsthm}
\usepackage{nccmath}

\usepackage[capitalize,noabbrev]{cleveref}
\usepackage{natbib}

\usepackage[noend]{algorithmic}
\usepackage{algorithm}
\usepackage{microtype}      %

\usepackage{bm}
\usepackage{textcomp}
\usepackage{wrapfig,lipsum}

\usepackage{textcomp}
\usepackage{multirow}
\usepackage{subcaption}
\usepackage{siunitx}
\usepackage{chemformula}
\usepackage{collcell}
\usepackage{supertabular}
\usepackage{makecell}
\usepackage{color}
\usepackage{enumerate}

\usepackage{pythonhighlight}
\newcommand{\revise}[1]{{\color{black}#1}}

  {\color{black}}

\usepackage{xspace}

\usepackage[capitalize,noabbrev]{cleveref}
\Crefname{equation}{Eq.}{Eqs.}

\usepackage{adjustbox}

\usepackage{amsmath,amsfonts,bm}

\def\eqref#1{equation~\ref{#1}}

\def\1{\bm{1}}

\DeclareMathAlphabet{\mathsfit}{\encodingdefault}{\sfdefault}{m}{sl}
\SetMathAlphabet{\mathsfit}{bold}{\encodingdefault}{\sfdefault}{bx}{n}

\theoremstyle{plain}
\newtheorem{theorem}{Theorem}

\newtheorem{lemma}{Lemma}

\theoremstyle{definition}
\newtheorem{definition}{Definition}
\newtheorem{assumption}{Assumption}
\theoremstyle{remark}

\usepackage{wrapfig}
\usepackage{xspace}
\usepackage{enumitem}
\newcommand{\name}{\texttt{VIM}\xspace}

\NewDocumentCommand{\chulin}{ mO{} }{\textcolor{brown}{\textsuperscript{\textit{Chulin}}\textsf{\textbf{\small[#1]}}}}

\newcommand{\ouradmm}{\texttt{VIMADMM}\xspace}
\newcommand{\linearadmm}{\texttt{Linear-ADMM}\xspace}
\newcommand{\oursgd}{\texttt{Split Learning}\xspace}
\newcommand{\vafl}{\texttt{VAFL}\xspace}
\newcommand{\fdml}{\texttt{FDML}\xspace}
\newcommand{\fedbcd}{\texttt{FedBCD}\xspace}
\newcommand{\ouradmmjoint}{\texttt{VIMADMM-J}\xspace}

\newcommand{\vflpbm}{\texttt{VFL-PBM}\xspace}

\usepackage{amssymb}%
\usepackage{pifont}%
\newcommand{\cmark}{\textcolor{green}{\checkmark}}%
\newcommand{\xmark}{\textcolor{red}{$\times$}}%

\newcommand{\alibi}{\texttt{ALIBI}\xspace}

\newcommand{\mnist}{MNIST\xspace}
\newcommand{\cifar}{CIFAR\xspace}
\newcommand{\nus}{NUS-WIDE\xspace}
\newcommand{\modelnet}{ModelNet40\xspace}

\usepackage{titletoc}
\usepackage{thmtools} 
\usepackage{thm-restate}

\newcommand{\add}[1]{{\textcolor{black}{#1}}}

\usepackage[disable,textsize=tiny]{todonotes}

\title{Improving Privacy-Preserving Vertical Federated Learning by Efficient Communication with ADMM}

\makeatletter
    \newcommand{\linebreakand}{%
      \end{@IEEEauthorhalign}
      \hfill\mbox{}\par
      \mbox{}\hfill\begin{@IEEEauthorhalign}
    }
    \makeatother

\author{\IEEEauthorblockN{Chulin Xie}
\IEEEauthorblockA{
\textit{University of Illinois Urbana-Champaign}
\\
chulinx2@illinois.edu}
\and
\IEEEauthorblockN{Pin-Yu Chen}
\IEEEauthorblockA{
\textit{IBM Research}\\
pin-yu.chen@ibm.com
}
\and
\IEEEauthorblockN{Qinbin Li}
\IEEEauthorblockA{ 
\textit{UC Berkeley}\\
qinbin@berkeley.edu}
\linebreakand %
\IEEEauthorblockN{Arash Nourian}
\IEEEauthorblockA{ 
\textit{Amazon Web Services \& UC Berkeley}\\
nouriara@amazon.com}
\and
\IEEEauthorblockN{Ce Zhang}
\IEEEauthorblockA{
\textit{University of Chicago}\\
cez@uchicago.edu}
\and
\IEEEauthorblockN{Bo Li}
\IEEEauthorblockA{
\textit{University of Chicago \& UIUC}\\
lbo@illinois.edu}
}

\vspace{-10mm}

\begin{document}
\abovedisplayskip=1pt
\abovedisplayshortskip=1pt
\belowdisplayskip=1pt
\belowdisplayshortskip=1pt

\maketitle
\thispagestyle{plain}
\pagestyle{plain}

\begin{abstract}

Federated learning (FL) enables distributed resource-constrained devices to jointly train shared models while keeping the training data local for privacy purposes. 
Vertical FL (VFL), which allows each client to collect partial features, has attracted intensive research efforts recently. 
We identified the main challenges that existing VFL frameworks are facing: the server needs to communicate \textit{gradients} with the clients for \textit{each} training step, incurring high communication cost that leads to  rapid consumption of privacy budgets. To address these challenges, in this paper, we introduce a VFL framework with multiple heads (\name), which takes the separate contribution of each client into account, and enables an efficient decomposition of the VFL optimization objective to sub-objectives that can be iteratively tackled by the server and the clients \textit{on their own}.
In particular, we propose an Alternating Direction Method of Multipliers (ADMM)-based method to solve our optimization problem, \add{which allows clients to conduct \textit{multiple local updates} before communication, and thus reduces the communication cost and leads to better performance under differential privacy (DP).}
We provide the \revise{client-level} DP mechanism for our framework to protect user privacy.
Moreover, we show that a byproduct of \name is that the weights of learned heads reflect the importance of local clients. We conduct extensive evaluations and show that on four vertical FL datasets, \name achieves significantly higher performance and faster convergence compared with the state-of-the-art. We also explicitly evaluate the importance of local clients and show that \name enables functionalities such as client-level explanation and client denoising. We hope this work will shed light on a new way of effective VFL training and understanding. \footnote{
Our code is available at: \href{https://github.com/AI-secure/VFL-ADMM}{https://github.com/AI-secure/VFL-ADMM} }

\end{abstract}

\section{Introduction}
Federated learning (FL) has enabled large-scale training with data privacy guarantees on distributed data for different applications~\cite{yang2019ffd, brisimi2018federated,hard2018federated,yang2018applied,yang2019federated}. In general, FL can be categorized into Horizontal FL (HFL)~\cite{mcmahan2016communication} where data samples are distributed across \revise{clients}, and Vertical FL (VFL)~\cite{yang2019federated} where features of the samples are  partitioned across \revise{clients} and the labels are usually owned by the server (or the active party in two-party setting~\cite{hardy2017private}).
In particular, VFL allows \revise{clients} with partial information of the same dataset to jointly train the model, which leads to many real-world applications~\cite{hu2019fdml,yang2019federated,hard2018federated}. For instance, a patient may go to different types of \revise{healthcare providers}, such as dental clinics and pharmacies for different purposes, and therefore it is important for different  \revise{healthcare providers 
 (i.e., VFL clients/data owners/organizations)} to ``share" their information about the same patient \revise{(i.e., partial features of the same sample)} to better model the health condition of the patient. In addition, nowadays multimodal data has been ubiquitous, while usually, each \revise{client} is only able to collect one or a few data modalities due to resource limitations. Therefore, VFL provides an effective way to allow such \revise{clients} to train a model leveraging information from different data modalities jointly.

\add{Despite the importance and practicality of VFL, the state-of-the-art (SOTA) VFL frameworks suffer from notable weaknesses:
since the clients own the local features and the server holds the whole labels, 
the server needs to calculate training loss based on the labels and then send \textit{gradients} to clients for \textit{each} training step to update their local models~\cite{vepakomma2018split,chen2020vafl,kang2020fedmvt}, which incurs high communication cost and leads to potential rapid consumption of the privacy budget.}

\begin{table*}[h]
\centering
\caption{\small Comparison between our work and existing VFL studies.}
\vspace{-2mm}
\label{tb:compare_related_work}
\resizebox{0.92\linewidth}{!}{%
\begin{tabular}{lccccccccccc} \midrule
VFL Setup & Method & Support DNN & Support $N>2$ parties &Labels only held by one party   & Support multiple local updates &  Privacy guarantee \\ \midrule
\multirow{6}{*}{w/ model splitting} &\vafl~\cite{chen2020vafl}, \revise{\vflpbm~\cite{tran2023privacy}} &   \cmark &  \cmark & \cmark & \xmark & \cmark \\
& \oursgd{}~\cite{vepakomma2018split} &  \cmark &  \cmark & \cmark  &  \xmark  &  \xmark \\
& \texttt{FedBCD}~\cite{liu2022fedbcd}  & \cmark & \cmark & \cmark & \cmark &  \xmark \\
& \texttt{CELU-VFL} & \cmark& \xmark & \cmark & \cmark & \xmark \\
& \texttt{Flex-VFL}~\cite{castiglia2023flexible} & \cmark&  \cmark & \xmark &  \cmark & \xmark \\
\rowcolor{aliceblue} & {\ouradmm (Ours)} & \cmark & \cmark & \cmark & \cmark  & \cmark \\  \midrule
\multirow{4}{*}{w/o model splitting} & Fu et al.~\cite{fu2022usenix}, \fdml~\cite{hu2019fdml} & \cmark & \cmark & \xmark  &  \xmark &  \xmark{} \\
&\texttt{AdaVFL}~\cite{zhang2022adaptive}, \texttt{CAFE}~\cite{jin2021catastrophic} & \cmark& \cmark & \xmark & \cmark & \xmark  \\ 
&\linearadmm~\cite{hu2019learning} & \xmark& \cmark& \cmark & \cmark & \cmark \\ 
\rowcolor{aliceblue} &  {\ouradmmjoint (Ours)} &  \cmark & \cmark   & \cmark  & \cmark & \cmark \\
\bottomrule
\end{tabular}
}
\vspace{-4mm}
\end{table*}

To solve the above challenges, in this work, we propose an efficient VFL optimization framework with multiple heads (\name), where each head corresponds to one local client. \add{\name takes the individual contribution of clients into consideration and facilitates a thorough decomposition of the VFL optimization problem into multiple subproblems that can be iteratively solved by the server and the clients.
In particular, we propose an Alternating Direction Method of Multipliers (ADMM)~\cite{boyd2011distributed}-based method  that splits the overall \name optimization objective into smaller sub-objectives, and the clients can conduct multiple local updates w.r.t their local objectives at each communication round with the coordination of ADMM-related variables.   
This leads to faster model convergence and significantly reduces the communication cost, which is crucial to preserve privacy because the privacy cost of clients increases when the number of communication rounds increases~\cite{abadi2016deep,brendan2018learning}, due to the continuous transmission of sensitive local information. }
We consider two typical VFL settings: \textit{with model splitting} (i.e., clients host partial models) and \textit{without model splitting} (i.e., clients hold the entire model). 
Under with model splitting setting, we propose an ADMM-based algorithm \ouradmm{} under \name framework. Compared to gradient-based methods, \ouradmm{} not only reduces communication frequency but also reduces the dimensionality by only exchanging ADMM-related variables. We provide convergence analysis for \ouradmm and prove that it can converge to stationary points with mild assumptions.
With modifications of communication strategies and updating rules for servers and clients, we extend \ouradmm{} to the without model splitting setting and introduce \ouradmmjoint{}.
Under both settings, to further protect the privacy of the local features held by clients, we introduce privacy mechanisms that clip and perturb local outputs to satisfy \textit{client-level} differential privacy (DP)~\cite{dwork2006our, dwork2011firm, dwork2014algorithmic,mcmahan2016communication} and prove the DP guarantees.
Moreover, we offer a basic solution to separately protect the privacy of labels owned by server, leveraging the established label-DP mechanism \alibi ~\cite{malek2021antipodes} that perturbs the labels.
Finally, we show that a byproduct of \name is that the weights of learned heads reflect the importance of local clients, which enables functionalities such as client-level explanation, client denoising, and client summarization.  Our main \underline{contributions} are: 
\begin{itemize}[noitemsep,leftmargin=*]
\item We propose an efficient and effective VFL optimization framework with multiple heads (\name). To solve our optimization problem, we propose an ADMM-based method, \ouradmm, which reduces communication costs by allowing multiple local updates at each step. 
\item  We theoretically analyze the convergence of \ouradmm and prove that it can converge to stationary points.
\item We introduce the \revise{client-level} DP mechanism for our \name framework and prove its privacy guarantees. 
\item We conduct extensive experiments on four diverse datasets (i.e., \mnist, \cifar, \nus, and \modelnet), and show that ADMM-based algorithms under \name converge faster, achieve higher accuracy, and remain higher utility under \revise{client-level} DP and label DP than four existing VFL frameworks. 
\item We evaluate our client-level explanation under \name based on the weights norm of the heads, and demonstrate the functionalities it enables such as clients denoising and summarization. 
\end{itemize}

\vspace{-3mm}
\section{Related Work}
\label{sec:related_work}
\paragraph{Vertical Federated Learning}
VFL has been well studied for simple models including trees~\cite{cheng2021secureboost,wu2020privacy}, kernel models~\cite{gu2020federated}, and linear and logistic regression~\cite{hardy2017private,yang2019parallel,zhang2021secure, feng2020multi, hu2019learning, liu2019communication}. 
For instance, Hardy et al. \cite{hardy2017private} propose secure logistic regression for two-party VFL with homomorphic encryption~\cite{rouhani2018deepsecure, gilad2016cryptonets} and multiparty computation~\cite{ben1988completeness,bonawitz2017practical}.
However, a limitation of these methods is the performance constraint associated with the logistic regression. Subsequent research has expanded the scope of VFL to encompass Deep Neural Networks (DNNs), facilitating VFL training with a larger number of clients and on large-scale models and datasets.
For DNNs, there are two popular VFL settings: with model splitting~\cite{vepakomma2018split,kang2020fedmvt,chen2020vafl} and without model splitting~\cite{hu2019fdml,jin2021catastrophic}. 

In the with model splitting setting, \oursgd~\cite{vepakomma2018split} is the first related paradigm,
where each client trains a partial network up to a cut layer, the server concatenates local activations and trains the rest of the network.
\vafl~\cite{chen2020vafl} is proposed for asynchronous VFL where the server averages the local embeddings and sends gradients back to clients to update local models. However, such embedding averaging might lose the unique properties of each client. FedMVT~\cite{kang2020fedmvt} focuses on the semi-supervised VFL with multi-view learning.  C-VFL~\cite{castiglia2022compressed} proposes embedding compression techniques to improve communication efficiency. 
However, we note that these methods ~\cite{vepakomma2018split,chen2020vafl,kang2020fedmvt,castiglia2022compressed} still require the communication of gradients (w.r.t embeddings) from server to the client at \textit{each} training step, leading to high communication frequency and communication cost before convergence.  
Recent research efforts have sought to reduce VFL communication frequency by allowing clients to make multiple local updates at each round. Particularly, in  \fedbcd~\cite{liu2022fedbcd},
after obtaining gradients from the server, clients update local models using the same stale gradients for multiple steps. Building upon this, \texttt{CELU-VFL}~\cite{celuvfl2022} enhances the performance of \fedbcd  by caching stale gradients from earlier rounds and reusing them to estimate better model gradients at current round.  
Nonetheless, it is limited to supporting only two clients (party A and B, with B holding the labels) and cannot be directly extended to scenarios with more than two parties, as our study considers (specifically, it lacks a design to aggregate information from more parties).
On another note, \texttt{Flex-VFL}~\cite{castiglia2023flexible} allows each party to undergo a different number of local updates constrained by a set timeout for every round. Yet, it assumes that clients possess copies of labels and receive local embeddings from other clients, enabling them to compute local gradients independently for multi-step local updates. In contrast, we propose an ADMM-based framework that enables multiple local updates and assumes that only the server possesses labels,  which cannot be shared with other clients due to privacy restriction~\cite{fu2022usenix}.

For VFL without model splitting setting, each client submits local logits to the server, who then averages over the logits and send gradients w.r.t logits back to clients, as detailed in Fu et al.~\cite{fu2022usenix}. 
Several other approaches assume that the server shares both labels and aggregated logits with the clients, enabling them to locally compute the gradient~\cite{hu2019fdml,zhang2022adaptive}.  \fdml~\cite{hu2019fdml} performs one step of local update at each round for asynchronous and distributed SGD.  
Considering that certain clients might have slower local computation speeds, 
\texttt{AdaVFL}~\cite{zhang2022adaptive}  optimizes the number of local updates  for each client at each round to minimize overall time. 
Meanwhile, \texttt{CAFE}~\cite{jin2021catastrophic} directly applies FedAvg~\cite{mcmahan2016communication} from  Horizonta FL to VFL where all clients possess the labels  and  can exchange the model parameters with others for model aggregation. This deviates from the standard VFL setup where only the server retains the label and local models cannot be shared owing to privacy implications~\cite{fu2022usenix}.

\paragraph{Differentially Private VFL} 
In existing VFL frameworks, \vafl~\cite{chen2020vafl} provides Gaussian DP guarantee~\cite{dong2019gaussian} \revise{and \vflpbm~\cite{tran2023privacy} quantizes local embeddings into DP integer vectors. However, they do not calculate the exact privacy budget in the evaluation.}
\fdml~\cite{hu2019fdml} evaluate their framework under different levels of empirical noises, yet without offering detailed DP mechanisms or DP guarantee. 
The ADMM-based linear VFL framework (abbreviated to \linearadmm)~\cite{hu2019learning} provides $(\epsilon,\delta)$-DP guarantee for linear models by calculating the closed-form sensitivity of each sample and perturbing the linear model parameters,  which is not directly applicable to DNNs whose sensitivity is hard to estimate due to the nonconvexity. 
Instead, we propose to perturb local outputs and provide formal \revise{client-level} $(\epsilon,\delta)$-DP theoretical guarantee in \cref{sec:dpvim}. 

We provide an overall comparison between our work and existing studies in \cref{tb:compare_related_work}.

\vspace{-2mm}
\section{VFL with Multiple Heads (\name) }\label{sec:vimmethod}
\vspace{-1mm}
In this section, we start with the VFL background in \cref{sec:probformulation}, and then discuss VFL with model splitting setting and
introduce our framework \name and  ADMM-based method \ouradmm in \cref{sec:vfl-with-model-split}. Finally, we show that our ADMM-based method can be easily extended to  VFL without model splitting setting with slight modifications on communication strategies and update rules, yielding  \ouradmmjoint{} in \cref{sec:vfl-without-model-split}.

\vspace{-1mm}
\subsection{VFL Background}\label{sec:probformulation}
\vspace{-1mm}
Typically in VFL, there are $M$ clients who hold \textit{different feature sets of the same training samples} and jointly train the machine learning models.  We consider the classification task and denote $d_c$  as the number of classes. Suppose there is a training dataset $ D =\{ x_j, y_j \}_{j=1}^{N}$ containing $N$ samples, the server owns the labels $\{y_j\}_{j=1}^{N}$, and each client $k$ has a local feature set $X_k = \{ x_j^k\}_{j=1}^{N}$, where the vector $x_j^k \in \mathbb{R}^{ d^k}$ denotes the local (partial) features of sample $j$. The overall feature $x_j \in \mathbb{R}^{ d} $ of sample $j$ is the concatenation of all local features $\{ x_j^1,  x_j^2, \dots,  x_j^M \}$, with $d = \sum_{k=1}^{M}d^k $. 

Due to the privacy protection requirement of VFL, each client $k$ does not share raw local feature set $X_k$ with other clients or the server. Instead, VFL consists of two steps: (1) {\em local processing step}: each client learns a local model that maps the local features to local outputs and sends them to the server. 
(2) {\em server aggregation step}: the server aggregates the local outputs from all clients to compute the final prediction for each sample as well as the corresponding losses. Depending on whether or not the server holds a model, there are two popular VFL settings~\cite{fu2022usenix}: VFL \emph{with model splitting}~\cite{chen2020vafl,vepakomma2018split} and VFL \emph{without model splitting}~\cite{hu2019fdml}:
(i) In the model splitting setting,  each client trains a feature extractor as the local model that outputs \textit{local embeddings}, and the server owns a model which predicts the final results based on the aggregated embeddings. 
(ii)  In the VFL without model splitting setting, the clients host the entire model that outputs the \textit{local logits}, and the server simply performs the logits aggregation operation without hosting any model.  

In both settings, the local model is updated based on SGD with federated backward propagation~\cite{fu2022usenix}: a) server first computes the gradients w.r.t the local output (either embeddings or logits) from each client separately and sends the gradients back to clients; b) each client calculates the gradients of local output w.r.t the local model parameters and updates the local model using the chain rule.

\vspace{-1mm}
\subsection{VFL with Model Splitting}\label{sec:vfl-with-model-split}
\vspace{-1mm}

\paragraph{Setup} 
Let $f$ parameterized by $\theta_k$ be the local model (i.e., feature extractor) of client $k$, which outputs a local embedding vector $h_j^k = f(x_j^k;\theta_k)$ $\in \mathbb{R}^{ d_f} $ for each local feature $x_j^k$. 
We denote the parameters of the model on the server-side as $\theta_0$. Overall, the clients and the server aim to collaboratively solve the  Empirical Risk Minimization (ERM) objective:
\begin{align}\label{eq:ms-obj}
\medmath{\underset{\{\theta_k\} , \theta_0 }{\operatorname{min}} \frac{1}{N} \sum_{j=1}^N  \ell ( \{ h_j^1, \ldots, h_j^M \}, y_j; \theta_0) +  \sum_{k=1}^M \beta_k  \mathcal{R}(\theta_k)  + \beta  \mathcal{R}(\theta_0)} 
\end{align} 
where $\ell$ is  a loss function  (e.g., cross-entropy loss with softmax), $\mathcal{R}$ is a regularizer on model parameters, and $\beta_k \in \mathbb{R}$ is the regularization weight
for client $k$, and $\beta$ is the weight for server. 
The local embeddings for each sample $j$ can be either concatenated together $h_j=[ h_j^1, \ldots, h_j^M]$ as in \oursgd~\cite{vepakomma2018split} or averaged $h_j=\sum_{k=1}^M  \alpha_k h_j^k$ with aggregation weights $\alpha_k \in \mathbb{R}$ as in \vafl~\cite{chen2020vafl}. Then $h_j$ is used as the input for server model $\theta_0$ to calculate the loss.  
For more detailed description of the training algorithm \oursgd under VFL with
model splitting, please refer to \cref{algo:splitlearn} in \cref{app:algo_splitlearn}.

However, as outlined in \cref{sec:probformulation}, these VFL methods are based on SGD and depend on the server model $\theta_0$ to complete the loss and gradient calculation using server labels for updating local models $\{\theta_k\}$. Consequently, the server needs to send the gradient w.r.t embeddings back to clients at \emph{every} training step of the local models.  Such (1) frequent communication and (2) high dimensionality of gradients (i.e., $b d_f $ for $b$ samples) lead to high communication costs.

\paragraph{VIM Formulation}

To address these challenges, we propose the \name framework where the server learns a model with multiple heads corresponding to multiple local clients. It takes the separate contribution of each client into account and facilitates the breakdown of the VFL optimization into several sub-problems to be solved by clients and the server independently via ADMM without communicating gradients, as we will elaborate on later.
Specifically, the server's model $\theta_0$ consists of $M$ heads $W_1,W_2, \dots, W_M$ where $W_k \in  \mathbb{R}^{ d_f \times d_c }, k\in [M] $.
For the sake of simplicity, we consider each $W_k$ to be a linear head here, and our formulation can be easily extended to the non-linear heads by viewing each $W_k$ as a non-linear model (see the end of \cref{sec:vfl-with-model-split} for more details). 
This is motivated by the recent studies in representation learning, which have shown that learning a linear classifier is sufficient to accurately predicting the labels on top of embedding representations~\cite{radford2021learning,khosla2020supervised}, given the expressive power of the local feature extractor that captures essential information from raw feature sets.
For sample $j$, the server's model outputs $\hat{y}_j = \sum_{k=1}^M h_{j}^k  W_k $ as the prediction,  yielding our \name objective:

\begin{align}\label{eq:multihead-obj}
 \underset{\{W_k\},\{\theta_k\}}{\operatorname{min}}
  \mathcal{L}_{\mathrm{\name}} & ( \{W_k\},\{\theta_k\})  :=  \frac{1}{N}\sum_{j=1}^N   \ell\left( \sum_{k=1}^M f(x_j^k;\theta_k) W_k , y_j\right)  \nonumber \\ 
 &+  \sum_{k=1}^M \beta_k   \mathcal{R}_{k}(\theta_k) +  \sum_{k=1}^M \beta_k  \mathcal{R}_{k}(W_k) 
\end{align}

\paragraph{VIMADMM} \add{Based on the \name formulation, we propose an ADMM-based method, reducing the communication frequency by allowing the clients to perform multiple local updates w.r.t their  ADMM objectives at each round}, and reducing the dimensionality by only exchanging ADMM-related variables (i.e., $(2 b  + d_f)d_c$ for $b$ samples  where $d_c  \ll d_f,  b$ for most VFL settings today~\cite{chen2020vafl,hu2019fdml}).
Specifically, we note that Eq.~\ref{eq:multihead-obj} can be viewed as the \textit{sharing problem} \cite{boyd2011distributed} involving each client adjusting its variable
to minimize the \textit{shared} cost term $\ell( \sum_{k=1}^M h_{j}^k  W_k , y_j)$ as well as  its \textit{individual} cost $\mathcal{R}(\theta_k) +\mathcal{R}(W_k)$.
Moreover, the multiple heads in \name enable the application of ADMM via a special decomposition into simpler sub-problems that can be solved in a distributed manner. We begin by rewriting Eq.~\ref{eq:multihead-obj} to an 
equivalent constrained optimization problem by introducing auxiliary variables $z_1,z_2, \dots , z_N \in  \mathbb{R}^{   d_c } $:
\begin{align}\label{eq:admm-obj}
 &\underset{\{W_k\},\{\theta_k\},\{z_j\} }{\operatorname{min}}
   \frac{1}{N}  \sum_{j=1}^N   \ell (z_j,y_j)  + \sum_{k=1}^M \beta_k   \mathcal{R}_{k}(\theta_k) +  \sum_{k=1}^M \beta_k  \mathcal{R}_{k}(W_k) \nonumber \\  
 & \quad \text { s.t. }  \sum_{k=1}^M f(x_j^k;\theta_k) W_{k} - z_j =0 , \forall j \in [N].
\end{align}
Notably, 
each constraint implies a consensus between the server's output $ \sum_{k=1}^M h_{j}^k  W_{k}$ and the auxiliary variable $z_j$ for each sample $j$.
The augmented Lagrangian, which adds a quadratic term to the Lagrangian of Eq.~\ref{eq:admm-obj}, is given by:
\begin{align}\label{eq:admm-lagrang}
& \underset{ \{W_k\}, \{\theta_k\},\{z_j\},\{\lambda_j\} }{\operatorname{min}} \mathcal{L}_{\mathrm{ADMM}} ( \{W_k\}, \{\theta_k\},\{z_j\},\{\lambda_j\} )\nonumber  \\
& \quad\quad :=  \frac{1}{N} \sum_{j=1}^N  \ell (z_j,y_j) + \sum_{k=1}^M \beta_k   \mathcal{R}_{k}(\theta_k) +  \sum_{k=1}^M \beta_k  \mathcal{R}_{k}(W_k) \nonumber  \\   
 & \quad\quad + \frac{1}{N} \sum_{j=1}^N  \lambda_{j}^\top  (   \sum_{k=1}^M f(x_j^k;\theta_k)  W_{k} - z_j ) \nonumber  \\
  & \quad\quad + \frac{\rho}{2N}  \sum_{j=1}^N \left\|  \sum_{k=1}^M f(x_j^k;\theta_k)  W_{k} - z_j\right\|_F^2,  
\end{align}
where $\lambda_{j}  \in \mathbb{R}^{ d_c }$ is the dual variable for sample $j$, and $\rho  \in \mathbb{R}^+$ is a constant penalty factor. Recall that  $\hat{y}_j= \sum_{k=1}^M f(x^k_j, \theta_k) W_k$ is the  server output (i.e., prediction) for sample $x_j$.   \ouradmm essentially aims to minimize the loss between $z_j$ and ground-truth label $y_j$, as well as the difference between $z_j$ and $\hat{y}_j$ during training. Specifically, as shown in the ADMM loss (Eq.~\ref{eq:admm-lagrang}),  $l(z_j, y_j)$ is the loss between $z_j$ and $y_j$, while $\hat{y}_j - z_j  = \sum_{k=1}^M f(x^k_j, \theta_k) W_k - z_j$ appears in the linear constraint and quadratic constraint terms. The auxiliary variables $\{z_j\}$ and dual variables $\{\lambda_{j}\}$ are used to facilitate the training of server heads $\{W_k\}$ and local models $ \{\theta_k\}$.

To solve Eq.~\ref{eq:admm-lagrang}, we follow standard ADMM~\cite{boyd2011distributed} and update the primal variables $\{W_k\}$ , $\{\theta_k\}$, $\{z_j\}$ and the dual variables $\{\lambda_j\}$  \emph{alternatively},
which decomposes the problem in Eq.~\ref{eq:admm-obj} into four sets of sub-problems over $\{W_k\}$, $\{\theta_k\}$, $\{z_j\}$, $\{\lambda_j\}$, and the parameters in each sub-problem can be solved \textit{in parallel}. 
In practice, we propose the following strategy for the alternative updating in the server and clients:
(i) updating $\{z_j\}$, $\{\lambda_j\}$ and $\{W_k\}$ at server-side, (ii) updating $\{\theta_k\}$ at the client-side in parallel. Moreover, we consider the realistic setting of stochastic ADMM with mini-batches.  
Concretely, at communication round $t$, 
the server samples a set of data indices, $B(t)$, with batch size $|B(t)|= b $. 
Then we describe the key steps of \ouradmm as follows:

\textbf{(1) \emph{Communication from client to server.}} Each client $k$ sends a batch of embeddings $\{{ h_j^k}^{(t)}\}_{j \in B(t)}$ to the server, where ${ h_j^k}^{(t)}=f(x_j^k; \theta_k^{(t)}) $ $, \forall j \in B(t)$. 

\textbf{(2) \emph{Sever updates auxiliary variables $\{z_j\}$}.}  After receiving the local embeddings from all clients, the server updates the auxiliary variable for each sample $j \in B(t)$  as:
\begin{equation}\label{eq:update_z}
z_{j}^{(t)} =  \underset{z_j}{\operatorname{argmin}}  \quad \ell (z_j,y_j) -     {  \lambda_{j}^{(t-1)}}^\top   z_{j}  +\frac{\rho}{2}\left\| \sum_{k=1}^M  {h_{j}^k}^{(t)}  W_k^{(t)}- z_j \right\|_{F}^{2}.
\end{equation}
Since the optimization problem in Eq.~\ref{eq:update_z} is convex and differentiable with respect to $z_j$, we use the L-BFGS-B algorithm~\cite{zhu1997algorithm} to solve the minimization problem. 

\textbf{(3) \emph{Sever updates dual variables $\{\lambda_j\}$}.} The server updates dual variable  for each sample $j  \in B(t)$:
\begin{equation}\label{eq:update_lambda}
\lambda^{(t)}_{j} =  \lambda^{(t-1)}_{j}+\rho\left(\sum_{k=1}^M {h_{j}^k}^{(t)} W_k^{(t)} -z^{(t)}_{j}\right).
\end{equation}

\textbf{(4) \emph{Sever updates the heads $\{W_k\}$}.} Each head  $W_k, \forall k\in [M] $ of the server is then updated:  
\begin{align} \label{eq:update_W}
&W_{k}^{(t +1)}= \underset{W_k}{\operatorname{argmin}}  \quad \beta_k   \mathcal{R}_k(W_k) +   \frac{1}{b} \sum_{ j \in B(t)} {\lambda_{j}^{(t)}}^\top {h_{j}^k}^{(t)}  W_{k}   \nonumber\\ 
&  + \sum\limits_{ j \in B(t)} \frac{\rho}{2b} \left\|  \sum\limits_{i\in [M] , i \neq k}  {h_{j}^i}^{(t)}  {W_{i}}^{(t)}  +   {h_{j}^k}^{(t)}  W_{k} - {z_j}^{(t)}\right\|_F^2. 
\end{align}
For squared $\ell_2$ regularizer $\mathcal{R}$, we can solve $W_{k}^{(t +1)}$ in an inexact way to save the computation by \textit{one} step of SGD with the objective of Eq.~\ref{eq:update_W}.

\textbf{(5) \emph{Communication from server to client.} }After the updates in Eq.~\ref{eq:update_W}, we define a residual variable ${s_j^{k}}^{(t+1)} $ for each sample $ j \in B(t)$ of $k$-th client, which provides supervision for updating local model:
\begin{equation}\label{eq:def_residual}
{s_j^{k}}^{(t)} \triangleq {z_j}^{(t)} - \sum_{i\in [M] , i \neq k}  {h_{j}^i}^{(t)} {W_{i}}^{(t+1)}
\end{equation} 
The server sends the dual variables $  \{ \lambda^{(t)}_{j} \}_{j \in B(t)}$ and the residual variables $ \{  {s_j^{k}}^{(t)}\}_{j \in B(t)} $ of all samples, as well as the \textit{corresponding} head $ W_{k}^{(t+1)}$ to each client $k$.

\textbf{(6) \emph{Client updates local model parameters $\theta_k$}.} 
Finally, every client $k$  locally updates the model  parameters  $\theta_k$ as follows:
\begin{align}\label{eq:update_theta}
\theta_{k}^{(t +1)}= & \underset{\theta_{k}}{\operatorname{argmin}}  \quad \beta_k  \mathcal{R}_k(\theta_k) + \frac{1}{b}  \sum_{ j \in B(t)}    {\lambda_{j}^{(t)}}^\top    f(x_j^k; \theta_k )  { W_{k}^{(t+1)}  } 
\nonumber \\
& + \frac{\rho}{2b}  \sum_{ j \in B(t)}   \left\|   {s_j^{k}}^{(t)} -   f(x_j^k; \theta_k ) { W_{k}^{(t+1)} }   \right\|_F^2.
\end{align}
Due to the nonconvexity of the loss function of DNNs, 
we use $\tau$ local steps of SGD 
to update the local model at each round with the objective of Eq.~\ref{eq:update_theta}. 
We note that multiple local updates of Eq.~\ref{eq:update_theta} enabled by ADMM lead to better local models at each communication round compared to gradient-based methods, thus \ouradmm{} requires fewer communication rounds to converge as we will show in \cref{sec:evaluation}. 
These six steps of \ouradmm are summarized in Algorithm~\ref{algo:vimadmm}.

Note that ADMM auxiliary variables $\{z_{j}\}$ and dual variables $ \{\lambda_{j}\}$ are only used during the training phase to help update server heads and local models. Therefore, in the test phase, for any sample $x$, the server directly uses the trained multiple heads to make prediction $\hat{y} = \sum_{k=1}^M h^k  W_k$.  

\begin{small}
\setlength{\textfloatsep}{0pt}
\begin{algorithm}[t!]
 \caption{ \small \ouradmm{}  \colorbox[RGB]{239,240,241}{(
 with user-level differential privacy)}}
 \label{algo:vimadmm}
\begin{algorithmic}[1]
\small
 \STATE {\bfseries Input:}{number of communication rounds $T$, number of clients $M$, number of training samples $N$, batch size $b$ , input features $\{\{ x_j^1\}_{j=1}^{N}, \{ x_j^2\}_{j=1}^{N}, \ldots, \{ x_j^M\}_{j=1}^{N}\} $, the labels $\{y_j\}_{j=1}^{N}$, local model $\{\theta_k\}_{k=1}^M $; linear heads $\{W_k\}_{k=1}^M $;
auxiliary variables $ \{ z_{j} \}_{j=1}^{N} $;
dual variables $ \{ \lambda_{j} \}_{j=1}^{N} $;
\colorbox[RGB]{239,240,241}{ noise parameter $\sigma$, clipping constant $C$ }} 
\FOR {communication round $ t \in [T]$} 
    \STATE Server samples a set of data indices $B(t)$ with $|B(t)|= b $  
     \FOR {{client} $ k \in [M]$} 
        \STATE  \textbf{generates} a local training batch $ \{x_j^k\}_{j\in B(t) }$ 
        \STATE  \textbf{computes} local embeddings  $\medmath{\{{h_j^k}^{(t)}\gets f(x_j^k; \theta_k) \}_{j\in B(t)}}$ 
        \STATE   \colorbox[RGB]{239,240,241}{\textbf{clips and perturbs}  local embedding matrix }
         \STATE   \colorbox[RGB]{239,240,241}{$\medmath{ \{{h_j^k}^{(t)}\}_{j\in B(t) } \gets \mathtt{Clip}\left( \{{h_j^k}^{(t)}\}_{j\in B(t) }  , C\right) + \mathcal{N}\left(0, \sigma^{2} C^{2}\right)} $ }
        \STATE \textbf{sends} local embeddings $ \{{h_j^k}^{(t)}\}_{j\in B(t) }$ to the server
    \ENDFOR
    \STATE Server \textbf{updates} auxiliary variables $ \{ z_{j}^{(t )} \}_{j\in B(t)} $ by Eq.~\ref{eq:update_z} 
    \STATE Server \textbf{updates} dual variables $\{ \lambda^{(t)}_{j}\}_{j\in B(t)}$ by Eq.~\ref{eq:update_lambda} 
    \STATE Server \textbf{updates} linear heads $\{W_{k}^{(t +1)} \}_{k\in [M]}$ by  Eq.~\ref{eq:update_W} 
    \STATE  Server \textbf{computes} residual variables $ \{{s_j^{k}}^{(t)}\}_{j\in B(t), k\in [M]} $ by Eq.~\ref{eq:def_residual}
    \STATE Server \textbf{sends} $  \{ \lambda^{(t)}_{j} \}_{j \in B(t)}$ , $ \{  {s_j^{k}}^{(t)}\}_{j \in B(t)} $  and {corresponding} $ W_{k}^{(t+1)}$ to each client $k, \forall k\in [M]$ 
    \FOR {{client} $ k \in [M]$} 
        \STATE \textbf{updates} local model  $\theta_{k}^{(t +1)}$ for $\tau$ steps by Eq.~\ref{eq:update_theta}  via SGD 
    \ENDFOR
\ENDFOR
 \end{algorithmic}
 \vspace{-1mm}
\end{algorithm}
 \end{small}

\paragraph{Extending \ouradmm to multiple non-linear heads}
The server can learn non-linear transformation from the collected embeddings to uxiliary variables  $\{z_{j}\}$  by employing multiple non-linear heads. 
To achieve this, we rewrite all $ f(x^k_j, \theta_k) W_k$ as a more generalized form $ g( f(x^k_j, \theta_k), W_k)$ from Eq.~\ref{eq:multihead-obj} to Eq.~\ref{eq:update_theta}. Here,  $g$ can be a non-linear function parameterized by $W_k$.  Consequently, the prediction for each sample $j$ becomes $\hat{y}_j = \sum_{k=1}^M g( f(x^k_j, \theta_k), W_k)$. In this context, \ouradmm still aims to minimize the loss between $z_j$ and ground-truth label $y_j$, as well as the difference between $z_j$ and $\hat{y}_j$ during training in Eq.~\ref{eq:admm-lagrang}.

\vspace{-5mm}
\subsection{VFL without Model Splitting}\label{sec:vfl-without-model-split}
\vspace{-2mm}
\paragraph{Setup}
Recall the VFL without model splitting setting described in \textsection~\ref{sec:probformulation}.
Let $g$ parameterized by ${\psi_k}$ be the local model (i.e., whole model) of client $k$, which outputs local logits $o_j^k = g(x_j^k;{\psi_k}) \in \mathbb{R}^{ d_c} $ for each local feature $x_j^k$. 
The clients and the server aim to jointly solve the  problem
\begin{equation}
  \min_{\{{\psi_k}\}_{k=1}^M } 
\frac{1}{N}\sum_{j=1}^N  \ell ( \{o_j^1, \ldots, o_j^M \}  , y_j) + \beta_k \sum_{k=1}^M \mathcal{R}_k({\psi_k}) , \forall k\in [M]  
\end{equation}

\vspace{-1mm}
\paragraph{VIMADMM-J}
In exisiting VFL frameworks, the server averages the local logits as final prediction $\sum_{k=i}^M o_j^k$, but these methods also suffers from the high communication cost by sending the gradients w.r.t. local logits to each client at \textit{each} training step of the local model~\cite{fu2022usenix}. 
To solve this problem with our \name framework, we  adapt \ouradmm to the without model splitting setting and propose \ouradmmjoint{}, where each feature extractor $\theta_k$ and each head $W_k$ are held by the corresponding client $k$, and are always updated locally. 
The corresponding Algorithm~\ref{algo:vimadmmjoint} and detailed description are in  Appendix~\ref{app:algos}.

\vspace{-5mm}
\section{Convergence Analysis for \ouradmm}\label{sec:convergence}
\vspace{-5mm}
In this section, we provide the convergence guarantee for \ouradmm, which is non-trivial due to \add{the complexity of the \textit{alternative optimization}} between four sets of parameters $\{W_k\}, \{\theta_k\},\{z_j\},\{\lambda_j\}$.
To convey the salient ideas of convergence analysis, we consider full batch, i.e., $B(t)=[N]$ and use the exact minimization solutions during training (Eq.~\ref{eq:update_z},~\ref{eq:update_lambda},~\ref{eq:update_W}) following~\cite{hong2016convergence}.

We present our main results below and defer formal proofs to  \cref{app:conv-guatantee} due to space constraints. 
\begin{restatable}{theorem}{thmconv}
\label{theorem:main-paper-conv}
\add{Assume that $ \mathcal{L}_{\mathrm{\name}}$ is bounded from below, that is $\underline{e} := \min_{\{\theta_k\},  \{W_k\} }    \mathcal{L}_{\mathrm{\name}} (\{\theta_k\},  \{W_k\}) > - \infty$. 
Assume that  $\ell(z;\cdot)$ is $L$-Lipschitz smooth w.r.t $z$ and $\mathcal{L}_{\mathrm{ADMM}} $ loss is strongly convex w.r.t $\{z_j\},\{W_k\}, \{\theta_k\}$ with  constant $\mu_{z},\mu_{W},\mu_{\theta}$ respectively. Assume that the norm of $W_k$ is bounded $\|W_k\| \leq \sigma_{W}$, the local model $f(\cdot;\theta)$ has bounded gradient $\|\nabla f(\cdot;\theta) \|\leq L_{\theta}$ and  bounded output norm $\|f(\cdot; \theta)\| \leq \sigma_{\theta}$.
If \cref{algo:vimadmm} is run, and there exists a  $\rho$ satisfying 
$ \max\{L, \frac{2L^2}{ \mu_z }\} < \rho < \min\{\frac{\mu_{\theta}}{   L_{\theta}^{2} \sigma_{W}^2},\frac{\mu_{W}}{ \sigma_{\theta}^2 }  \}$, 
then we have the following:}\\
\textbf{(A)} $\mathcal{L}_{\mathrm{ADMM}} $ loss is monotonically decreasing and lower-bounded:
\begin{align}\label{eq:conv-mono-decrease}
& \mathcal{L}_{\mathrm{ADMM}} ( \{W_k^{(t+1)}\}, \{\theta_k^{(t+1)}\},\{z_j^{(t+1)}\},\{\lambda_j^{(t+1)}\} )  \nonumber \\
&< \mathcal{L}_{\mathrm{ADMM}} ( \{W_k^{(t)}\}, \{\theta_k^{(t)}\} ,\{z_j^{(t)}\},\{\lambda_j^{(t)}\})  
\end{align}
\begin{align}\label{eq:conv-lower-bonud}
\lim_{t\rightarrow \infty} \mathcal{L}_{\mathrm{ADMM}} ( \{W_k^{(t)}\}, \{\theta_k^{(t)}\} ,\{z_j^{(t)}\},\{\lambda_j^{(t)}\}) \geq \underline{e}
\end{align}
\textbf{(B)} Let $( \{W_k^{*}\}, \{\theta_k^{*}\} ,\{z_j^{*}\},\{\lambda_j^{*}\}) $ denote any limit points of the sequence $( \{W_k^{(t+1)}\}, \{\theta_k^{(t+1)}\},\{z_j^{(t+1)}\},\{\lambda_j^{(t+1)}\} )$ generated by \cref{algo:vimadmm}, then it is stationary:
\begin{align}\label{eq:conv-stationary}
 & z_{j}^{*} \in \arg \min_{z_j}   \ell \left( z_j; y_j \right) + {\lambda_j^*}^\top \left(  \sum_{k=1}^M f(x_j^k;\theta_k^{*})  W_{k}^{*} - z_j \right)   \text{\xspace and \xspace}  \nonumber  \\
 & \sum_{k=1}^M f(x_j^k;\theta_k^{*}) W_{k}^{*} = z_j^{*} ,   \forall j \in [N], \text{\xspace and \xspace}  \nonumber\\
 &\beta_k \nabla \mathcal{R}_k (W_k^*) +   \frac{1}{N} \sum_{j=1}^{N} \lambda_j^{* \top}  f(x_j^k;\theta_k^{*}) =0  \text{\xspace and \xspace} \nonumber \\
 & \beta_k \nabla \mathcal{R}_k (\theta_k^*) +   \frac{1}{N} \sum_{j=1}^{N} \lambda_j^{* \top} \nabla f(x_j^k;\theta_k^{*})  W_{k}^{*} =0,  k \in [M].
\end{align}
\end{restatable}

\begin{proof}[Proof Sketch]
We obtain \cref{eq:conv-mono-decrease} by breaking down the changes of loss $\mathcal{L}_{\mathrm{ADMM}}$ at each round $t$ into the alternatively updates of four components: $\{\lambda_j^{(t+1)}\}$, $\{z_j^{(t+1)}\}$, $\{W_k^{(t+1)}\}$, and $\{\theta_k^{(t+1)}\}$, respectively. Through our assumptions and the optimality of the minimizers, we demonstrate that the combined loss decreases at each round. 
Next, to derive \cref{eq:conv-lower-bonud}, we leverage  the Lipschitz continuity of $\ell$, the condition $ \rho \geq  L $,  the lower bound of $\mathcal{L}_{\mathrm{\name}}$, and the fact that the quadratic loss term in $\mathcal{L}_{\mathrm{ADMM}}$ is non-negative.
\add{Finally, by letting $t\rightarrow \infty$ and examining the optimality conditions of the minimizers, we drive \cref{eq:conv-stationary}}.
\end{proof}
\textit{Remark.} \cref{theorem:main-paper-conv} (A) shows that  \ouradmm converges, measured by the monotonically decreasing and convergent loss, and (B) establishes that any limit point is a stationary solution to the problem~\ref{eq:admm-lagrang}. 
Note that we make several assumptions in \cref{theorem:main-paper-conv} to derive the above guarantees, as often made in ADMM analysis~\cite{hong2016convergence} for alternative optimization of multiple sets of variables. 
Specifically, we follow Hong et al.~\cite{hong2016convergence} to assume convexity, Lipschitz smoothness, and the bounded loss for convergence analysis of \ouradmm. Furthermore, we acknowledge that analyzing the local model can be challenging, given the complexity of DNNs, so we introduce an additional assumption that bounds the norm of the gradient and the output of local models, which could be practical when the model training exhibits stability. Similarly, we assume a bounded norm for the server model. By incorporating these assumptions, we aim to offer a more comprehensive understanding of the convergence behavior of \ouradmm.

\section{\revise{client-level} Differentially Private \name}\label{sec:dpvim}

While the raw features and local models are kept locally without sharing in VFL, sharing the model outputs such as local embeddings or predictions during the training process might also leak sensitive client information~\cite{mahendran2015understanding, papernot2018sok}.
Therefore, we aim to further protect the privacy of the local feature set $X_k$ of each client $k$ against potential adversaries such as \revise{honest-but-curious server and clients, and external attackers.} 
\paragraph{\revise{Threat Model}}
\revise{ We consider different types of \textit{potential adversaries based on their capabilities}:
(1) Honest-but-curious server and clients: they follow the VFL protocol correctly but might try to infer private client information from information exchanged between the clients and server~\cite{tran2023privacy}.
(2) External attackers: they are not directly involved in the VFL process but may observe the predicted results from the server and the communicated information during training, trying to extract private client information.
Regarding \textit{attack scenarios}, these attackers may conduct membership inference attacks~\cite{shokri2017membership} to determine whether the data of a specific VFL client was included during training. 
Our goal is to protect the local data of each client against potential attackers so that the attacker cannot make significant inferences about any single client's data.
Next, we provide privacy-preserving mechanisms to satisfy client-level differential privacy (DP) guarantees. 
}
\paragraph{\revise{Client-level DP}}
We begin with the $(\epsilon,\delta)$-DP definition,  which guarantees that the change in a randomized algorithm's output distribution caused by an input difference is bounded. 
\begin{definition} [$(\epsilon,\delta)$-DP~\cite{dwork2014algorithmic}]
\label{def:dp}
A randomized algorithm $\mathcal{M}: \mathcal{X}^n \mapsto \Theta$ is $(\epsilon, \delta)$-DP if for every pair of neighboring datasets $X, X^{\prime} \in  \mathcal{X}^n$ (i.e., differing only by one sample), and every possible (measurable) output set $E \subseteq \Theta $ the following inequality holds: $\operatorname{Pr}[\mathcal{M}(X) \in E] \leq e^{\epsilon} \operatorname{Pr}\left[\mathcal{M}\left(X^{\prime}\right) \in E\right]+\delta$.
\end{definition}

Next, we introduce \revise{client-level} $(\epsilon,\delta)$-DP~\cite{mcmahan2018learning}, which guarantees that the algorithm's output would not be changed much by differing one \revise{client}.

\begin{definition} [\revise{Client-level} $(\epsilon,\delta)$-DP~\cite{mcmahan2018learning}]
\label{def:user-dp}
Let $X$ and $X'$ be adjacent datasets if they differ by
all samples associated with a single \revise{client}\footnote{\revise{We consider the ``zero-out'' notion for the neighboring dataset, following \cite{ponomareva2023dp}: datasets are adjacent if any one \revise{client}’s local data is replaced with the special “zero” data (exactly zero for numeric data). }}. 
The mechanism $\mathcal{M}$ satisfies \revise{client-level} $(\epsilon, \delta)$-DP if it meets  Definition~\ref{def:dp} with  $X$ and $X'$ as adjacent datasets.
\end{definition}

\revise{\textit{Remark.} (1) \textit{Client-level DP} protects the privacy of all local samples of each client~\cite{mcmahan2018learning}. 
The neighboring datasets in client-level DP are defined between client-adjacent datasets, denoted by $X=\{X_1, \ldots, X_k, \ldots, X_M \}$ and $X'=\{X_1, \ldots, X_k', \ldots, X_M \}$ for some client $k$. The algorithm's output should not change significantly if a single client's entire dataset is changed.
(2) \textit{User-level DP} is another prevalent privacy notion in FL literature, and its definition depends on how ``user'' is interpreted. If a ``user'' denotes a client/data owner in FL, then user-level DP aligns with client-level DP~\cite{geyer2017differentially,mcmahan2018learning,agarwal2018cpsgd}. Additionally, a ``user'' in VFL might refer to an entity contributing different samples with partial features, where $M$ VFL clients hold disjoint partial features $\{ x_j^1,  x_j^2, \dots,  x_j^M \}$ about the same user $j$~\cite{cohendifferentially, ranbaduge2022differentially}. For example, different healthcare providers (VFL clients such as dental clinics and pharmacies) can hold different features about the same patient (user). In such cases,  neighboring datasets are defined as those differing by all local samples associated with one user across all VFL client datasets. 
In this work, we focus on client-level DP due to its widespread adoption in FL~\cite{mcmahan2018learning}.
 }

Since the only shared information from clients is their local outputs, denoted as $\mathcal{A}_k$ for $k$-th client, we leverage the following DP mechanisms to perturb the local outputs of each client $k$ at every round $t$: (1)  \textit{clip} the whole local output matrix (either embeddings ${\mathcal{A}_k}^{(t)}= \{{h_j^k}^{(t)}\}_{j\in B(t) }$ or logits ${\mathcal{A}_k}^{(t)}= \{{o_j^k}^{(t)}\}_{j\in B(t) }$ ) with threshold $C$ such that the $\ell_2$-\textit{sensitivity for each \revise{client} is upper bounded} by $C$.
\revise{That is, $\mathtt{Clip}\left( \mathcal{A}_k  , C\right)= \mathcal{A}_k \cdot \min\left(1, \frac{C}{\| \mathcal{A}_k\|_F}\right) $ where $\|\cdot\|_F$ is the Frobenius norm\footnote{\revise{The Frobenius norm for a $m\times n$ matrix $A$ is $\|A\|_F=\sqrt{\sum_{i=1}^m \sum_{j=1}^n\left|a_{i j}\right|^2}$}}.}
(2) Then we add \revise{scalar} \textit{Gaussian noise} \revise{independently to each cell of the matrix. The noise is} sampled from $\mathcal{N}(0, \sigma^{2} C^{2})$, which is proportional to $C$ and can \textit{randomize the local output matrix of each \revise{client}}:
$ \mathcal{A}_k \gets \mathtt{Clip}\left( \mathcal{A}_k  , C\right) + \mathcal{N}\left(0, \sigma^{2} C^{2}\right)$. 
Based on the above modification to Algorithm~\ref{algo:vimadmm} and \ref{algo:vimadmmjoint},  we now provide their privacy guarantee in Theorem~\ref{theo:dp_guarantee}. 
\begin{restatable}{theorem}{thmdp}
\label{theo:dp_guarantee}
Given a total of $M$ clients, $T$ communication rounds, clipping threshold $C$ and noise level $\sigma$, DP versions of Algorithm~\ref{algo:vimadmm},~\ref{algo:vimadmmjoint}  satisfy \revise{client-level} $( \frac{T \alpha}{2\sigma^2} + \log \frac{\alpha -1 }{\alpha} - \frac{\log\delta + \log \alpha}{ \alpha -1}  , \delta)$-DP for any $\alpha > 1$ and $0 < \delta < 1$.
\end{restatable}
\textit{Proof Sketch}.
We derive the privacy guarantee using R\'enyi Differential Privacy (RDP)~\cite{mironov2017renyi} as a bridge. We first leverage the RDP guarantee for the Gaussian mechanism~\cite{mironov2017renyi} to analyze the privacy cost for one communication round under local output perturbation. Then we use RDP Composition property~\cite{mironov2017renyi} to accumulate the privacy costs over $T$ communication rounds. Finally, we convert \revise{client-level} RDP guarantee into \revise{client-level} DP guarantee~\cite{balle2020hypothesis}. Detailed proofs are deferred to \cref{app:dp}.

\textit{Remark.} 
\revise{
Since DP mechanisms (i.e., clipping and noise addition), are applied to each client's outputs (i.e., embedding or logits matrix) \textit{locally}, these local outputs satisfy \revise{client-level} local DP, protecting against privacy attacks from other \revise{clients}, server or external attackers. 
That is, by observing the local outputs matrix of one \revise{client}, other parties cannot determine the presence of that \revise{client}'s actual training data. 
The concatenated output matrix from all \revise{clients} satisfies the same client-level DP guarantee based on DP parallel composition \cite{mcsherry2009privacy}, due to non-overlapping nature of local data among clients}.

Note that the aforementioned DP mechanisms do not protect the privacy of labels held by server. Therefore,  we separately use state-of-the-art label DP mechanism~\cite{malek2021antipodes} to protects  server's label privacy via label perturbing, and conduct empirical evaluations of our method under label DP in \cref{sec:exp_labeldp}.

\vspace{-3mm}
\section{Experiments}\label{sec:exp}
\vspace{-1mm}

We conduct extensive experiments on four VFL datasets.
We show that our proposed framework \name achieves significantly faster convergence and higher accuracy than SOTA (\cref{sec:evaluation}), maintains higher utility under \revise{client-level} DP and label DP (\cref{sec:exp_dpvfl}), and enables client-level explainability (\cref{sec:exp_explainability}).

\subsubsection{Data and Models}
We consider classification tasks on four datasets: {\mnist}~\cite{lecun-mnisthandwrittendigit-2010},  {\cifar}~\cite{cifar}, multi-modality dataset {\nus} with image and textual features~\cite{chua2009nus}, 
and  multi-view dataset {\modelnet}~\cite{su2018deeper}. 
\begin{itemize}
    \item \mnist~\cite{lecun-mnisthandwrittendigit-2010} contains images with handwritten digits. We create the VFL scenario  by splitting the input features evenly by rows for 14 clients. We use a fully connected model of two linear layers with ReLU activations as the local model. 
    \item \cifar~\cite{cifar} contains colour images. We split each image into patches for 9 clients. We use a standard CNN architecture from the PyTorch library~\footnote{\href{https://github.com/pytorch/opacus}{https://github.com/pytorch/opacus}} as the local model. 
    \item \nus~\cite{chua2009nus} is a  multi-modality dataset with 634  low-level image features and 1000 textual tag features. We distribute image features to 2 clients (300 dim and 334 dim), and text features to 2 clients (500 dim and 500 dim). We use  a fully connected model of two linear layers with ReLU activations as the local model. 
    \item \modelnet~\cite{su2018deeper} is a multi-view image dataset, containing the shaded images from 12 views for the same objects. We use 4 views and distribute them to 4 clients respectively. We use  ResNet-18~\cite{he2016deep} as the local model. 
\end{itemize}

We split each dataset into the train, validation, and test sets. See \cref{tab:dataset} for more details about the number of samples and the number of classes for each dataset. 
{
\setlength{\tabcolsep}{4pt} %
\vspace{-2mm}
\begin{table}[h]
    \centering
    \renewrobustcmd{\bfseries}{\fontseries{b}\selectfont}
    \sisetup{detect-weight,mode=text,group-minimum-digits = 4}    
    \caption{\small Dataset description.}
    \label{tab:dataset}
\begin{subtable}{0.8\columnwidth}
    \centering
     \resizebox{\columnwidth}{!}{%
        \begin{tabular}{cccccccccccccccccc}
    \toprule
\multirow{2}{*}{Dataset} & 
\multirow{2}{*}{$\#$ features } & 
\multirow{1}{*}{$\#$ classes } & 
\multirow{1}{*}{$\#$ clients} & 
\multicolumn{3}{c}{$\#$ samples} \\
& &$d_c$   & $M$ &   train & validation & test    \\\midrule
\mnist & 28 $\times$ 28 & 10 & 14 & 54000 & 6000 & 10000 \\ \midrule
\cifar & 32 $\times$ 32 $\times$ 3  & 10 &  9 &  45000 &  5000 & 10000 \\ \midrule 
\nus &   1634   & 5 & 4 &  54000 & 6000 & 10000   \\ \midrule
\modelnet & 224 $\times$ 224 $\times$  3 $\times 12$   & 40 &   4 & 8877 & 966 & 2468  
\\ \bottomrule
\end{tabular}%
}
\end{subtable}
\end{table}
\vspace{-2mm}
}

{

\newlength{\tempdimab}
\settoheight{\tempdimab}{\includegraphics[width=.25\linewidth]{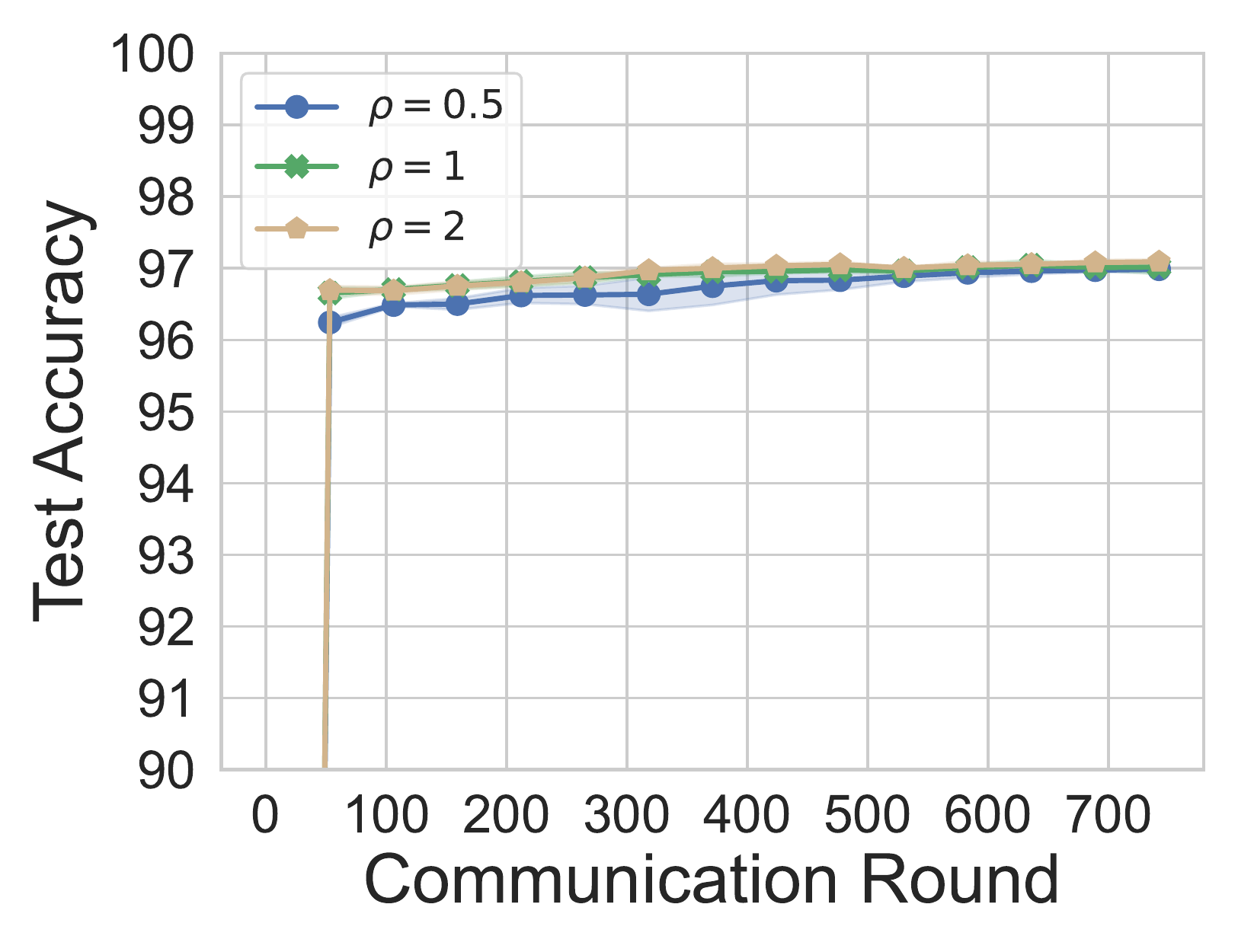}}%

\newlength{\cleanheightc}
\settoheight{\cleanheightc}{\includegraphics[width=.25\linewidth]{ICLRplots/acc/mnist_rho.pdf}}%

\newcommand{\rowname}[1]%
{\rotatebox{90}{\makebox[\tempdimab][c]{\scriptsize #1}}}

\renewcommand{\tabcolsep}{10pt}

\begin{figure*}[h]
\centering
\begin{subtable}{0.9\linewidth}
\centering
\resizebox{\columnwidth}{!}{%
\begin{tabular}{@{}p{2mm}@{}c@{}c@{}c@{}c@{}c@{}c@{}}
        & \makecell{\tiny{\mnist}}
        & \makecell{\tiny{\cifar}}
        & \makecell{\tiny{\nus}}
        & \makecell{\tiny{\modelnet}}
        \vspace{-3pt}\\
\rowname{\makecell{\tiny W/ Model Split}}&
\includegraphics[height=\cleanheightc]{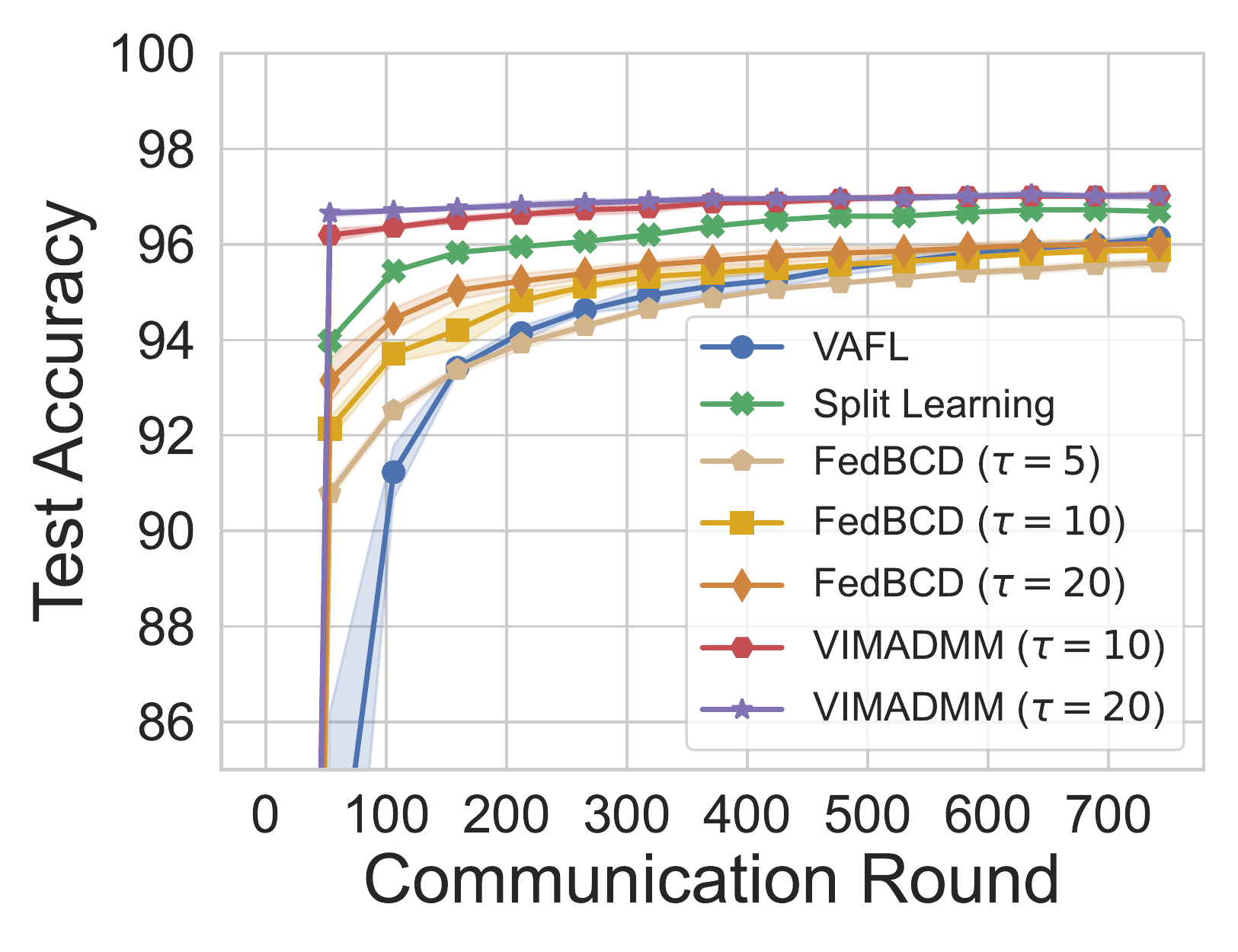}&
\includegraphics[height=\cleanheightc]{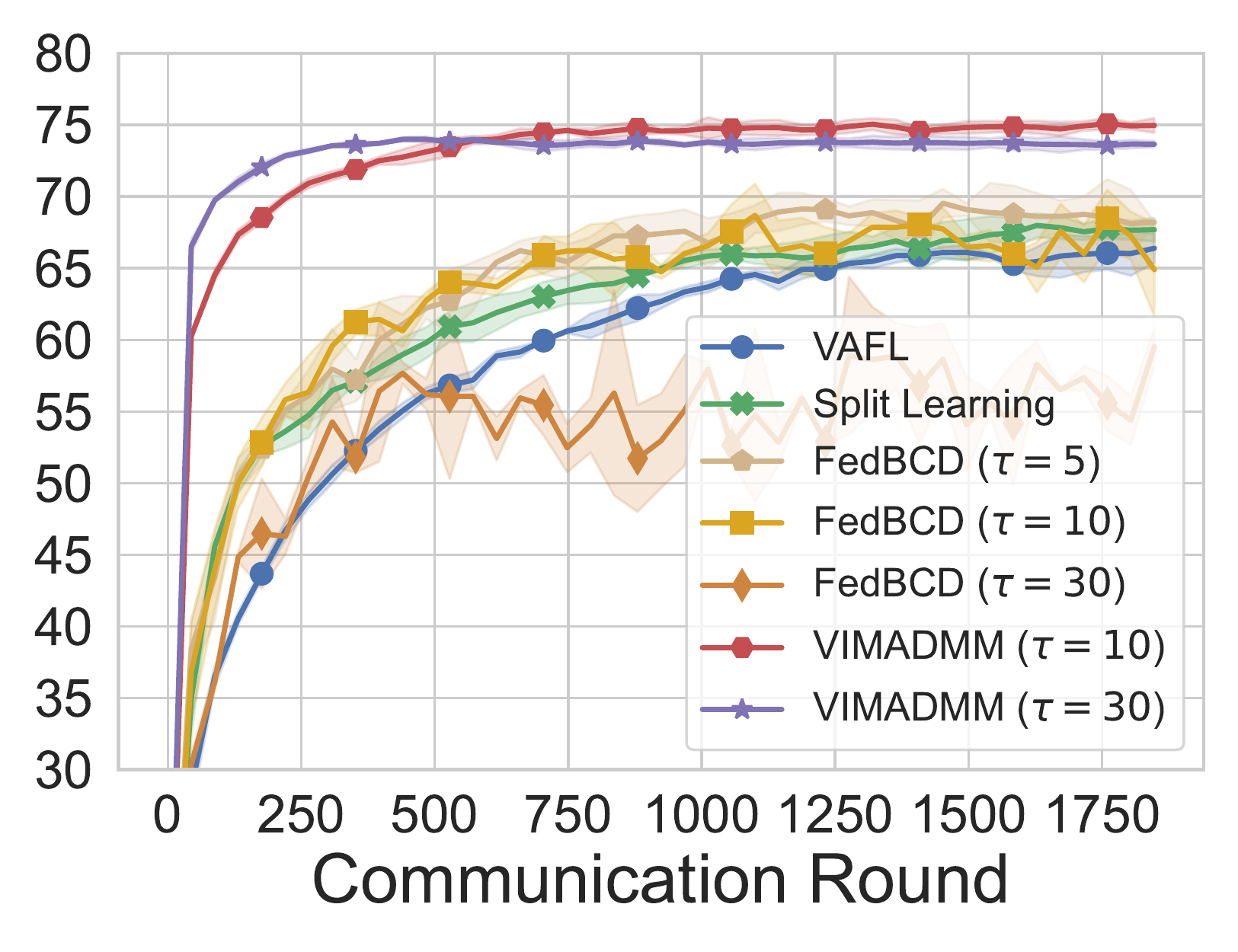}&
\includegraphics[height=\cleanheightc]{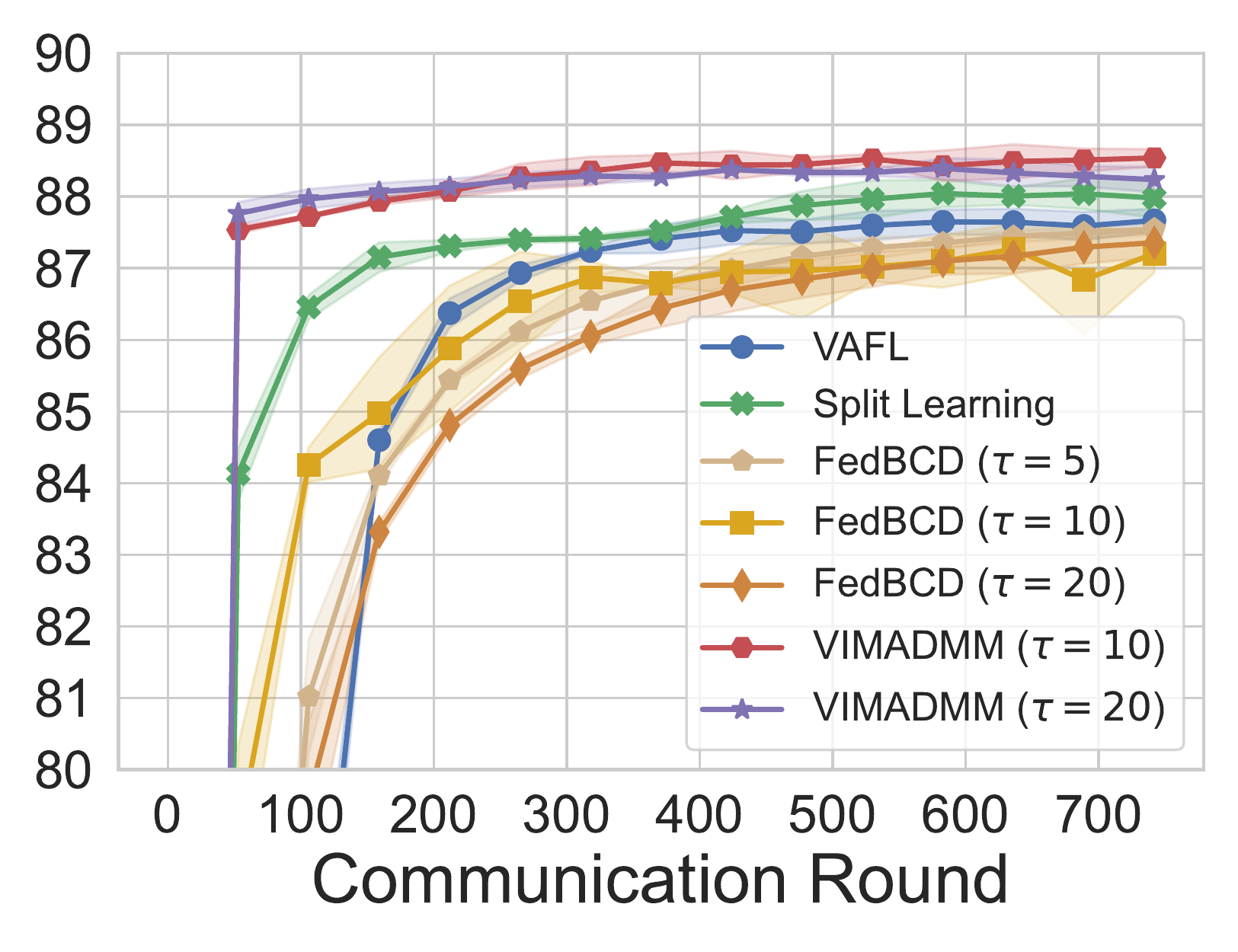}&
\includegraphics[height=\cleanheightc]{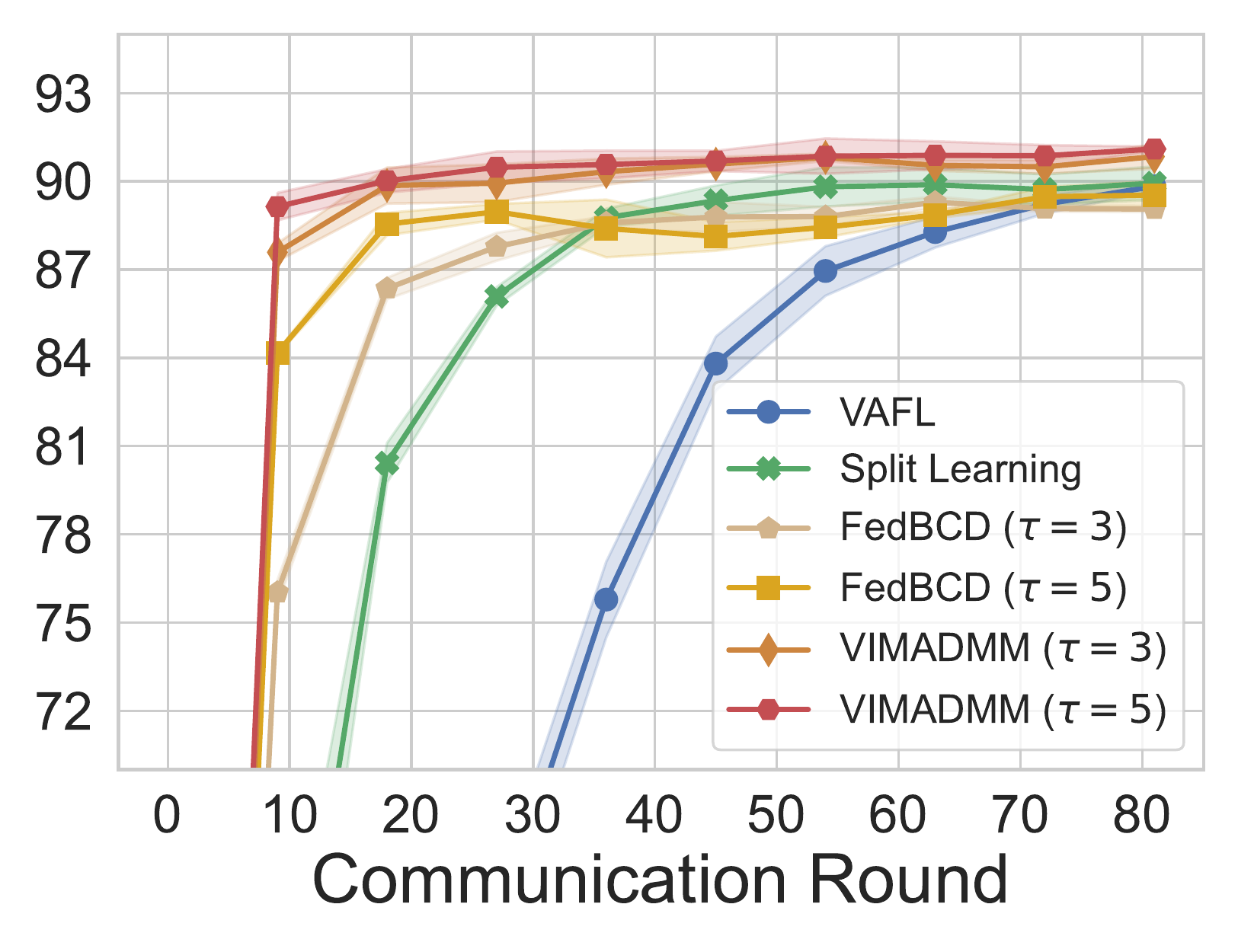}\\[-1.2ex]
\rowname{\makecell{\tiny W/o Model Split}}&
\includegraphics[height=\cleanheightc]{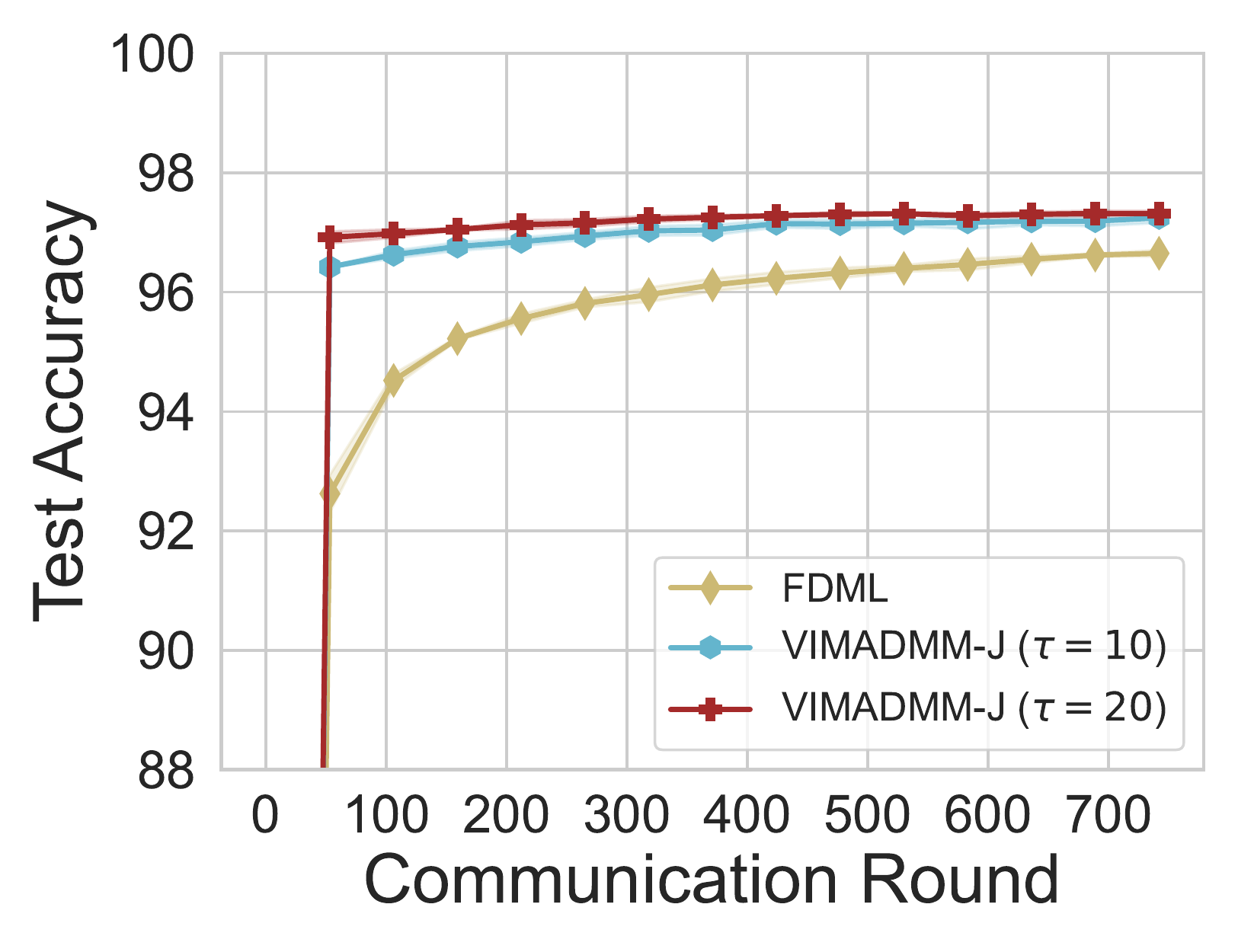}&
\includegraphics[height=\cleanheightc]{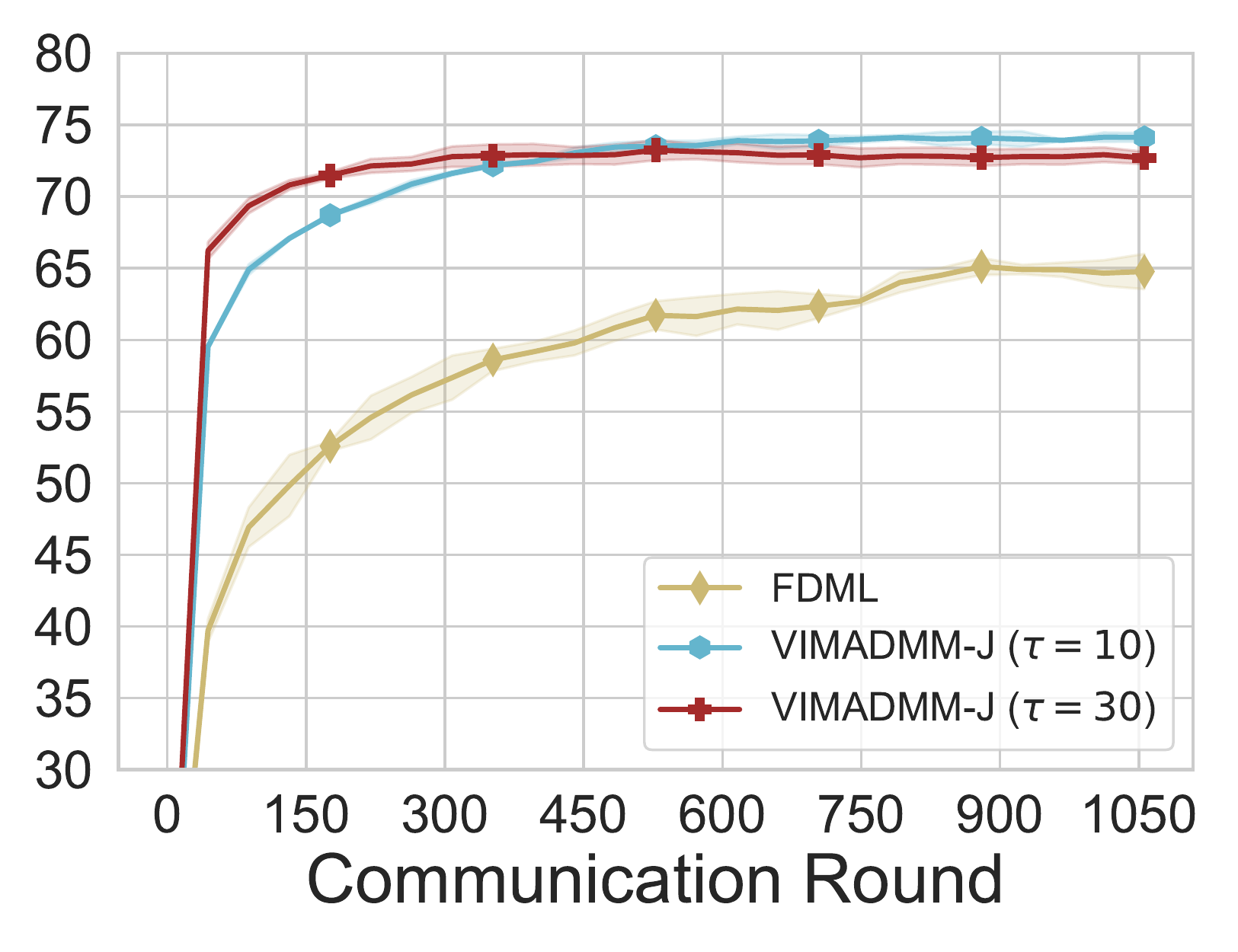}&
\includegraphics[height=\cleanheightc]{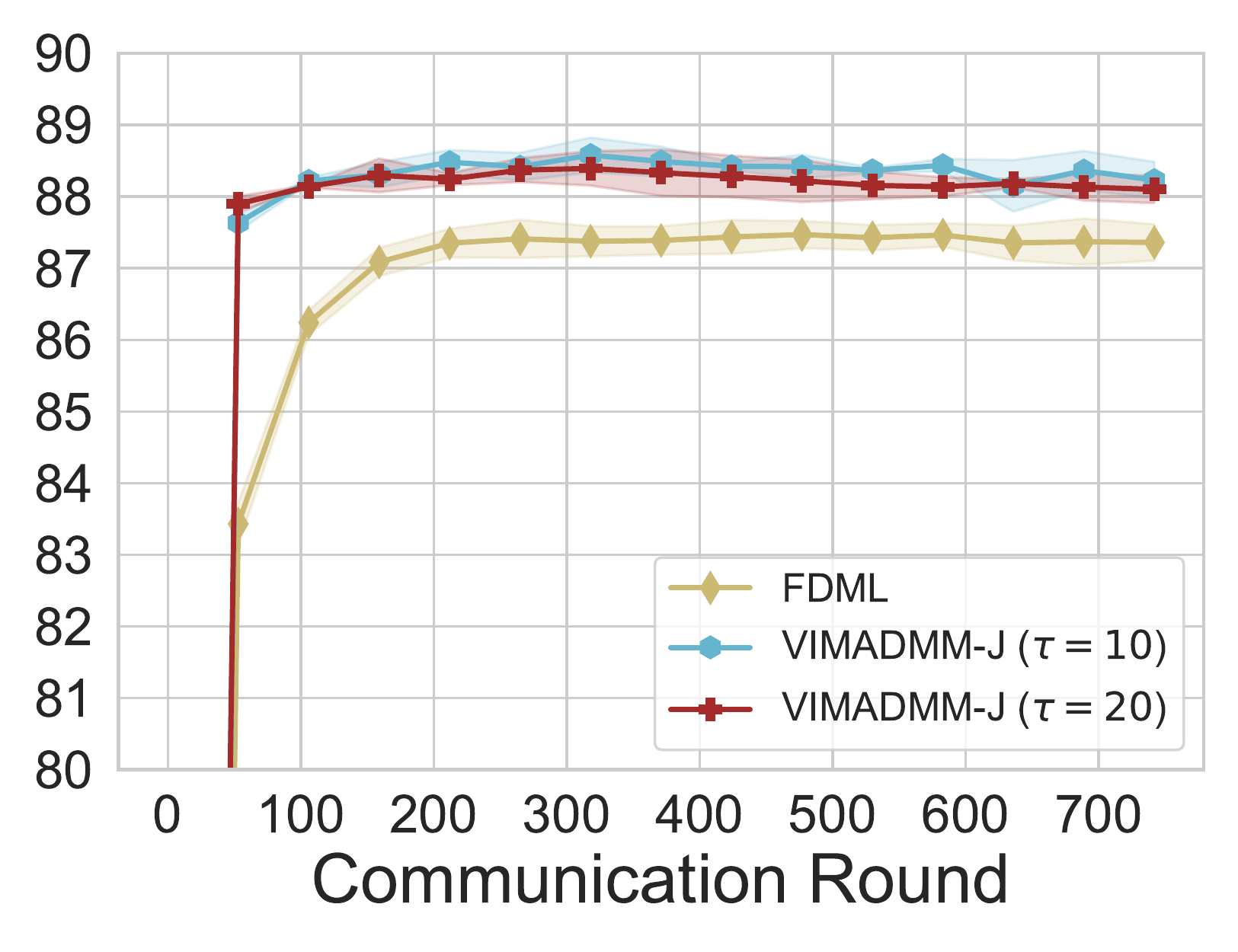}&
\includegraphics[height=\cleanheightc]{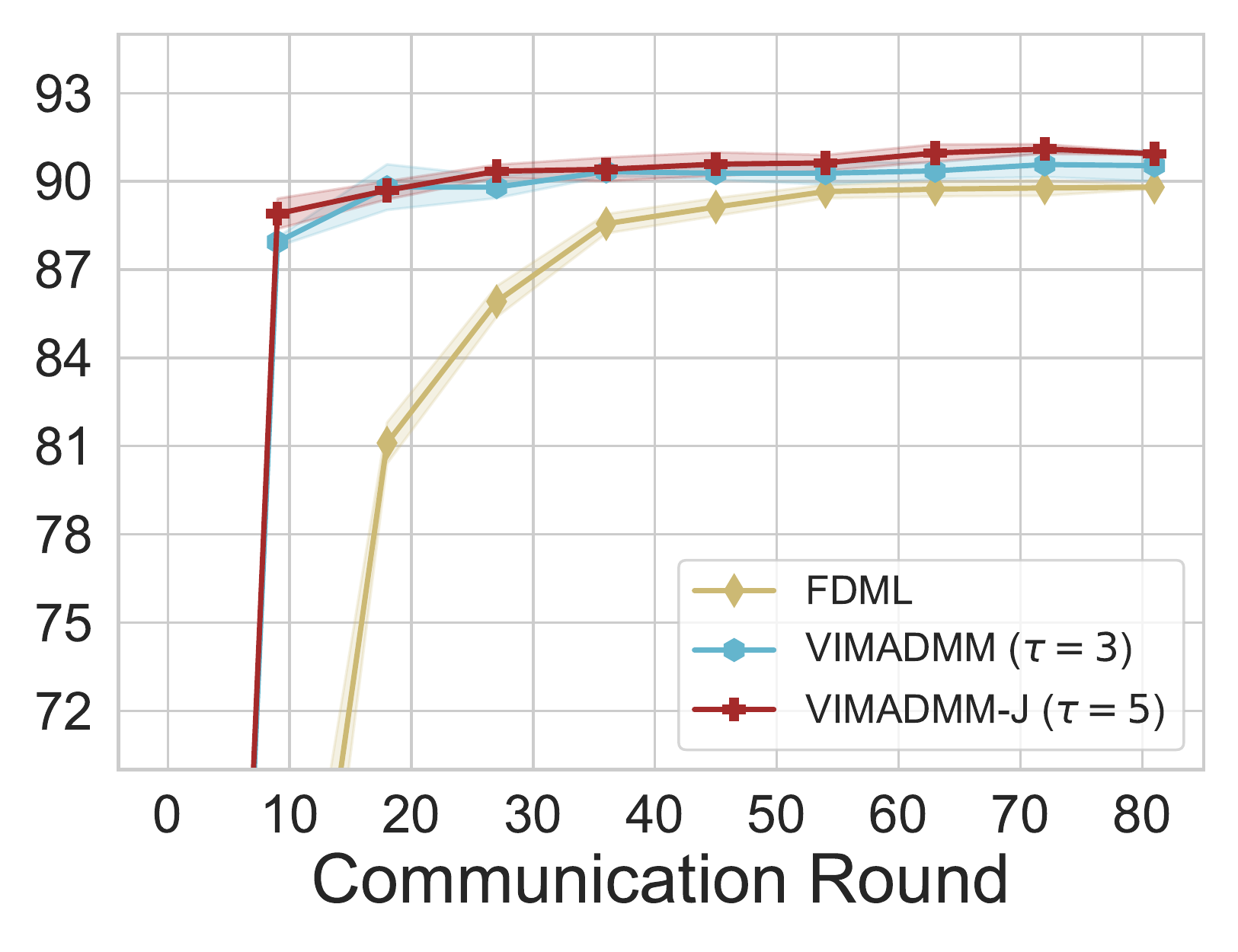}
\\[-1.2ex]
\end{tabular}
}
\end{subtable}
\caption{\small Test accuracy of VFL methods  under with model (first row) and without splitting (second row) settings on four datasets.  
Our methods (\ouradmm and \ouradmmjoint) outperforms baselines due to multiple local updates enabled by ADMM ($\tau>1$). Compared with \fedbcd under different number of local steps $\tau$, \ouradmm also achieves faster convergence and higher accuracy, which shows that the strategic utilization of ADMM-related variables in  \ouradmm is more effective than the stale partial gradient in \fedbcd for local updates.
}%
\label{fig:vanilla_vfl}
\vspace{-6mm}
\end{figure*}

}

To prevent over-fitting (due to the potential over-parameterization with the large number of model parameters from all clients and server as a global model), we adopt standard stopping criteria, i.e., stop training when the model converges or the validation accuracy starts to drop more than $2\%$. 
More details about setups and hyperparameters are in \cref{app:exp_details}.
 
\subsubsection{Baselines}
We (1) compare \ouradmm with  \vafl~\cite{chen2020vafl}  \oursgd~\cite{vepakomma2018split}, and \fedbcd~\cite{liu2022fedbcd}
under  \textit{w/ model splitting} setting;
(2) compare \ouradmmjoint with \fdml~\cite{hu2019fdml} 
under \textit{w/o model splitting} setting. 
Particularly, in \vafl, the server aggregates local embeddings using their linear combination with learnable aggregation weights, and subsequently use these aggregated embeddings as input for the server model. Both \oursgd and \fedbcd  utilize concatenated local embeddings as server model input.  
 Notably, in \vafl and \oursgd, the clients only perform one step of local update based the partial gradients from the server. Conversely, \fedbcd employs the same (stale) partial gradients for  $\tau$ local updates. 
 In \fdml, the server averages local logits, and sends aggregated logits  back to clients at eatch communication round. The clients, who owns the copies of labels, can calculate the local gradient and execute one step of local update.
 Our empirical findings suggest that our ADMM-based methods outperform the aforementioned methods, due to the multiple local updates that utilize ADMM-related variables.

For fair comparisons, we use the same local models for all methods. Under w/ model splitting setting,  owing to the strong feature extraction power of local DNN models, we utilize the linear model as server model by default.
Additionally, we evaluate all methods with the non-linear server model, as detailed in \cref{sec:exp_nonlinearhead}.

\revise{We further compare the utility of various VFL methods under differential privacy.
Existing VFL frameworks (see \cref{tb:compare_related_work}) focus on sample-level DP~\cite{chen2020vafl,hu2019fdml,hu2019learning,tran2023privacy,cohendifferentially, ranbaduge2022differentially}, where neighboring datasets are defined as those differing by a single sample in a client's local dataset. In particular, \vafl~\cite{chen2020vafl} adds random noise to the output of each local embedding convolutional layer; VFL-PBM~\cite{tran2023privacy} quantizes local embeddings into differentially private integer vectors;
\fdml~\cite{hu2019fdml} and \linearadmm~\cite{hu2019learning} add noise to local outputs. However, these methods lack exact privacy budget evaluations, providing only empirical utility under different levels of noise.  Additionally, \cite{ranbaduge2022differentially} perturbs local model weights to satisfy DP. However, it requires bounding the sensitivity of each layer's weights in the local model. 
To enable a fair comparison of VFL methods under DP guarantees, we evaluate  
all methods through our proposed DP mechanisms with perturbed local outputs to satisfy \revise{client-level} $(\epsilon,\delta)$-DP guarantee.
Notably, a mechanism satisfying $(\epsilon, \delta)$ client-level DP also satisfies $(\epsilon, \delta)$  sample-level DP based on their definitions. Since client-level DP offers stronger privacy protection, it has gained widespread adoption in FL~\cite{geyer2017differentially,mcmahan2018learning,agarwal2018cpsgd,bhowmick2018protection,xie2023unraveling}. 
} Furthermore, we evaluate all VFL methods using label DP mechanism \alibi~\cite{malek2021antipodes} to separately satisfy the $\epsilon$-label DP guarantee.
We report the averaged results of {three times} of experiments with different random seeds.

\vspace{-2mm}
\subsection{Evaluation on Vanilla VFL}
\label{sec:evaluation}
In this section, we evaluate the ADMM-based methods and baselines in terms of convergence rate, model performance and communication costs. Also, we show the generality of \ouradmm under non-linear server heads, and study \ouradmm performance  under different ADMM  penalty factor $\rho$.

\subsubsection{Convergence rates and model performance} 
\cref{fig:vanilla_vfl} shows the convergence rates of all methods, where two \name algorithms consistently outperform baselines.
Concretely, 
\textbf{(1)} 
our ADMM-based methods converge faster and achieve higher accuracy than gradient-based baselines,  especially on \cifar. 
This is because the multiple local updates enabled by ADMM lead to higher-quality local models at each round, thereby speeding up the convergence.
\textbf{(2)} \ouradmm outperforms \fedbcd under various local steps. This superiority can be attributed to the use of ADMM-related variables for local updates $\tau$ in \ouradmm, which is more effective than stale partial gradients in \fedbcd.
\textbf{(3)} When \# of local steps $\tau$ is larger, ADMM-based methods converge faster as the local models can be trained better with more local updates at each round. 
{

\newlength{\nonlinearheightc}
\settoheight{\nonlinearheightc}{\includegraphics[width=.26\linewidth]{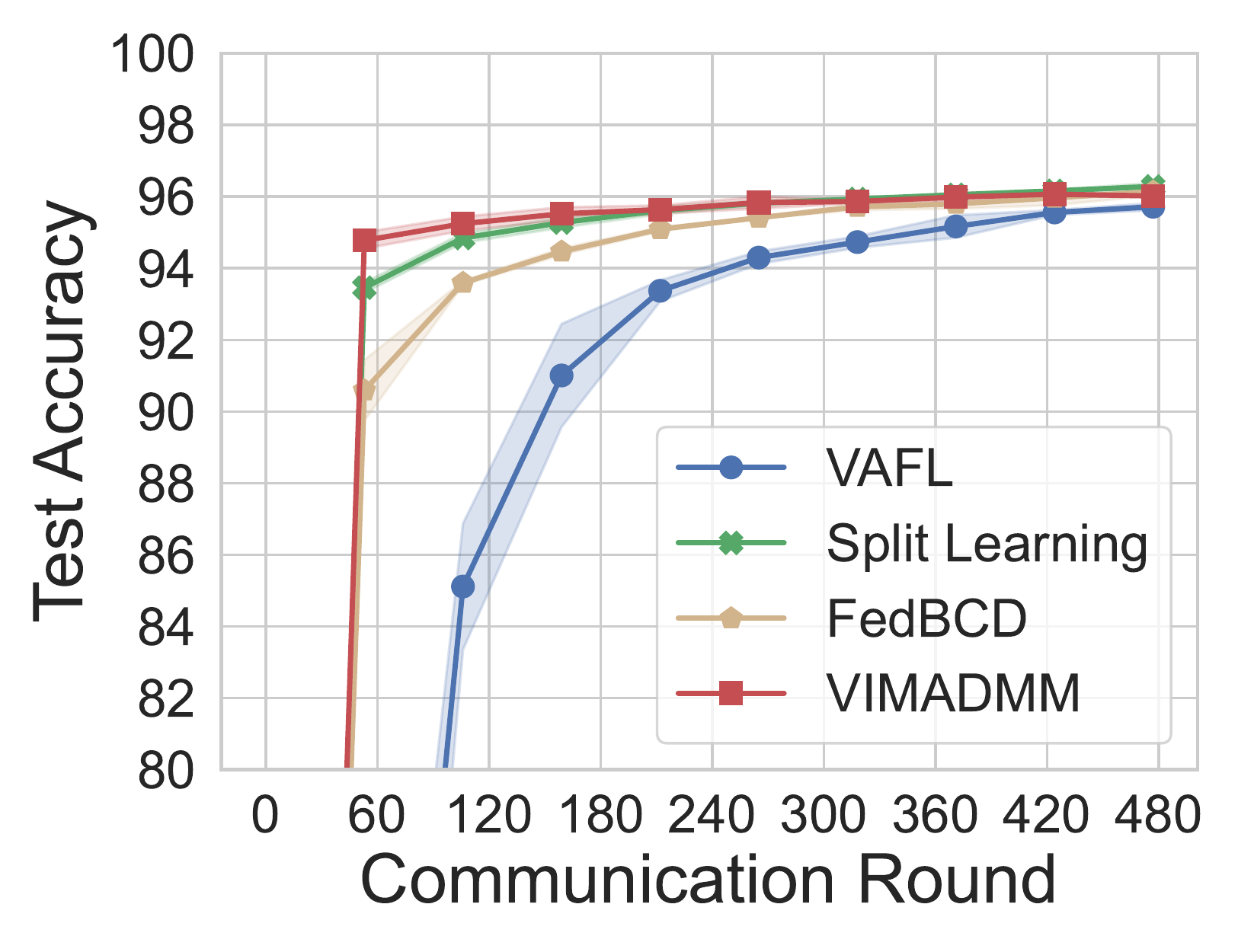}}%

\newcommand{\rowname}[1]%
{\rotatebox{90}{\makebox[\nonlinearheightc][c]{\tiny #1}}}
\renewcommand{\tabcolsep}{10pt}

\begin{figure*}[t]
\centering
\vspace{-7mm}
\begin{subtable}{0.9\linewidth}
\centering
\resizebox{\linewidth}{!}{%
\begin{tabular}{c@{}c@{}c@{}c@{}c@{}c@{}}
\makecell{\tiny{\mnist}}      & \makecell{\tiny{\cifar}} & \makecell{\tiny{\nus}}   & \makecell{\tiny{\modelnet}} \\[-0.6ex]
\includegraphics[height=\nonlinearheightc]{ICLRplots/nonlinear/mnist_split_2.pdf}&
\includegraphics[height=\nonlinearheightc]{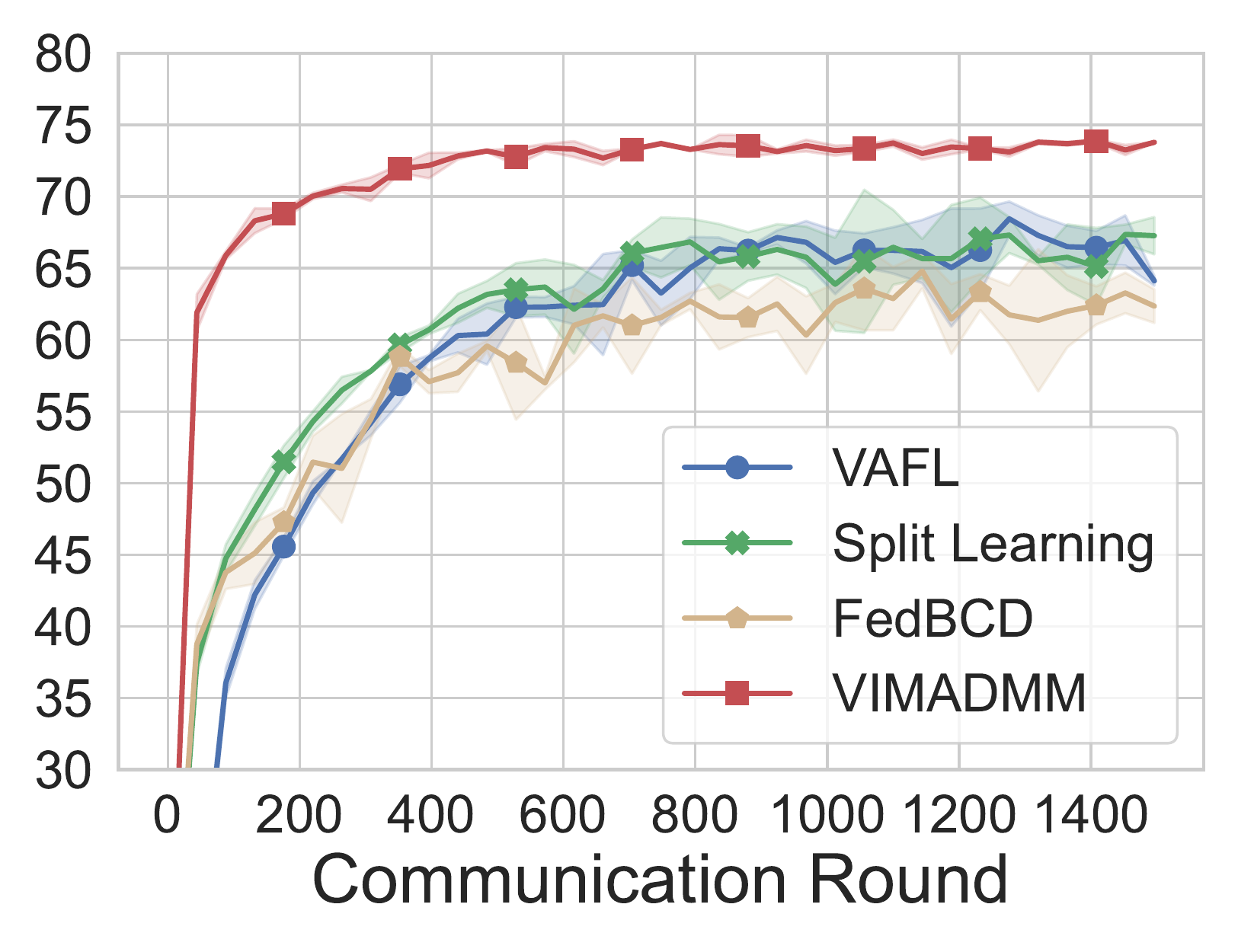}&
\includegraphics[height=\nonlinearheightc]{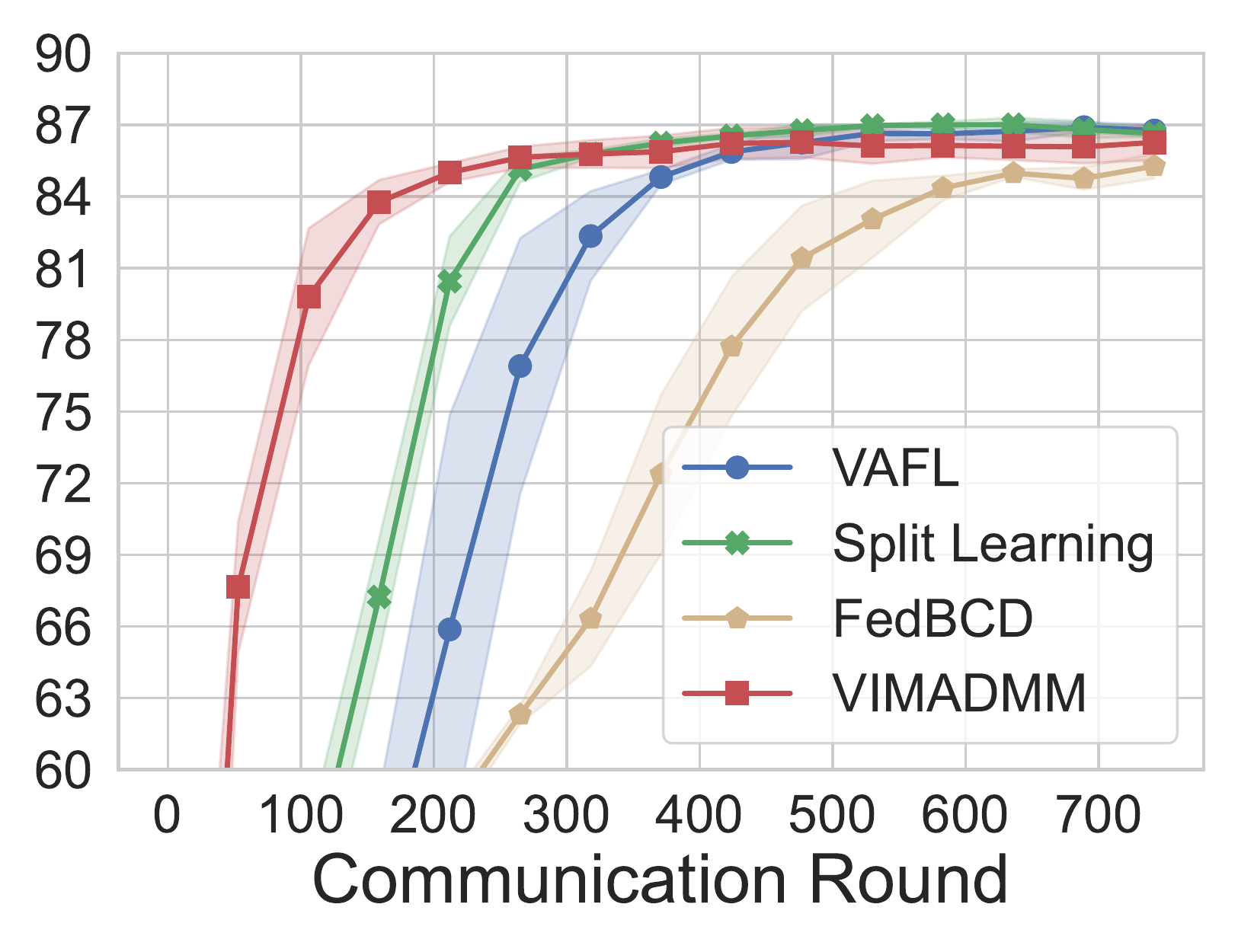}&
\includegraphics[height=\nonlinearheightc]{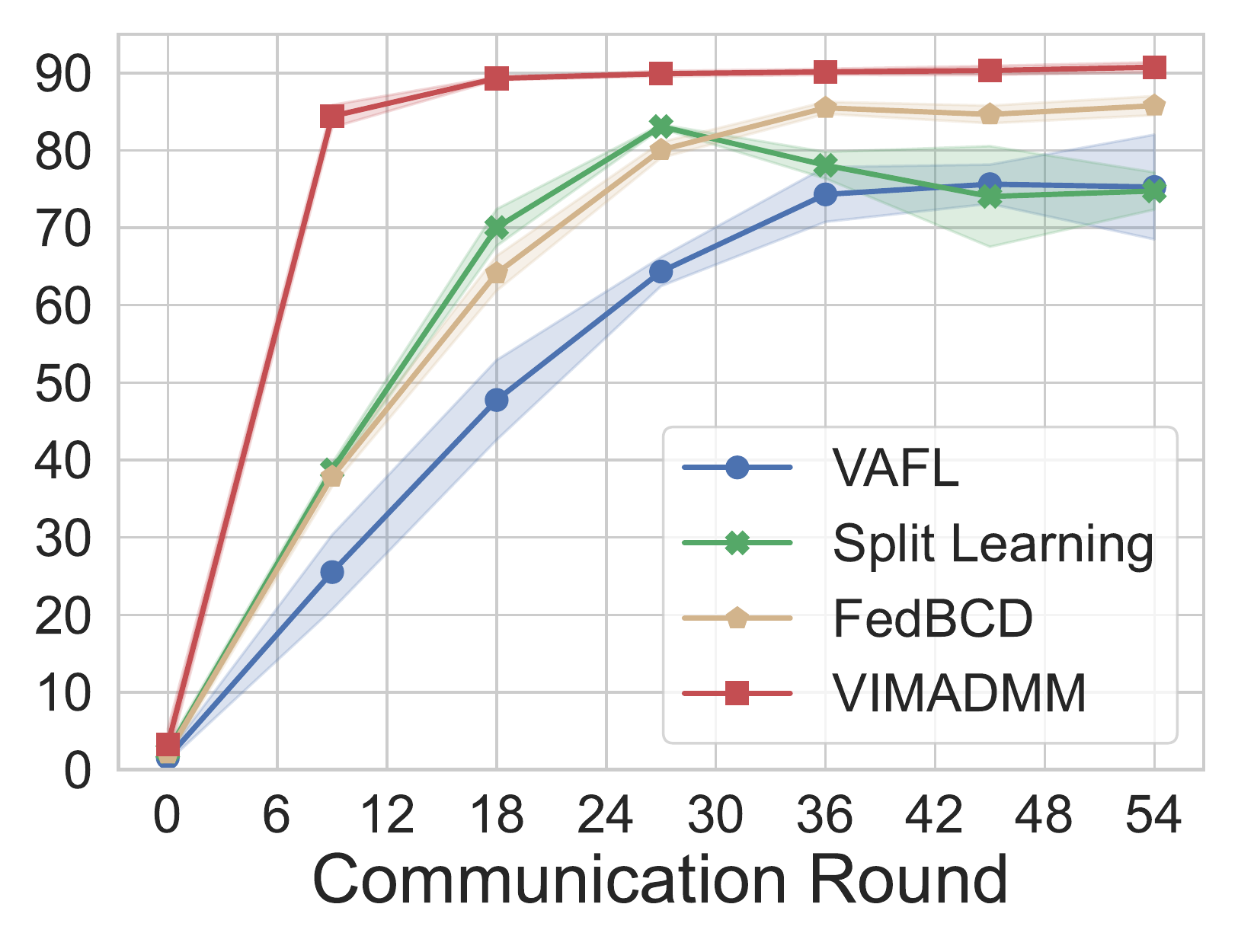}\\[-0.5ex]
\end{tabular}
}
\end{subtable}
\vspace{-3mm}
\caption{ Performance comparison when the server has the non-linear MLP model. ADMM-based method still outperforms other baselines under general architectures with the non-linear server model. }
\label{fig:vanilla_vfl_nonlinear}
\vspace{-3mm}
\end{figure*}

}

Moreover, we empirically compare \ouradmm with  \linearadmm~\cite{hu2019learning}. While both VFL methods are rooted in ADMM, we propose new VFL optimization objective and algorithm with multiple heads that enable the ADMM decomposition for 
 practical DNN training under model splitting. Results in \cref{tab:linearadmm} show that VIMADMM consistently outperforms \linearadmm on \mnist and \nus. Compared to DNNs enabled by \ouradmm, the limitations of logistic regression in \linearadmm would be more evident when applied to more complex datasets like \cifar and \modelnet.
  
\begin{table}[h]
\vspace{-2mm}
  \centering
   \caption{\small Performance comparison between \ouradmm and \linearadmm~\cite{hu2019learning}. \ouradmm achieves higher accuracy.}
     \label{tab:linearadmm}
\resizebox{0.65\linewidth}{!}{%
\begin{tabular}{ccccccccccccccc}\toprule
 & MNIST & NUS-WIDE \\\midrule  
\ouradmm & 97.13 & 88.51 \\
\linearadmm & 91.65 & 84.63
  \\\bottomrule  
\end{tabular}
}
\vspace{-3mm}
\end{table}

{

\begin{table*}[t]
\centering
\begin{minipage}[t]{\linewidth}
  \captionof{table}{\small Communication costs (in megabytes) comparison. \ouradmm requires lower communication costs per round than baselines under w/ model splitting setting.  ADMM-based methods require lower communication costs to achieve the same target accuracy performance.
  }
   \label{tab:communication}
\centering
 \resizebox{0.88\columnwidth}{!}{%
 \begin{tabular}{cccccccccccccccc}
    \toprule
 \multirow{3}{*}{VFL setup}  &  \multirow{3}{*}{Method}  &   \multicolumn{3}{c}{Comm. costs  per round} & \multicolumn{4}{c}{Comm. costs to reach target accuracy performance}\\
\cmidrule(lr){3-5}\cmidrule(lr){6-9}
&&  Each client &  Server to & \multirow{2}{*}{Total}  & \mnist  & \cifar  & \nus & \modelnet \\
&& to server & each client && ($\geq 96.0\%$) &($\geq 65.0\%$) &  ($\geq 85.0\%$) & ($\geq 89.0\%$)
\\\midrule
{\footnotesize \multirow{4}{*}{w/ model splitting}} & {\footnotesize \multirow{1}{*}{\vafl}}
&  0.23  & 0.23 & 0.46     &   4520.12	&5381.40	&397.37	&134.96 \\
[0.05em] 
&{\footnotesize \multirow{1}{*}{\oursgd}}
&  0.23 & 0.23 &  0.46&  1738.51 &	4082.44	&198.69	&84.35 \\
[0.05em] 
&{\footnotesize \multirow{1}{*}{\fedbcd}}
&  0.23 & 0.23 &  0.46&  4867.82  &2597.92& 397.37 & 118.09 \\
[0.05em] 
&{\footnotesize \multirow{1}{*}{\ouradmm} }
&  \multirow{1}{*}{0.23} &  \multirow{1}{*}{\textbf{0.08 }}  & \multirow{1}{*}{\textbf{0.31}} &  \textbf{233.36} 	& \textbf{124.54} 
 & \textbf{66.67 } & \textbf{11.32}  \\\midrule[0.05em]
{\footnotesize \multirow{2}{*}{w/o model splitting}} & {\footnotesize \multirow{1}{*}{\fdml}}
    &  0.039 &\textbf{0.039} &  \textbf{0.078} &  405.13	& 617.76& 33.07	& 89.13  \\
    [0.05em]
  &  {\footnotesize \multirow{1}{*}{\ouradmmjoint} }
    &  0.039 & {0.078} & {0.117} &  \textbf{86.81}	& \textbf{46.33} & \textbf{24.8 }&\textbf{ 8.42} \\
    
            \bottomrule

\end{tabular}%
  }
\end{minipage}
\vspace{-5mm}
\end{table*}
}

\subsubsection{Non-linear server heads} 
\label{sec:exp_nonlinearhead}
To demonstrate the  generality and applicability of \ouradmm,  we evaluate \ouradmm when the server model is non-linear. Specifically, the head 
 consists of multiple fully-connected layers accompanied by Dropout layers with 0.25 dropout rate and ReLu activation functions.  For a fair comparison, we also use MLP server model architecture for other baseline methods.
 We use 3 layered MLP for \nus and 2 layered MLP for other datasets. 
 The evaluation results in \cref{fig:vanilla_vfl_nonlinear}  show that our method still outperforms other baselines under general architectures with the non-linear server model. 

\subsubsection{Communication costs} 
Here we report the memory of parameters communicated between clients and the server to evaluate communication cost in Table~\ref{tab:communication}.  We use batch size 1024 and local embedding size 60 for all datasets. The overall embedding size scales with the number of clients.
From Table~\ref{tab:communication}, we observer that
\textbf{(1)} for each round, all methods under w/ model splitting setting  have the same number of parameters sent from each client to the server (i.e., 0.23 MB for a batch of embeddings), and \ouradmm has a smaller number of parameters sent from server to each client   (i.e., 0.08 MB in total for a batch of dual variables, residual variables as well as one corresponding linear head)  than \vafl, \oursgd and \fedbcd (i.e., 0.23 MB for a batch of gradients w.r.t. embeddings). 
\textbf{(2)} With smaller \# of communicated parameters at each round and faster convergence (i.e., smaller \# of communicated rounds to achieve a target accuracy), \ouradmm requires significantly lower communication costs than baselines. For example, to achieve 65.0\%  accuracy on \cifar, \vafl needs 5381.4 MB while \ouradmm only requires 124.54 MB, which is about 43x lower costs. Here we use $\tau=20,30,20,5$ for the four datasets respectively.
\textbf{(3)} The results under w/o model splitting setting  demonstrates that \ouradmmjoint incurs lower communication costs than \fdml to achieve the same accuracy, due to faster convergence with multiple local updates.
\textbf{(4)} We note that the communication cost under w/o model splitting setting is generally lower than w/ model splitting setting, which is because the local logits  have a lower dimension than local embeddings, i.e., $d_c < d_f$.

\subsubsection{Effect of penalty factor $\rho$}

In ADMM-based methods, we introduce one hyper-parameter -- penalty factor $\rho$.
Here we study the test accuracy of \ouradmm with different penalty factor $\rho$. The results in \cref{fig:rho} \revise{of \cref{app:exp_details}} show that \ouradmm is not sensitive to $\rho$ on four datasets, and we suggest that the practitioners choose the optimal $\rho$ from 0.5 to 2, which does not influence the test accuracy significantly.  

 \subsubsection{Evaluation on long-tail datasets}
Long-tail datasets are characterized by a significant imbalance, where minority classes have far fewer samples than majority ones. This horizontal imbalance is distinct from the challenges addressed by VFL, where  \textit{the same sample (whether it belongs to a majority or minority class)  is vertically split across multiple clients.} 
We compared the \ouradmm model, which consists of $M$ local models followed by a server model, with a reference model in a centralized setting. This reference model has the same model size as one local model coupled with a server model. 
The results in \cref{tab:longtail} demonstrate that  \ouradmm is still effective on challenging long-tail training datasets, yielding results comparable to those of the reference model in a centralized setting. 
We defer more discussion and detailed experimental setups to \cref{app:exp_details}.

\begin{table}[h]
  \centering
  \vspace{-2mm}
   \caption{\small Accuracy and fairness (measured by Standard Deviation of class-wise accuracy) on balanced data and long-tail data.}
   \label{tab:longtail}
\resizebox{\linewidth}{!}{%
\begin{tabular}{ccccccccccccccc}\toprule
& balanced \mnist  & long-tail \mnist & balanced \cifar  & long-tail \cifar  \\\midrule

\ouradmm   &
97.13 (0.76) &
95.69 (1.58) & 
75.25 (9.17) &
62.81 (15.27)\\

Reference model  in  centralized setting  &
98.19 (0.45) &
95.02 (2.70)&
77.61 (9.20)  &
66.11 (15.29) 
  \\\bottomrule  
\end{tabular}
}
\vspace{-3mm}
\end{table}

{

\begin{table*}[tbp] 
\centering
 \begin{minipage}{0.83\linewidth}
  \caption{\small Utility of VFL methods under \textit{user-level DP}. ADMM-based methods maintain higher utility.}
     \label{tab:dp-utility}
     \resizebox{1\columnwidth}{!}{%
     \centering
\begin{tabular}{ccccccccccccccc}\toprule
  \multirow{2}{*}{VFL setup} &   \multirow{2}{*}{Method}       &  \multicolumn{3}{c}{\mnist} &    \multicolumn{3}{c}{\cifar}  &    \multicolumn{3}{c}{\nus} &     \multicolumn{3}{c}{\modelnet}        \\\cmidrule(lr){3-5} \cmidrule(lr){6-8}  \cmidrule(lr){9-11}\cmidrule(lr){12-14}
&     & $\epsilon=\infty$ & $\epsilon=8$ & $\epsilon=1$  & $\epsilon=\infty$ & $\epsilon=8$ & $\epsilon=1$  & $\epsilon=\infty$    & $\epsilon=8$ & $\epsilon=1$ & $\epsilon=\infty$    & $\epsilon=8$ & $\epsilon=1$  \\\midrule
  \multirow{4}{*}{w/ model splitting} &  \vafl      & 96.86 & 22.29 & 11.31     & 66.39 & 16.82 & 14.91    & 87.81    & 38.27 & 38.19    & 90.07    & 4.66  & 4.29    \\
&\oursgd   & 96.92 & 56.53 & 16.77   & 68.32 & 21.09 & 15.8  &   88.25    & 38.29 & 33.05   & 89.98    & 18.19 & 6.28     \\
& \fedbcd &  96.59 & 66.07 & 65.05 & 71.2 &70.67  & 55.42 & 87.59 &42.95  & 41.02  & 89.87 & 88.3 & 87.02 \\
&\ouradmm   & 97.13 & 92.35 & 92.09      & \textbf{75.25} & \textbf{73.83} & \textbf{61.65}   & \textbf{88.51 }   & 83.77 &83.51    & \textbf{91.32 }  & \textbf{91.29}  & \textbf{91.18}  \\\midrule
  \multirow{2}{*}{w/o model splitting} &  \fdml      & 97.06 & 92.02 & 85.01    & 66.8  & 41.07 & 35.25   & 87.67    & 79.58 & 67.38      & 89.86    & 54.7  & 43.4     \\
& \ouradmmjoint & \textbf{97.37} & \textbf{92.71} & \textbf{92.33}   & 74.48 & 72.36 & 58.64   & 88.46    & \textbf{84.94} &\textbf{84.88}    & 91.13   & 90.13 & 89.37  \\\bottomrule  
\end{tabular}
}
\hfill\hfill\hfill
\end{minipage}
\end{table*}

}

\begin{table*}[tbp] 
\vspace{-1mm}
\centering
 \begin{minipage}{0.83\linewidth}
  \caption{\small Utility of VFL methods under \textit{label-level DP}. ADMM-based methods maintain higher utility.}
     \label{tab:labeldp-utility}
     \resizebox{1\columnwidth}{!}{%
     \centering
\begin{tabular}{ccccccccccccccc}\toprule
  \multirow{2}{*}{VFL setup} &   \multirow{2}{*}{Method}       &  \multicolumn{3}{c}{\mnist} &    \multicolumn{3}{c}{\cifar}  &    \multicolumn{3}{c}{\nus} &     \multicolumn{3}{c}{\modelnet}        \\\cmidrule(lr){3-5} \cmidrule(lr){6-8}  \cmidrule(lr){9-11}\cmidrule(lr){12-14}
&     & $\epsilon=\infty$ & $\epsilon=2.8$ & $\epsilon=1.4$   & $\epsilon=\infty$ & $\epsilon=2.8$ & $\epsilon=1.4$  & $\epsilon=\infty$    & $\epsilon=2.8$ & $\epsilon=1.4$  & $\epsilon=\infty$    & $\epsilon=2.8$ & $\epsilon=1.4$   \\\midrule
  \multirow{4}{*}{w/ model splitting} &  \vafl      & 96.86 &94.27  & 51.68    & 66.39 & 54.6 &  38.44   & 87.81    & 85.77  &   60.41   & 90.07    &45.26  &   2.59 \\
&\oursgd   & 96.92 & 94.93 & 91.75   & 68.32 &57.12 & 49.71  &   88.25    &85.86  &  82.3    & 89.98    & 65.68  &  33.79   \\
& \fedbcd &  96.59 & 94.47 & 87.95 & 71.2 & 61.05  & 46.14  & 87.59 &85.62  &64.01& 89.87 &65.92 & 43.15 \\
&\ouradmm   & 97.13 &95.48  &92.8      & \textbf{75.25} & \textbf{65.07} &  52.97   & \textbf{88.51 }   &86.62  &    82.43  & \textbf{91.32 }  &76.70 &  \textbf{46.39  } \\\midrule
\multirow{2}{*}{w/o model splitting} &  \fdml      & 97.06 & 94.97 &   91.87  & 66.8  & 58.78  & 49.83   & 87.67    &85.79 &   82.37    & 89.86    & 64.99 &  29.74    \\
& \ouradmmjoint & \textbf{97.37} &\textbf{95.80}  &  \textbf{ 93.25}  & 74.48 &  64.04 &  \textbf{53.49}   & 88.46    &\textbf{86.74}  &   \textbf{82.71}   & 91.13   & \textbf{77.15} & 45.22  \\\bottomrule  
\end{tabular}
}
\hfill\hfill\hfill
\end{minipage}
\vspace{-5mm}
\end{table*}

 \subsubsection{Fairness implication}
A common fairness definition is to
enforce accuracy parity between protected groups~\cite{zafar2017fairness}. 
Here we study the fairness implications of \ouradmm on achieving accuracy parity,  at both the class and client levels:
\textbf{(1)}
when considering class-level accuracy parity, a fair model should exhibit equalized accuracy for each class~\cite{tarzanagh2023fairness,xu2021robust}, indicating that the model's accuracy is statistically independent of the ground truth label. We use the Standard Deviation of class-wise accuracy~\cite{xu2021robust} to evaluate fairness, where a lower value indicates higher fairness.
The results in \cref{tab:longtail} show that \ouradmm performs comparably or even better in fairness than the reference model in a centralized setting, across MNIST and CIFAR10 datasets with both balanced and long-tail distributions.
\textbf{(2)}
Furthermore, client-level accuracy parity is a prevalent criterion for fairness in FL ~\cite{li2021ditto,Li2020Fair}, measuring the
degree of uniformity in performance across clients. Notably,  in VFL, all clients share the same prediction for each sample, where each of them contributes partial features. Consequently, all clients inherently achieve the same accuracy, fulfilling client-level accuracy parity by the nature of VFL.

\vspace{-2mm}
\subsection{Evaluation on Differentially Private VFL}
\label{sec:exp_dpvfl}
We evaluate the utility of ADMM-based methods and baselines under \revise{client-level} DP and label DP, which protect the privacy of local features and server labels, respectively.

\subsubsection{Utility under \revise{client-level} DP (privacy of client data)}
We report the utility under  $\epsilon=8$ and $\epsilon=1$ \revise{client-level} DP. To ensure fair comparison,  we perform a grid search for the combination of hyperparameters, including noise scale $\sigma$,  clipping threshold $C$, and learning rate $\eta$,  for all methods (details are deferred to \cref{app:exp_details}). 
\cref{tab:dp-utility} shows that 
\textbf{(1)} the accuracy of ADMM-based methods under DP is on par with  the non-private accuracy ($\epsilon=\infty$) on \mnist, \nus and \modelnet. Nevertheless,  there is a discernible decrease of 13.6\% for \ouradmm on \cifar when $\epsilon=1$, which underscores the inherent privacy-utility trade-off for algorithms with formal DP privacy guarantees~\cite{abadi2016deep}.
\textbf{(2)} Our ADMM-based methods reach significantly higher utility than gradient-based methods, especially under {small} $\epsilon$.
We attribute this to the fact that ADMM-based methods converge in fewer rounds than gradient-based methods at each round, which is also evident in the non-DP setting as shown in \cref{fig:vanilla_vfl}. \textit{This rapid convergence is critical for DP, since the privacy budget $\epsilon$ is consumed quickly as communication rounds increase. 
}
The fast convergence and high utility of \ouradmm under DP compared to other baselines can be interpreted through two lenses. First, \textbf{multiple local updates} lead to a more effectively trained local model at each round. 
As a consequence, both \fedbcd and \ouradmm demonstrate a markedly better DP-utility tradeoff compared to \vafl and \oursgd, as illustrated in \cref{tab:dp-utility}. Furthermore, we explicitly investigate the influence of $\tau$ on the utility of \ouradmm under $\epsilon=1$ in \cref{tab:tau_userdp}. The results show that opting for a $\tau>1$ yields substantially enhanced accuracy than when $\tau=1$ (e.g.,  14.68\% improvement on \cifar).
Second, \textbf{update mechanism of ADMM} empowers clients to  \textit{independently} update their local models w.r.t the ADMM sub-objective (Eq.~\ref{eq:update_theta}). It is worth noting that during this local forward/backward computation based on Eq.~\ref{eq:update_theta},
clients do \textit{not} add noise locally, since local models always remain in their possession without sharing.
Clients only need to perturb local embeddings that are sent to the server (i.e., output perturbation). 
Consequently, even though  the server leverages these perturbed embeddings to derive  ADMM-related variables, the clients will re-calculate clean embeddings during  forward pass of Eq.~\ref{eq:update_theta} based on the received ADMM-related variables for local model updates.  This updating mechanism potentially facilitate convergence under DP. In contrast, gradients-based methods solely   rely on the partial gradients, which are derived from perturbed embeddings, for local update, leading to compromised utility. 

{
\renewcommand{\thesubfigure}{\alph{subfigure}}

\newlength{\histdimab}
\settoheight{\histdimab}{\includegraphics[width=.4\linewidth]{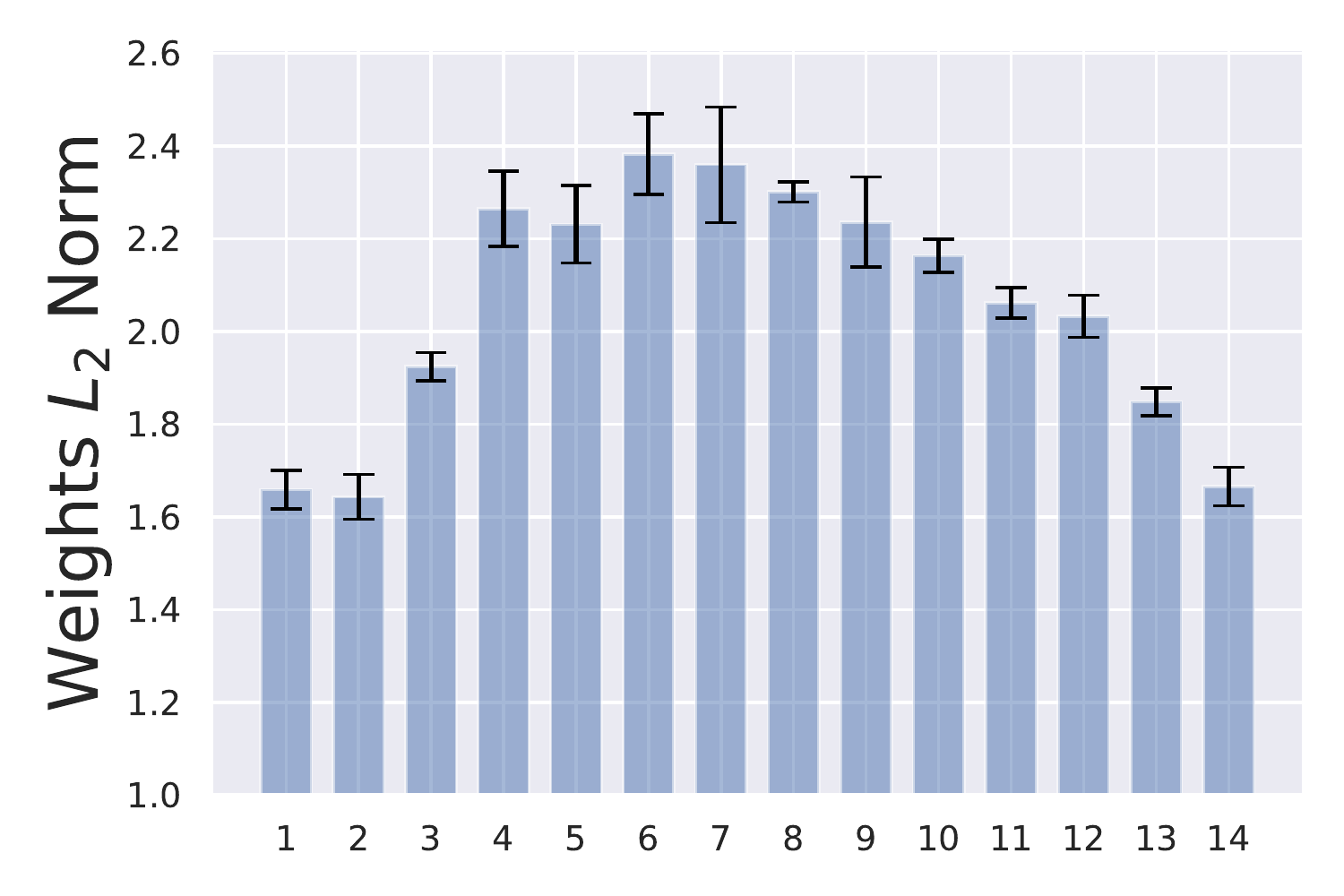}}%

\newlength{\histheightc}
\settoheight{\histheightc}{\includegraphics[width=.421\linewidth]{plots/acc/mnist_hist.pdf}}%

\newlength{\legendheighthist}
\setlength{\legendheighthist}{0.2\histheightc}%

\newcommand{\rowname}[1]%
{\rotatebox{90}{\makebox[\histdimab][c]{ #1}}}
\renewcommand{\tabcolsep}{10pt}

\begin{figure*}[t]
\vspace{-3mm}
\centering
{
\begin{subtable}{\linewidth}
\centering
\begin{tabular}{@{}p{5mm}@{}c@{}c@{}c@{}c@{}}
        & \makecell{{\mnist}}
        & \makecell{{\cifar}}
        & \makecell{{\nus}}
        & \makecell{{\modelnet}}
        \vspace{-3pt}\\
\rowname{\makecell{input features}}&
\includegraphics[height=\histheightc]{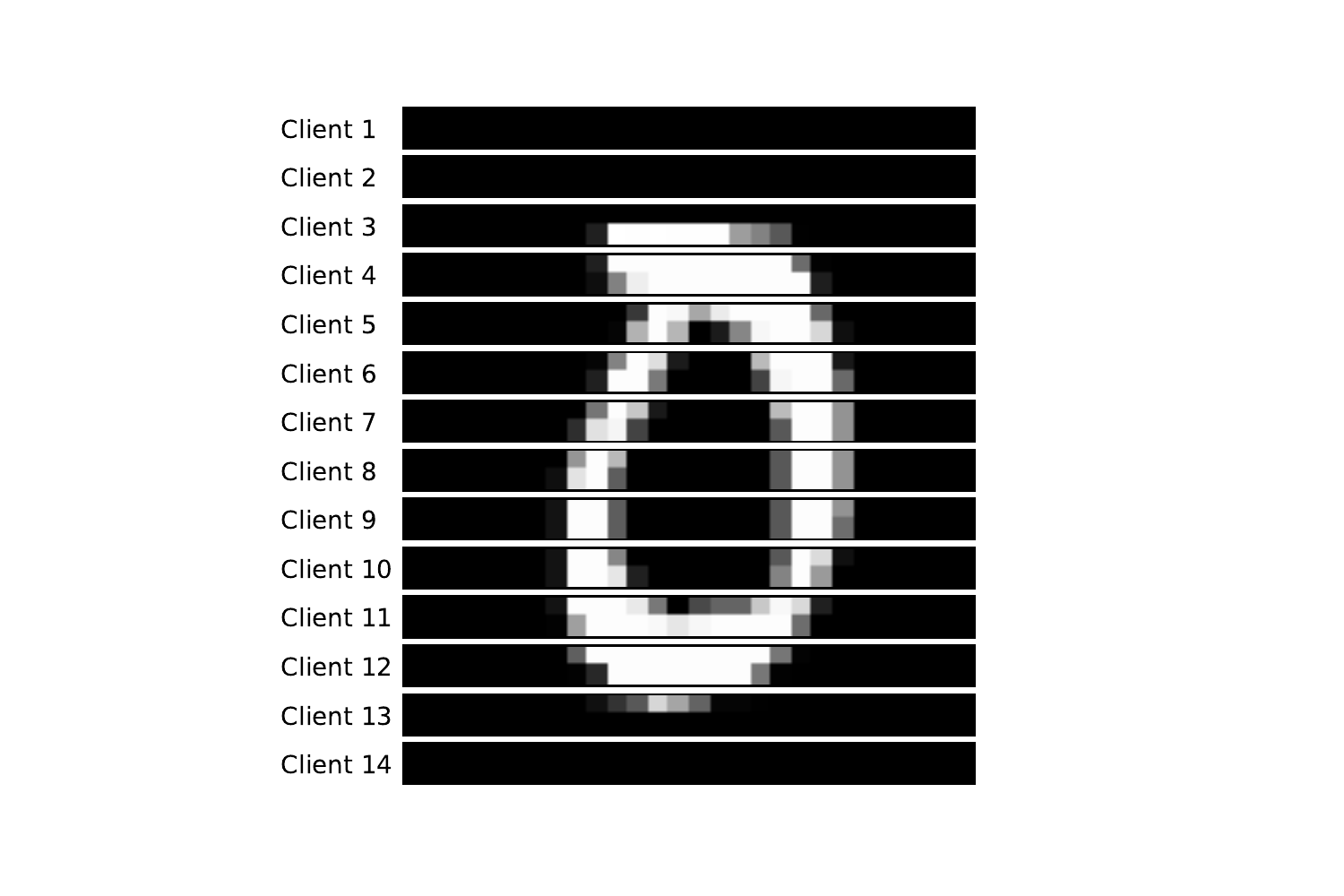}&
\includegraphics[height=\histheightc]{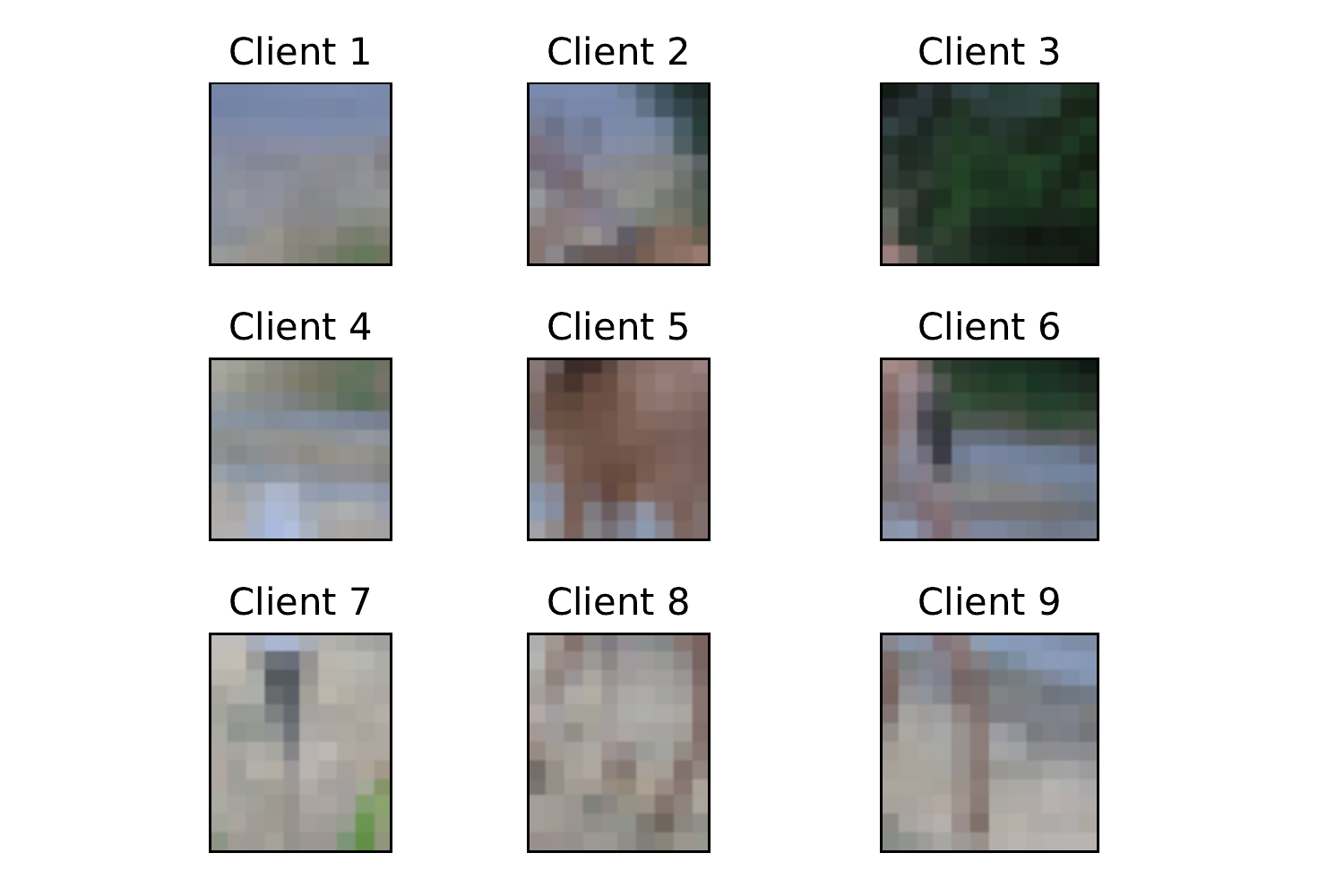}&
\includegraphics[height=\histheightc]{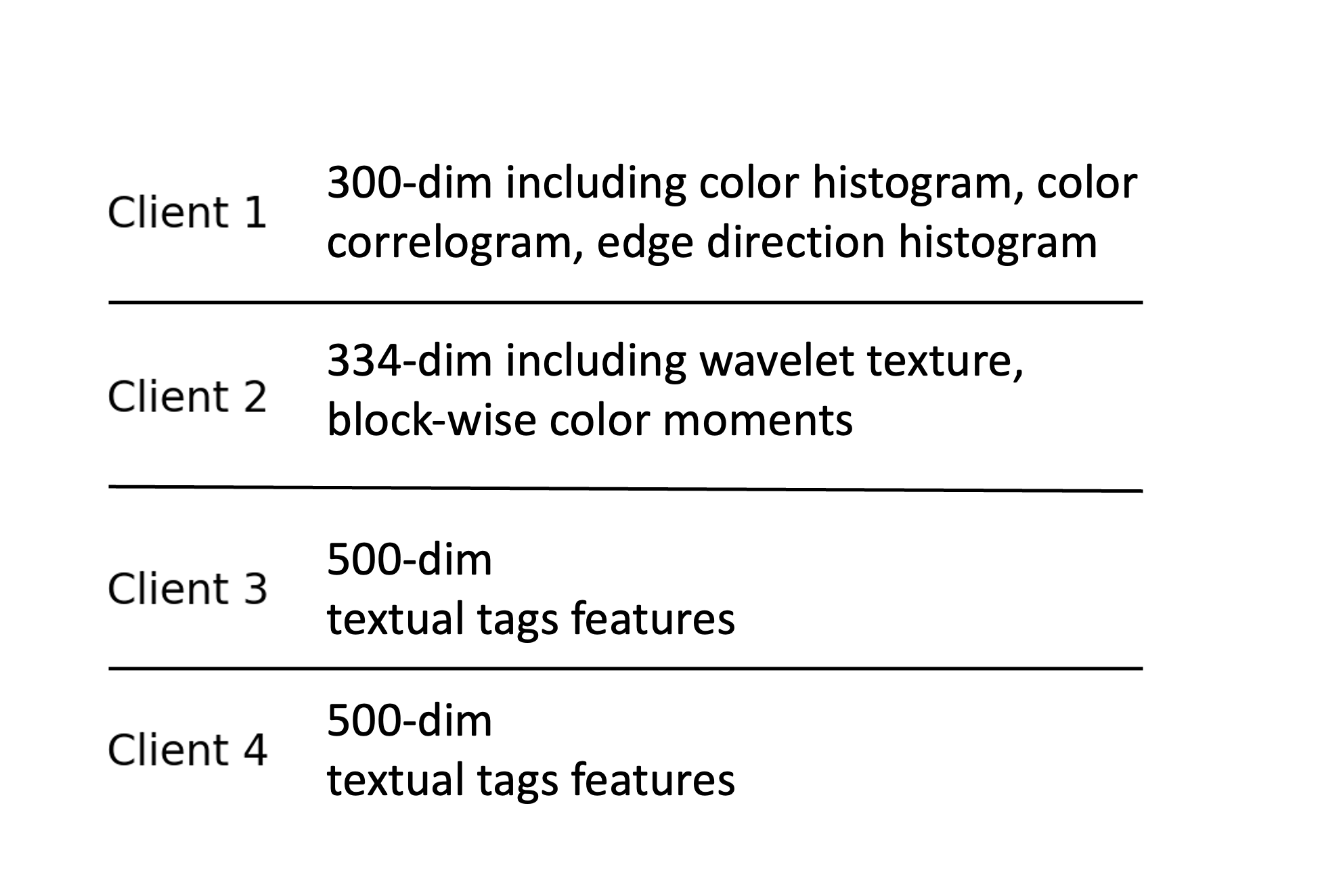}&
\includegraphics[height=\histheightc]{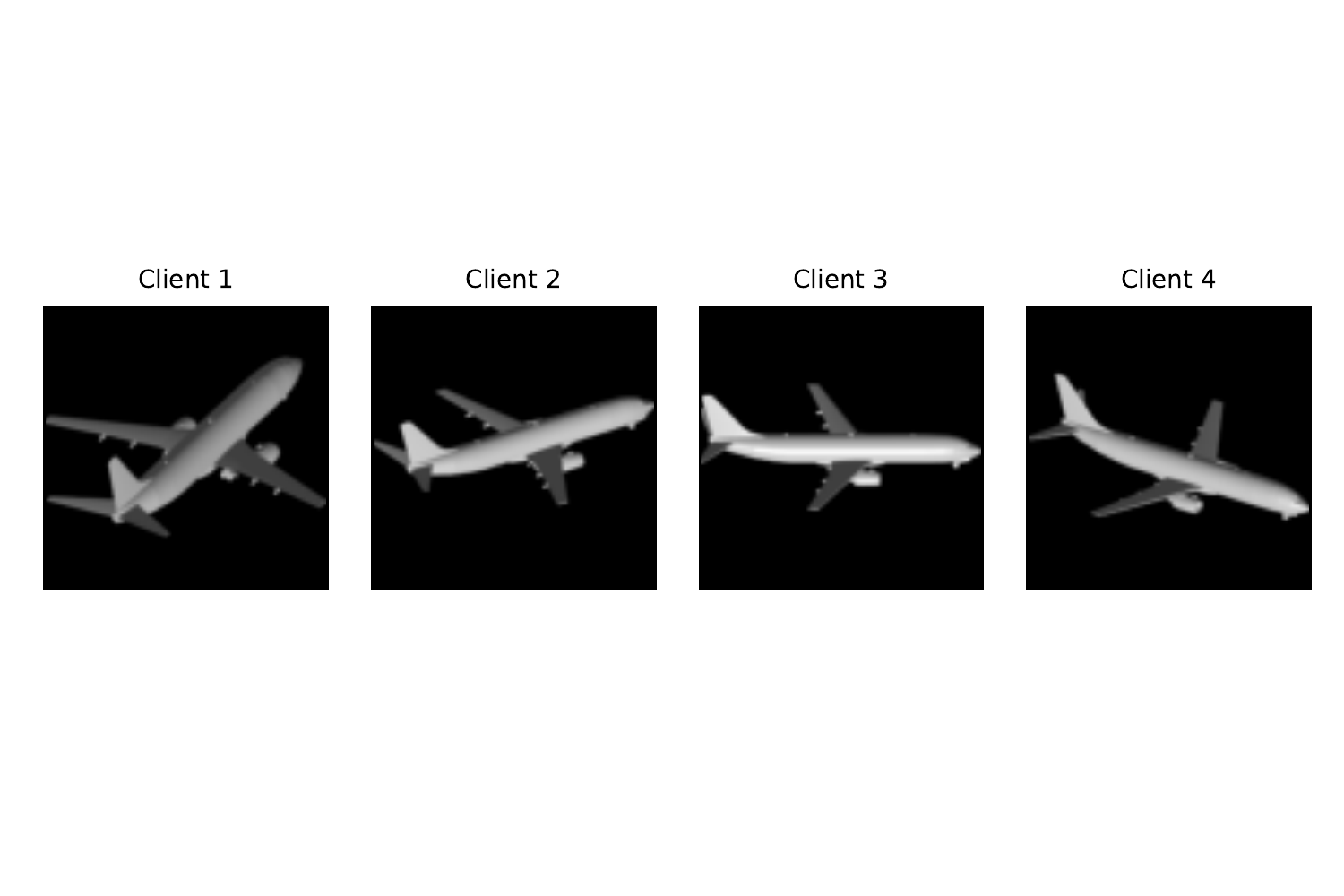}\\[-1.2ex]
\rowname{\makecell{clean}}&
\includegraphics[height=\histheightc]{plots/acc/mnist_hist.pdf}&
\includegraphics[height=\histheightc]{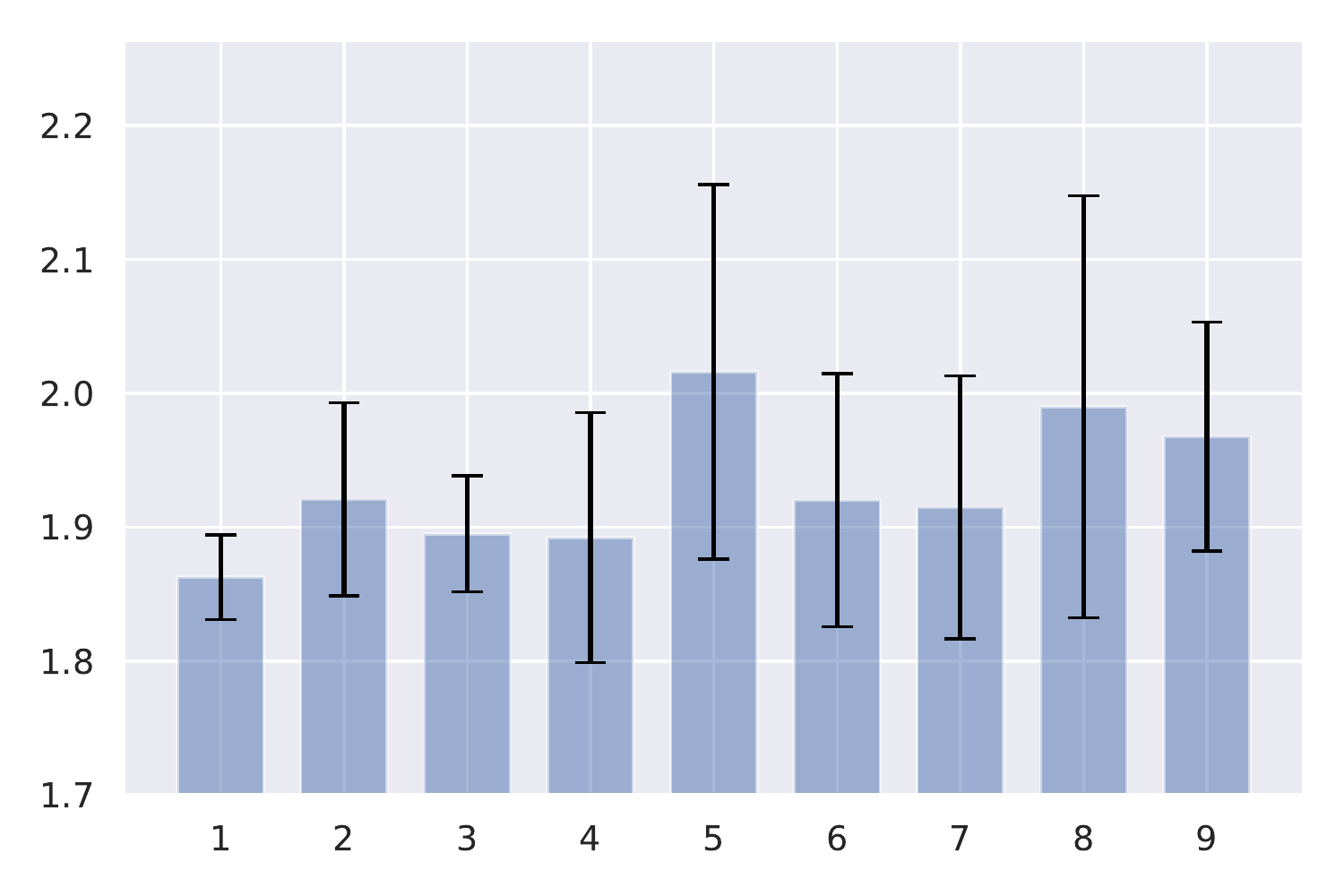}&
\includegraphics[height=\histheightc]{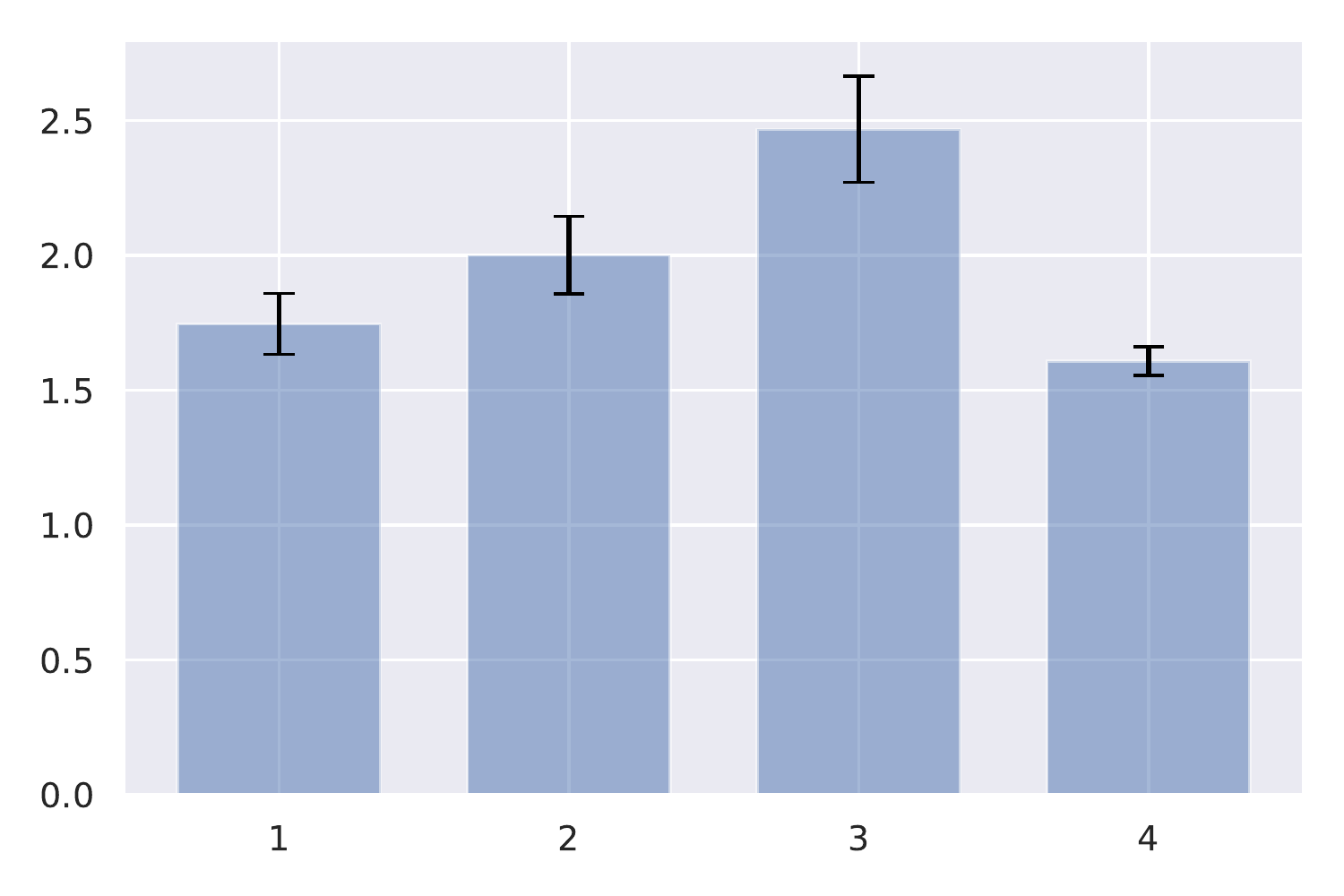}&
\includegraphics[height=\histheightc]{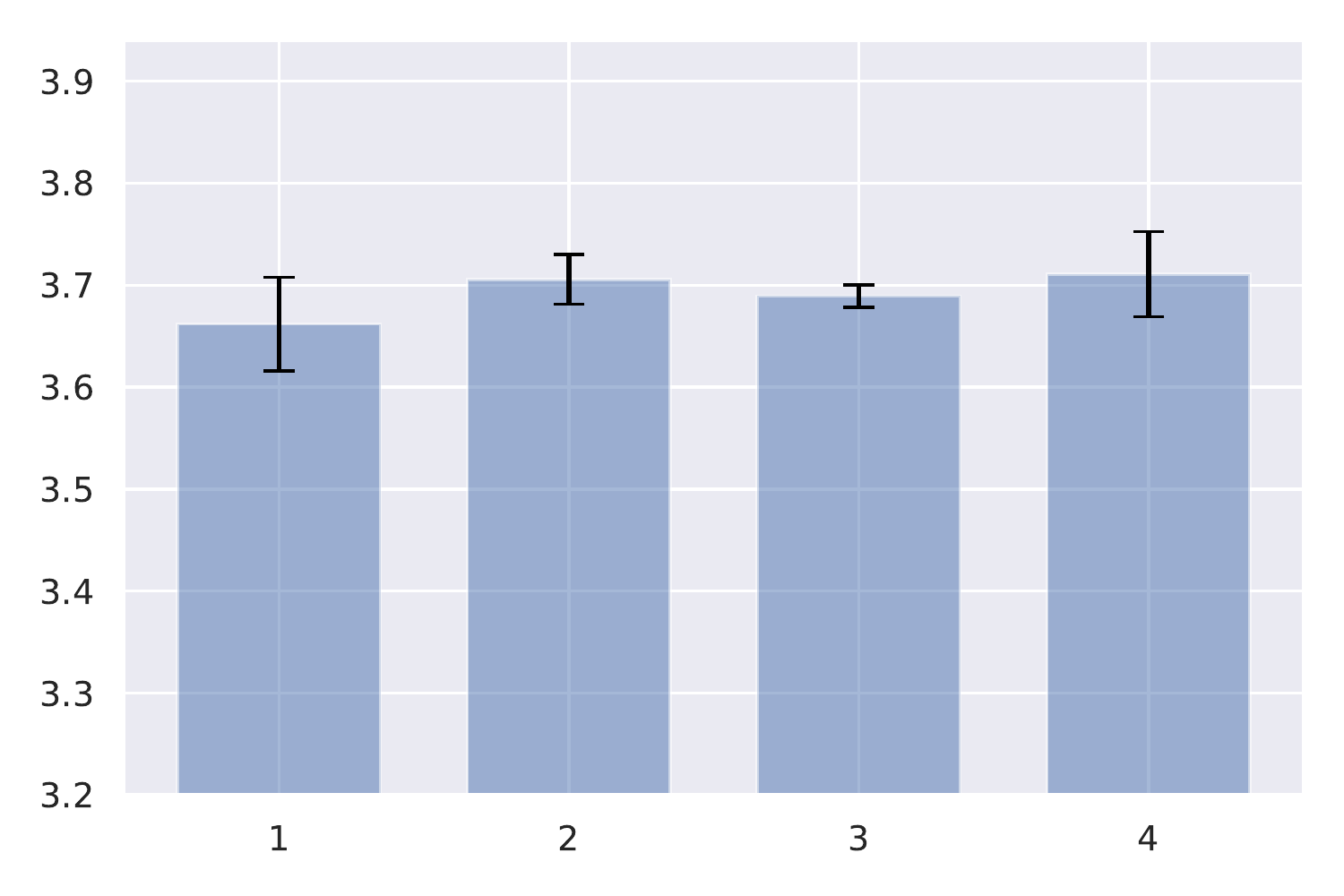}\\[-1.2ex]
\rowname{\makecell{noisy test client}}&
\includegraphics[height=\histheightc]{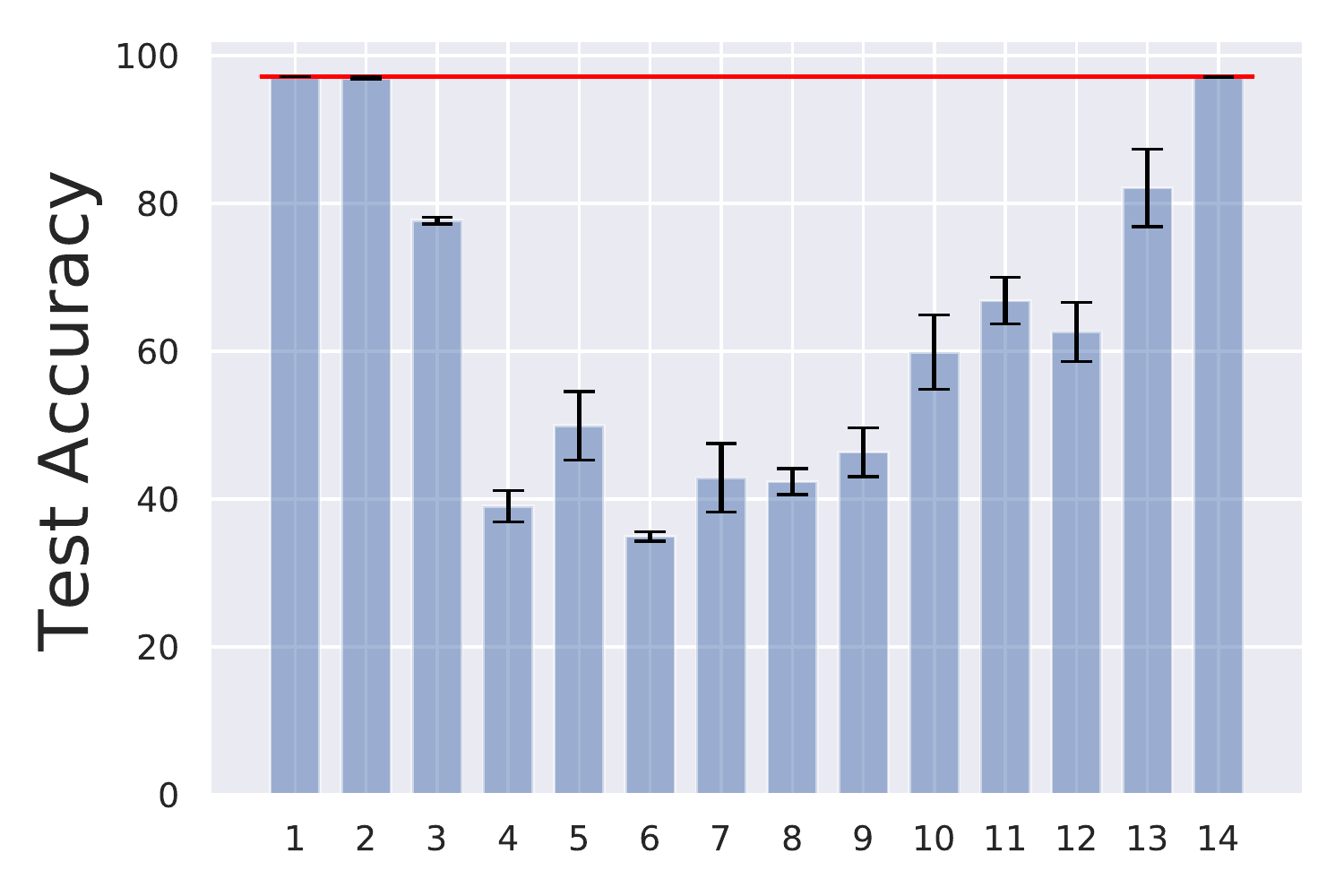}&
\includegraphics[height=\histheightc]{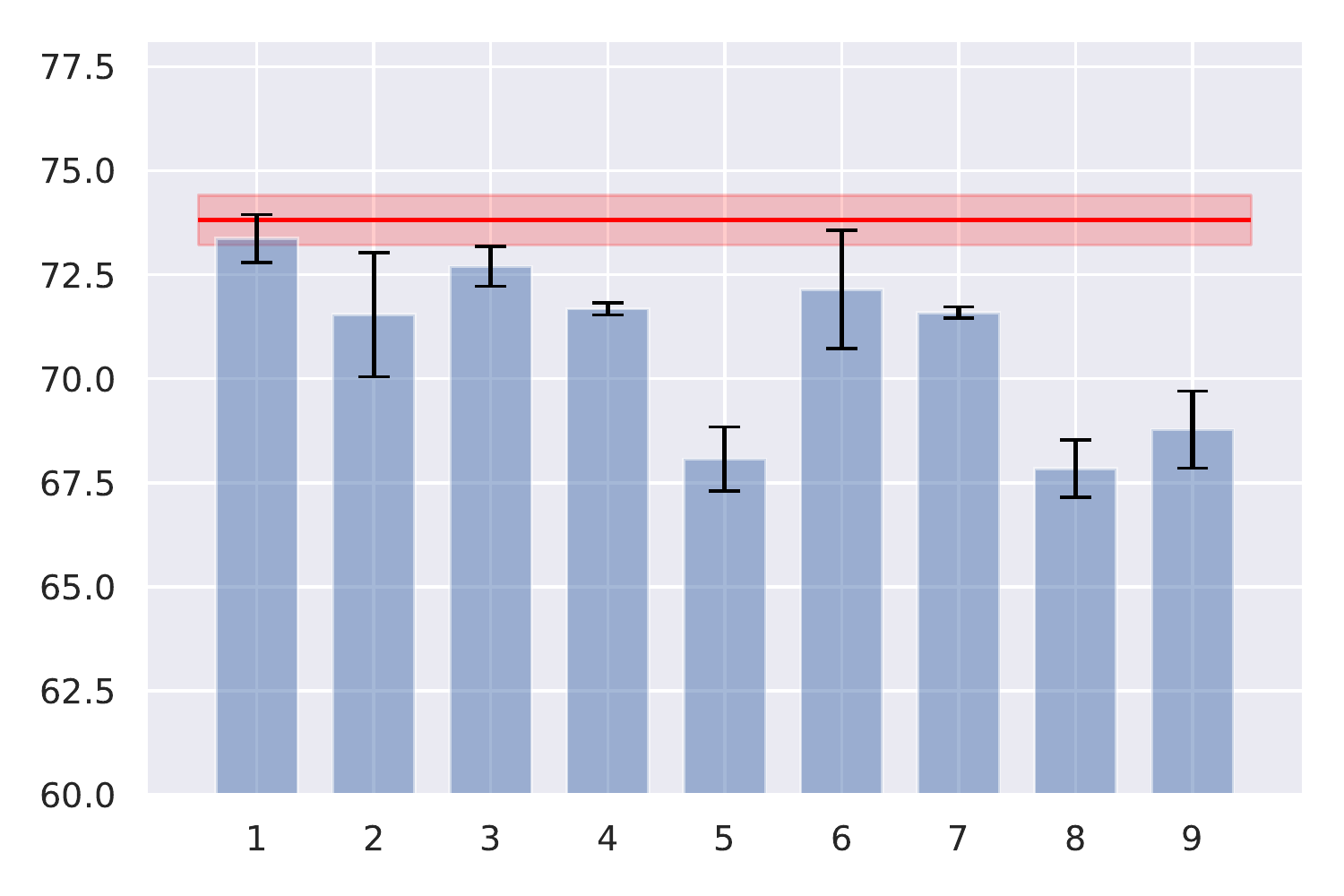}&
\includegraphics[height=\histheightc]{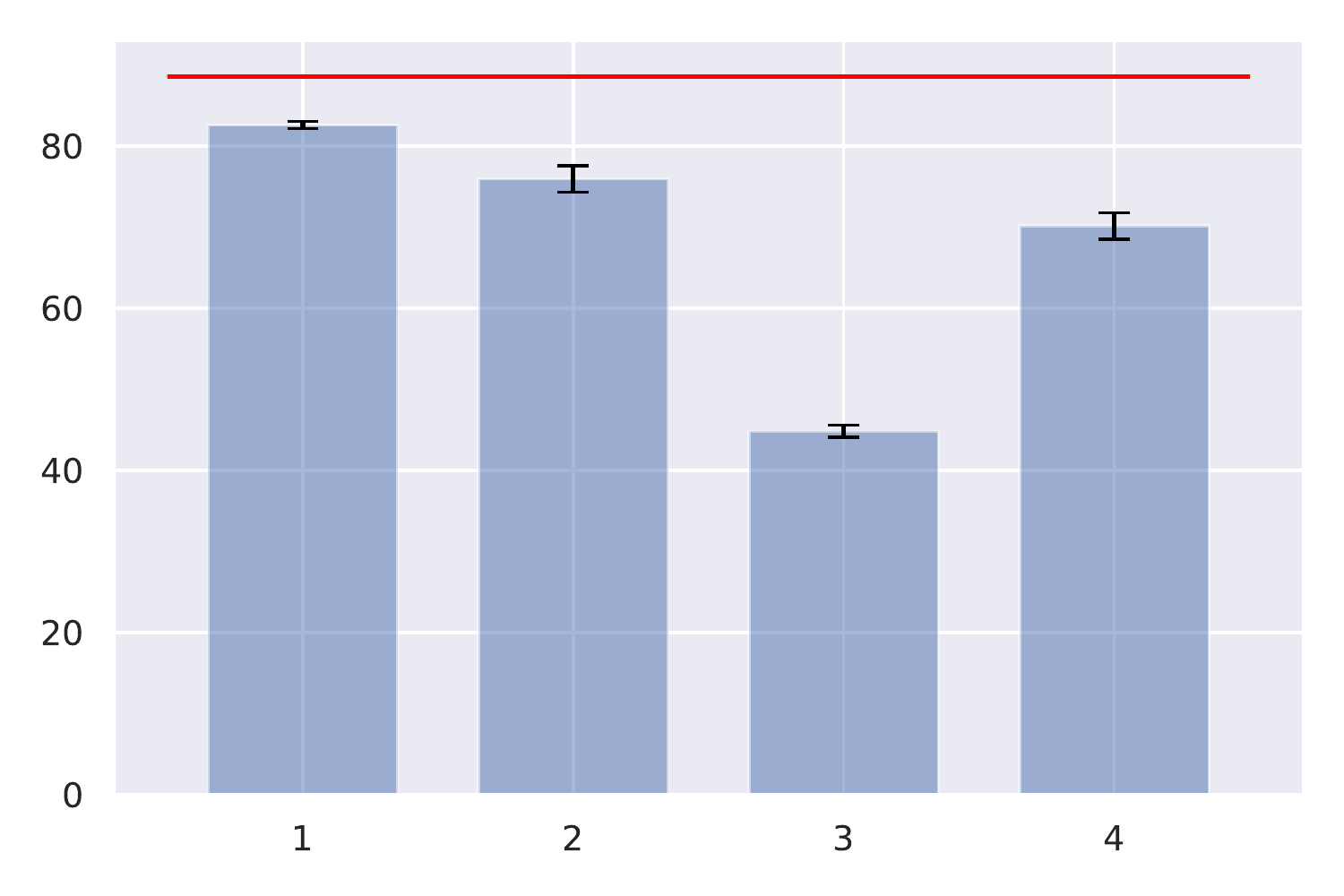}&
\includegraphics[height=\histheightc]{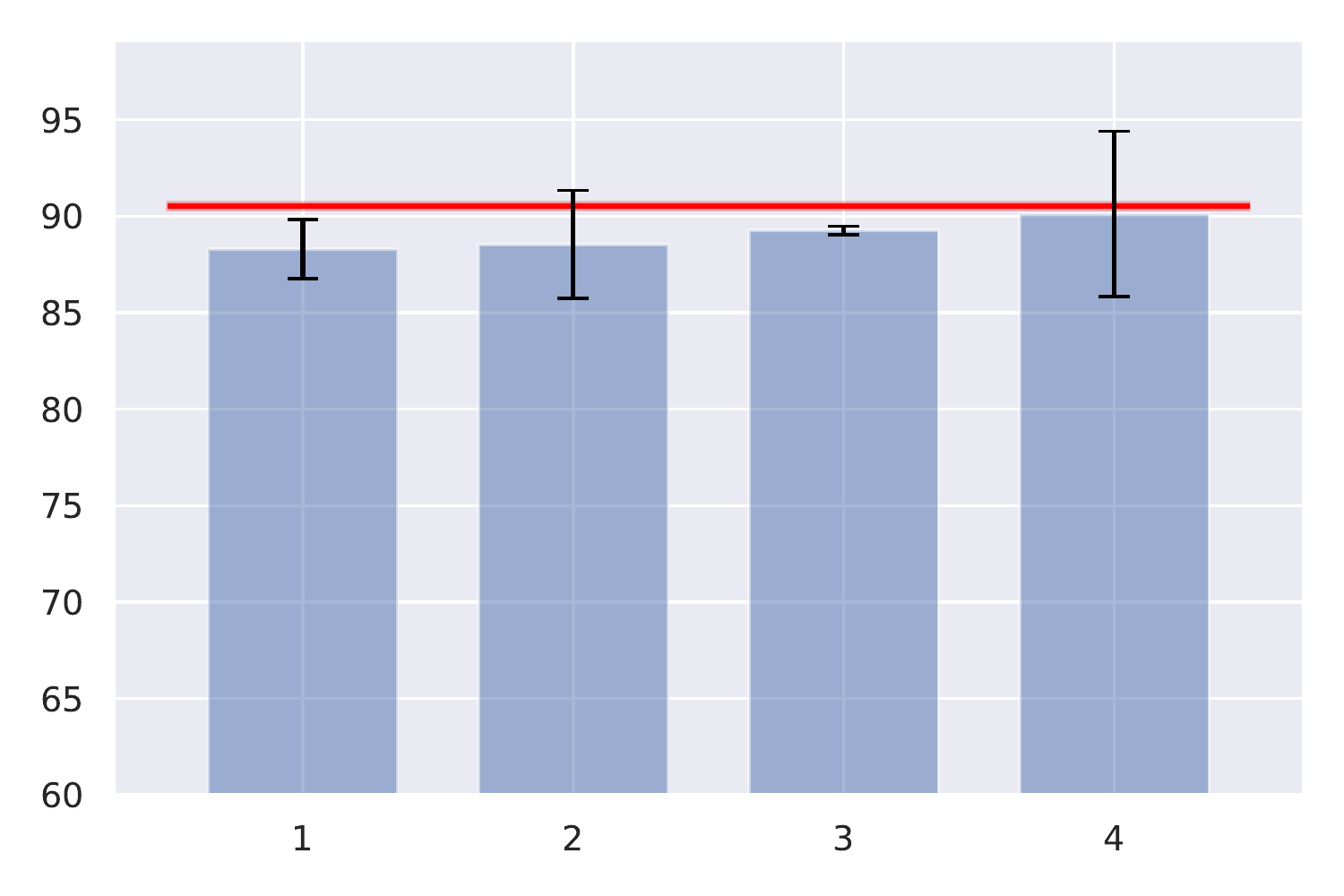}\\[-1.2ex]
\rowname{\makecell{denoising}}&
\includegraphics[height=\histheightc]{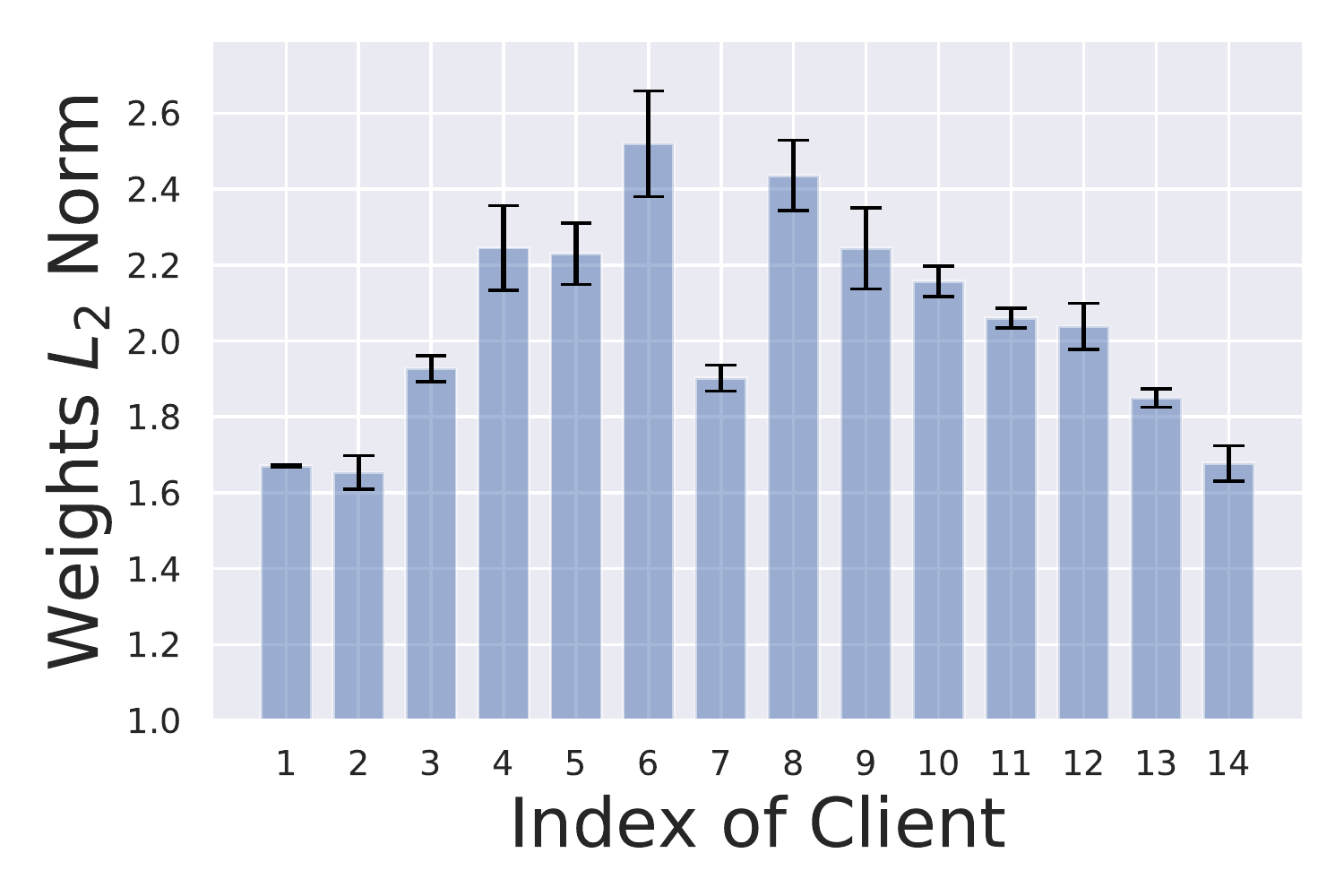}&
\includegraphics[height=\histheightc]{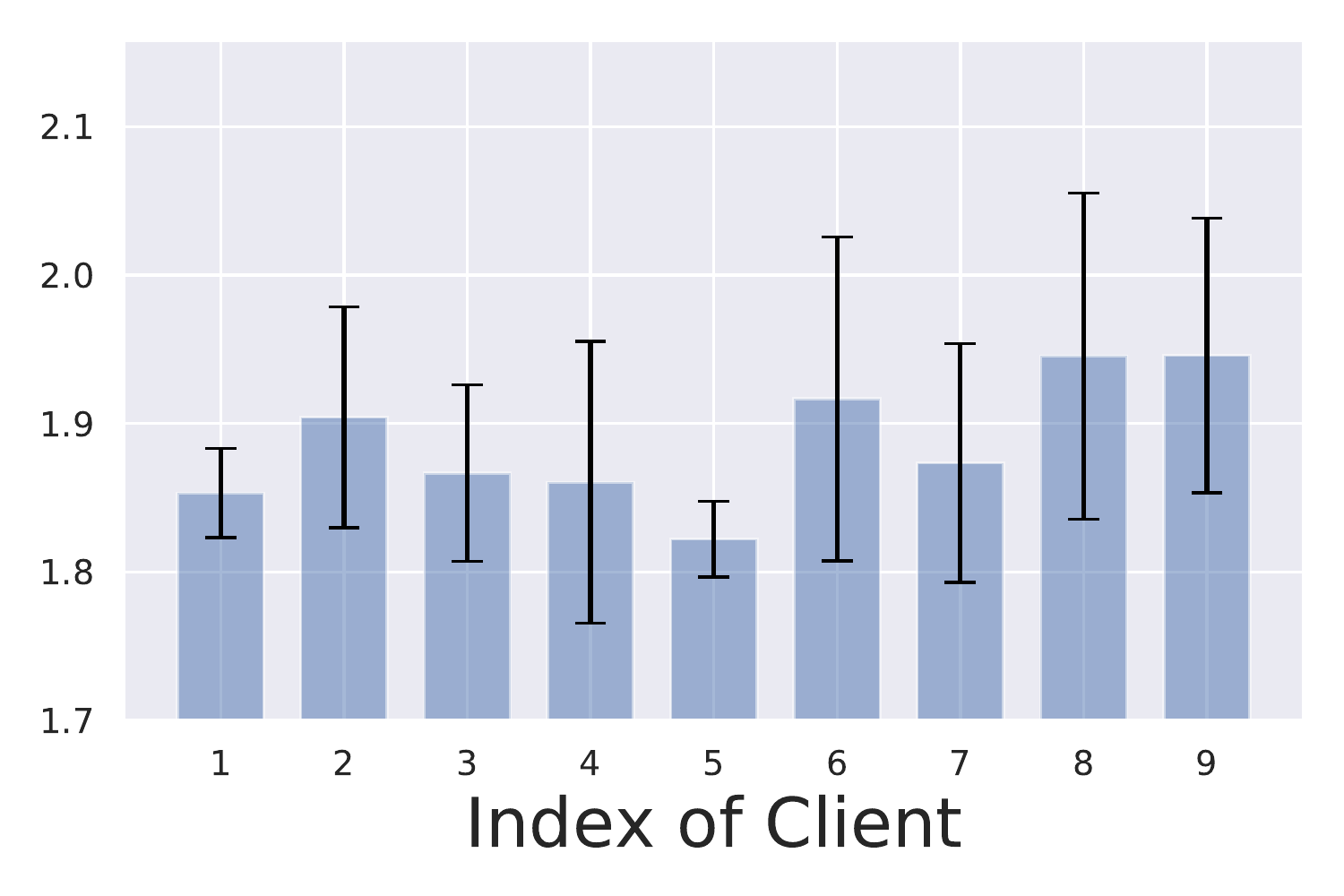}&
\includegraphics[height=\histheightc]{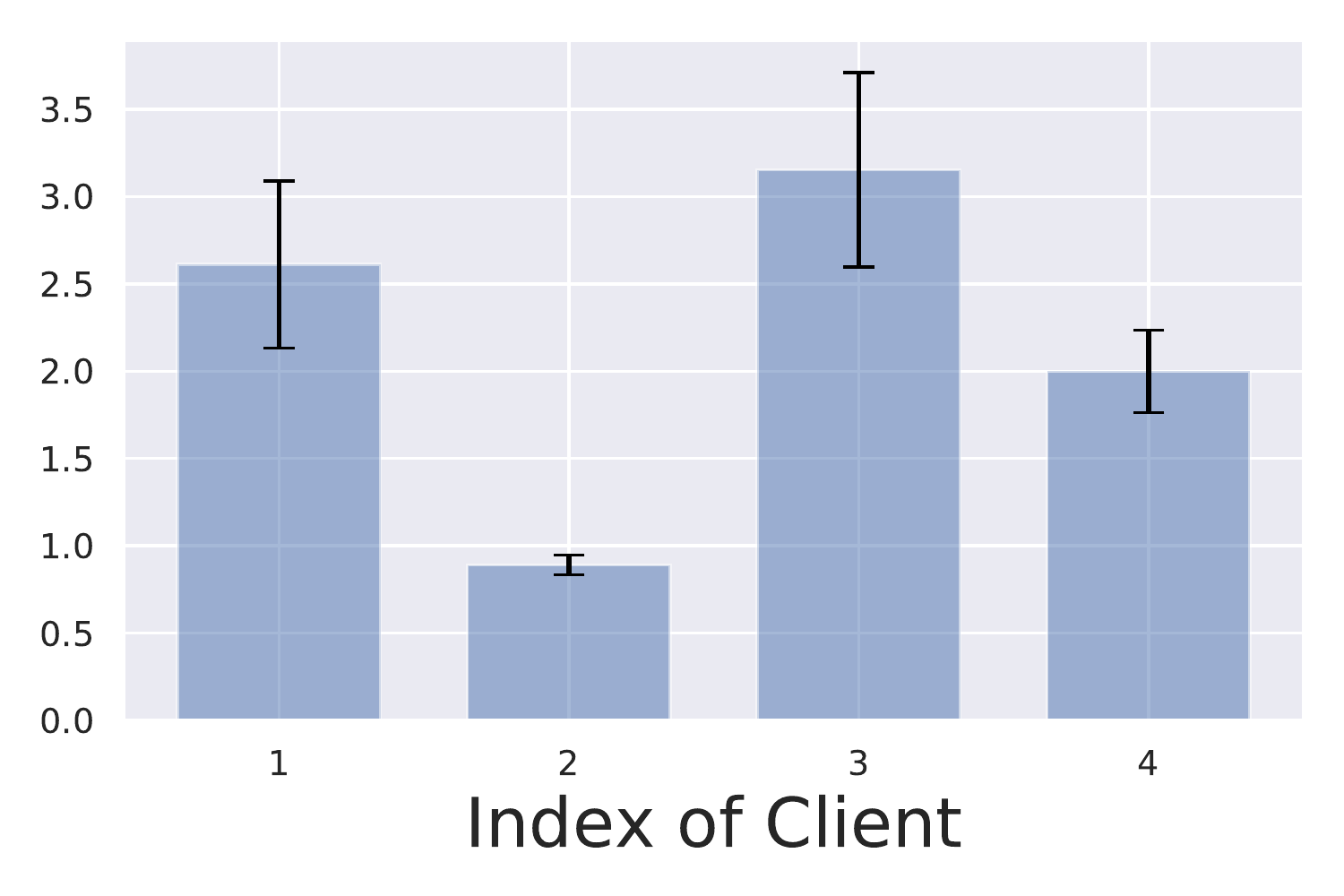}&
\includegraphics[height=\histheightc]{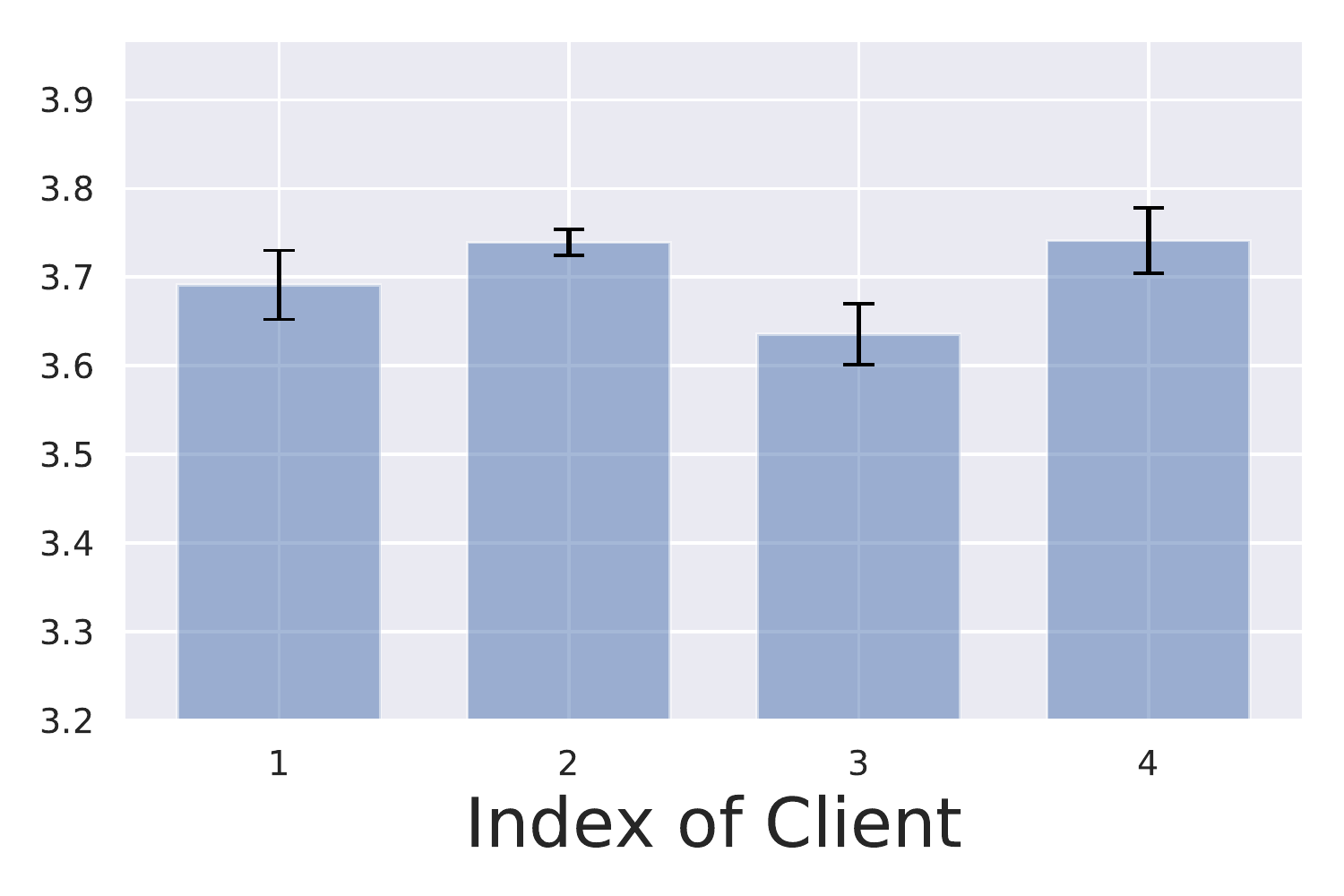}\\[-1.2ex]
\end{tabular}
\end{subtable}
}
\vspace{-1mm}
\caption{\small Client-level explainability of \name. Row 1 visualizes the input features. Row 2 shows the weights norm of linear heads. Row 3 shows the test accuracy when each client's test input features are perturbed (red line denotes the clean test accuracy). Row 4 shows the weights norm of linear heads under only one noisy client.}%
\label{fig:histogram}
\vspace{-6mm}
\end{figure*}
}

\begin{table}[h]
  \centering
  \vspace{-3mm}
   \caption{\small A larger number of local steps $\tau$ leads to better utility of \ouradmm under  \revise{client-level} DP $\epsilon=1$.   }
  \label{tab:tau_userdp}
\resizebox{0.5\linewidth}{!}{%
\begin{tabular}{ccccccccccccccc}\toprule
 \multicolumn{2}{c}{\mnist} & \multicolumn{2}{c}{\cifar} &  \multicolumn{2}{c}{\nus}  \\\cmidrule(lr){1-2} \cmidrule(lr){3-4}  \cmidrule(lr){5-6} 
 $\tau$ &  $\epsilon=1$ &  $\tau$ &  $\epsilon=1$ & $\tau$ &  $\epsilon=1$\\\midrule  
  1  & 90.63  &    1  &   48.19  &   1 & 79.38 \\
   5  & 90.84  & 10  & 61.65 &   3  & 82.58 \\
20  & 92.09  & 30  & 62.87 &  10 &   83.51
  \\\bottomrule  
\end{tabular}
}
\vspace{-3mm}
\end{table}

Methods w/o model splitting (FDML, \ouradmmjoint) generally performs better than methods w/ model splitting. This is mainly because the logits have a smaller dimension than the embeddings, and the total amount of noise added to the logits output is smaller than the embedding output; thus VFL w/o model splitting methods retain higher utility under DP. 

Additionally, the utility under \revise{client-level} DP {VFL} is not directly comparable to sample-level DP in centralized ML~\cite{abadi2016deep} or \revise{client-level} DP in standard (horizontal) FL~\cite{mcmahan2018learning} due to the \textit{unique properties} of VFL. 
For instances, \textbf{(1)} the \textit{dimension of DP-perturbed information} in VFL can be smaller (e.g., a batch of local embeddings or local logits) than the existing centralized learning or FL (e.g., gradients or model updates of a large model), which could lead to the higher utility under DP noise.
\textbf{(2)} The private local training set of VFL for each user has \textit{a  smaller raw feature dimension} (i.e., $1/M$ if features are divided evenly among $M$ clients) than the entire dataset (or local dataset) in the central setting (or horizontal FL) and it \textit{does not contain the labels}, which leads to a different dataset notion in DP definition. In our work, we follow existing privacy notions in VFL to protect each user's local training set~\cite{chen2020vafl,hu2019fdml} with proposed \revise{client-level} DP mechanisms. 

\vspace{-0.5mm}
\subsubsection{Utility under label DP (privacy of server labels)}
\label{sec:exp_labeldp}
To  protect the privacy of the labels in the server with formal privacy guarantee, we utilize the existing state-of-the-art label DP mechanism \alibi~\cite{malek2021antipodes}, which is originally proposed in centralized learning. 
We evaluate all methods under label DP with a privacy budget $\epsilon=2.8$ and $\epsilon=1.4$, which are obtained by adding Laplacian noise with noise parameter $\lambda_\mathrm{Lap}=1$ and $\lambda_\mathrm{Lap}=2$, respectively,  on the labels \textit{once} before VFL training, and we use randomized labels for training based on \alibi. In particular, \alibi post-processes the model predictions through Bayesian inference to  improve the model utility under noisy labels~\cite{malek2021antipodes}.
The results on \cref{tab:labeldp-utility} show that ADMM-based methods retain higher utility than gradient-based methods under the label DP. 
This could be due to two potential reasons: \textbf{(1)} the additional variables introduced by ADMM (i.e., auxiliary variables $\{z_j\}$ and dual variables $\{\lambda_j\}$) are dynamically adjusted during training, which might contribute to a more robust optimization~\cite{ding2019differentially} for VFL models (i.e., $\{W_k\}, \{\theta_k\}$) against label noises, and \textbf{(2)} multiple updates in each round could result in improved local models. As shown in \cref{tab:tau_labeldp}, more local steps $\tau$ can significantly enhance the utility of \ouradmm  under label-level DP $\epsilon=1.4$ ($\sim$10\% and $\sim$13\%  improvement  for \cifar and \nus, respectively).

 \vspace{-3mm}
\begin{table}[h]
  \centering
   \caption{\small A larger number of local steps $\tau$ leads to better utility of \ouradmm under  label-level DP $\epsilon=1.4$. }
  \label{tab:tau_labeldp}
\resizebox{0.5\linewidth}{!}{%
\begin{tabular}{ccccccccccccccc}\toprule
 \multicolumn{2}{c}{\mnist} & \multicolumn{2}{c}{\cifar} &  \multicolumn{2}{c}{\nus}  \\\cmidrule(lr){1-2} \cmidrule(lr){3-4}  \cmidrule(lr){5-6} 
 $\tau$ &  $\epsilon=1.4$ &  $\tau$ &  $\epsilon=1.4$ & $\tau$ &  $\epsilon=1.4$\\\midrule  
  1  &  92.28   &    1  &  46.08  &   1 &69.41   \\
   5  &  92.51  & 10  & 52.97 &   3  & 81.43  \\
20  & 92.8  & 30  &56.13   &  10 &   82.43 
  \\\bottomrule  
\end{tabular}
}
\vspace{-5mm}
\end{table}

\subsection{Client-level Explainability of \name}
\label{sec:exp_explainability}
{
\renewcommand{\thesubfigure}{\alph{subfigure}}

\begin{figure*}

\newlength{\tsnedimab}
\settoheight{\tsnedimab}{\includegraphics[width=.23\linewidth]{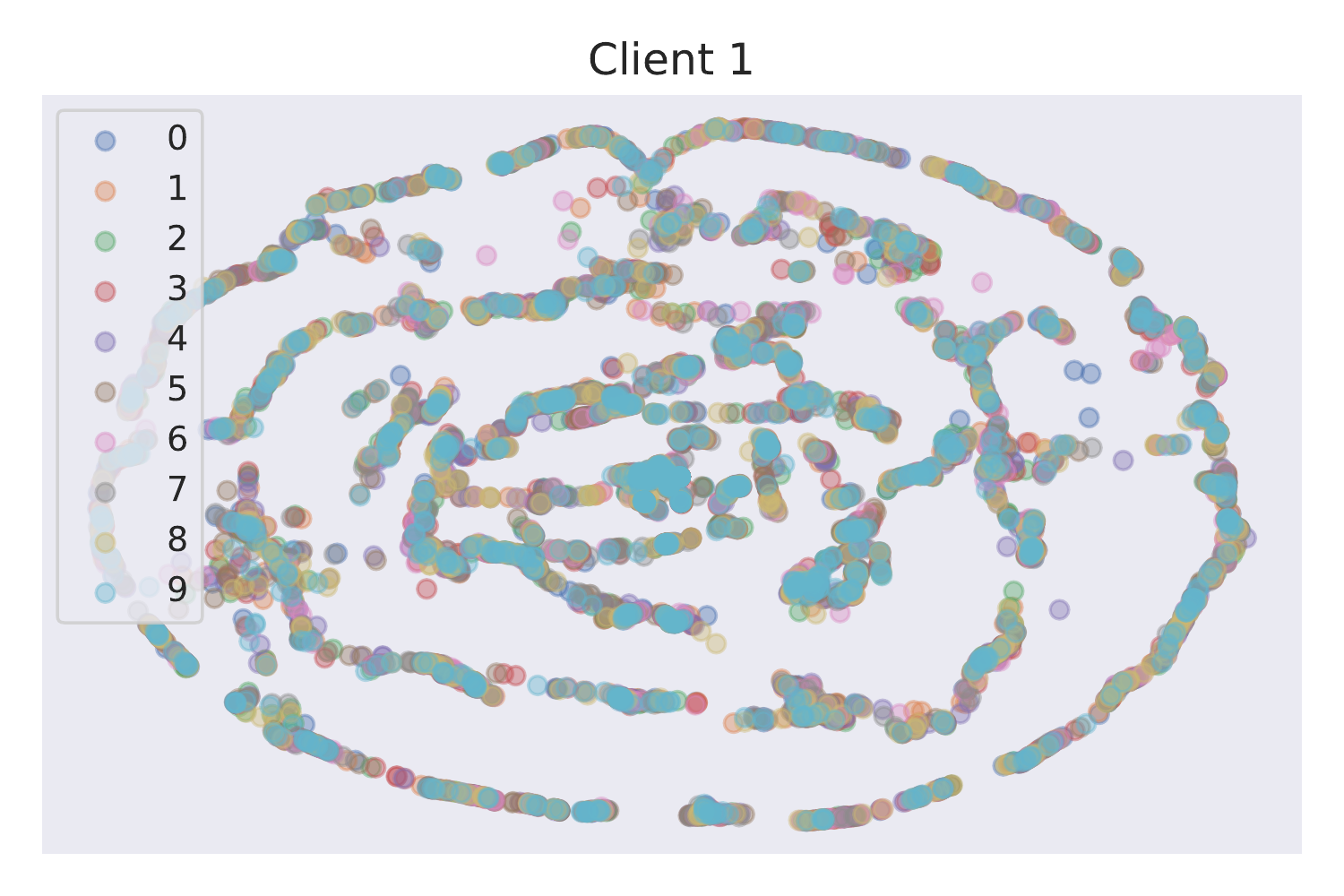}}%

\newlength{\tsneheightc}
\settoheight{\tsneheightc}{\includegraphics[width=.23\linewidth]{plots/acc/mnist_tsne_cli1.pdf}}%

\newlength{\legendheighttsne}
\setlength{\legendheighttsne}{0.2\tsneheightc}%

\newcommand{\rowname}[1]%
{\rotatebox{90}{\makebox[\tsnedimab][c]{\scriptsize #1}}}

\centering

{
\renewcommand{\tabcolsep}{10pt}

\centering
\begin{subtable}{0.9\linewidth}
\centering
\begin{tabular}{@{}p{5mm}@{}c@{}c@{}c@{}c@{}}
        & \makecell{\tiny{\mnist}}
        & \makecell{\tiny{\cifar}}
        & \makecell{\tiny{\nus}}
        & \makecell{\tiny{\modelnet}}
        \vspace{-3pt}\\
\rowname{\makecell{important}}&
\includegraphics[height=\tsneheightc]{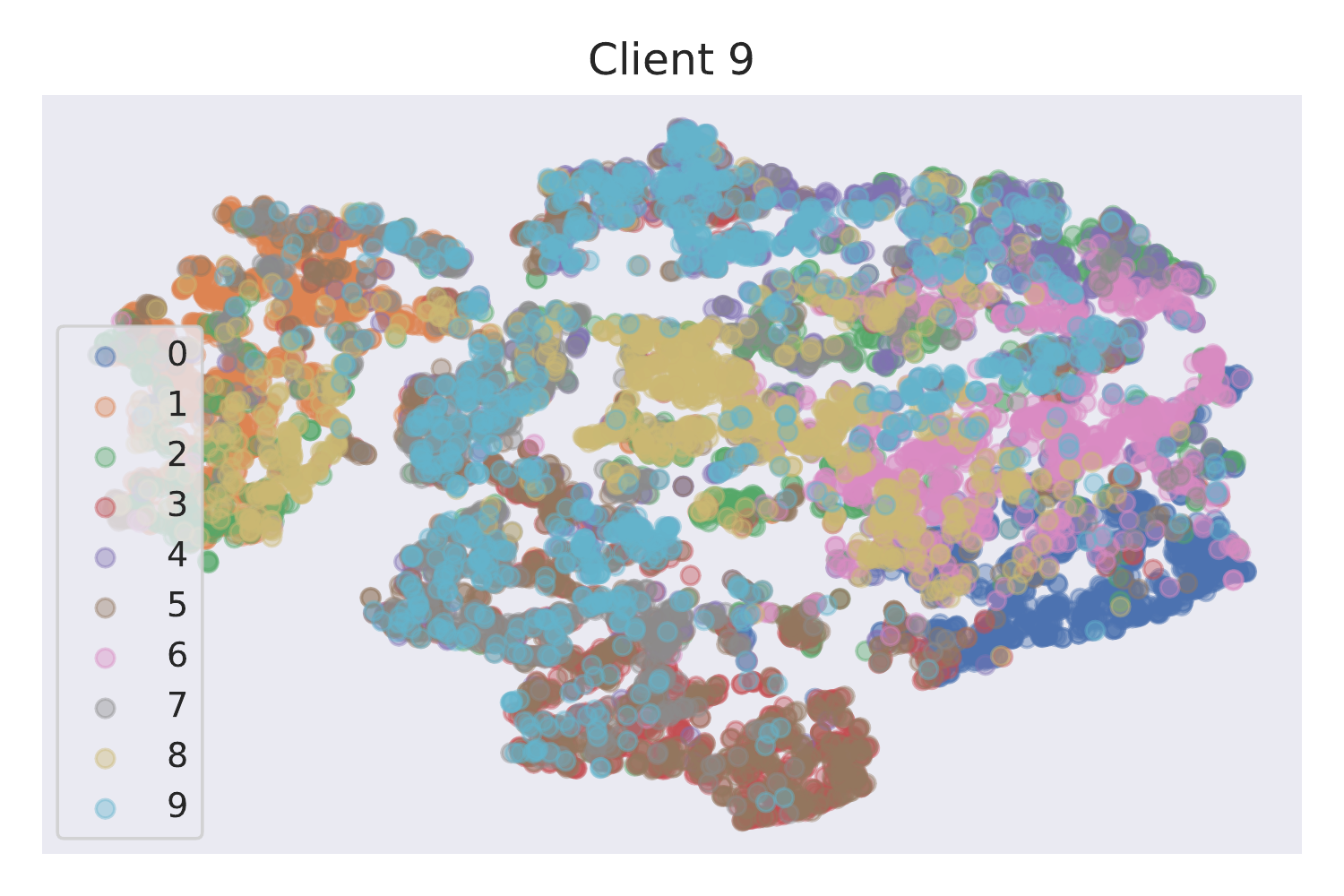}&
\includegraphics[height=\tsneheightc]{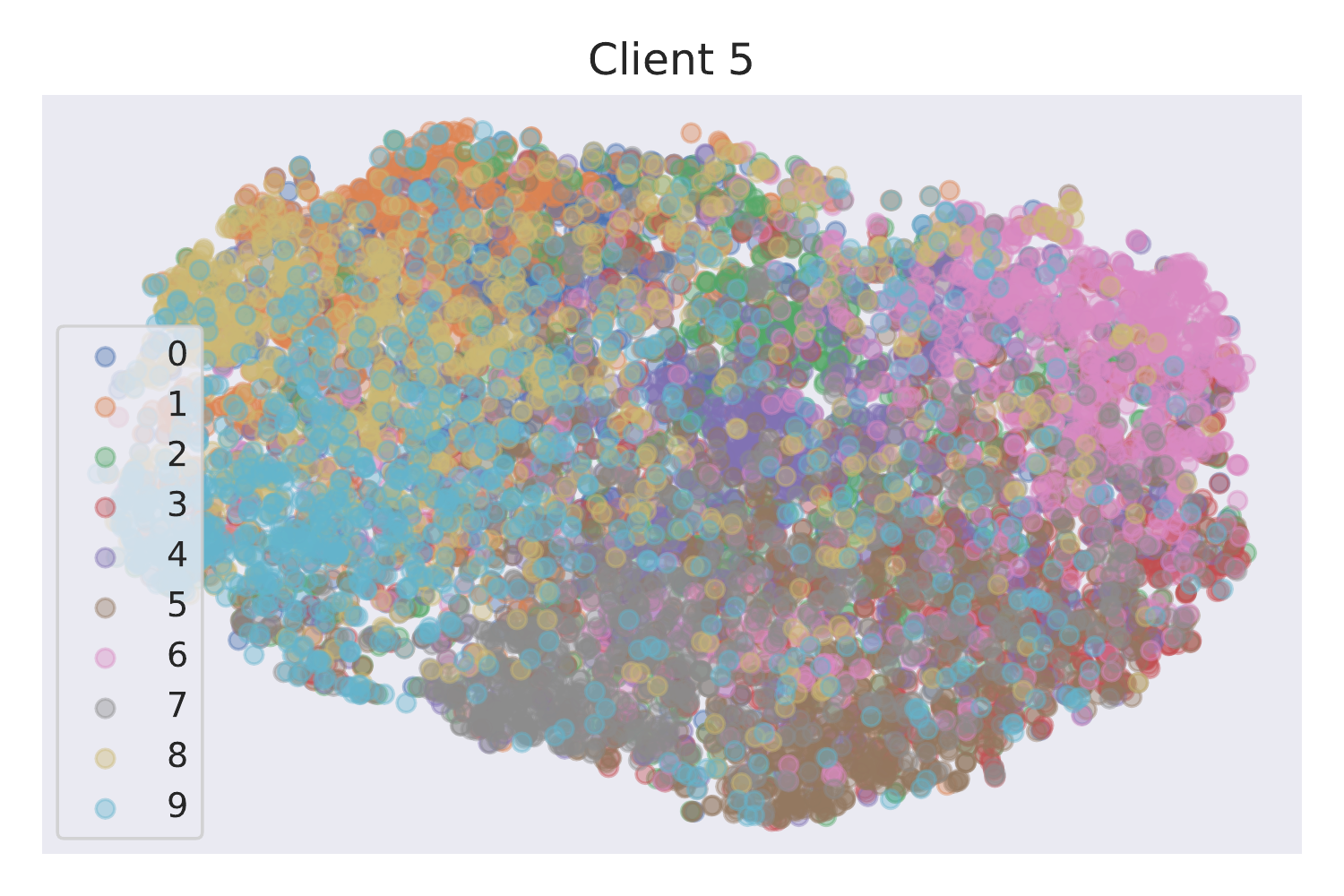}&
\includegraphics[height=\tsneheightc]{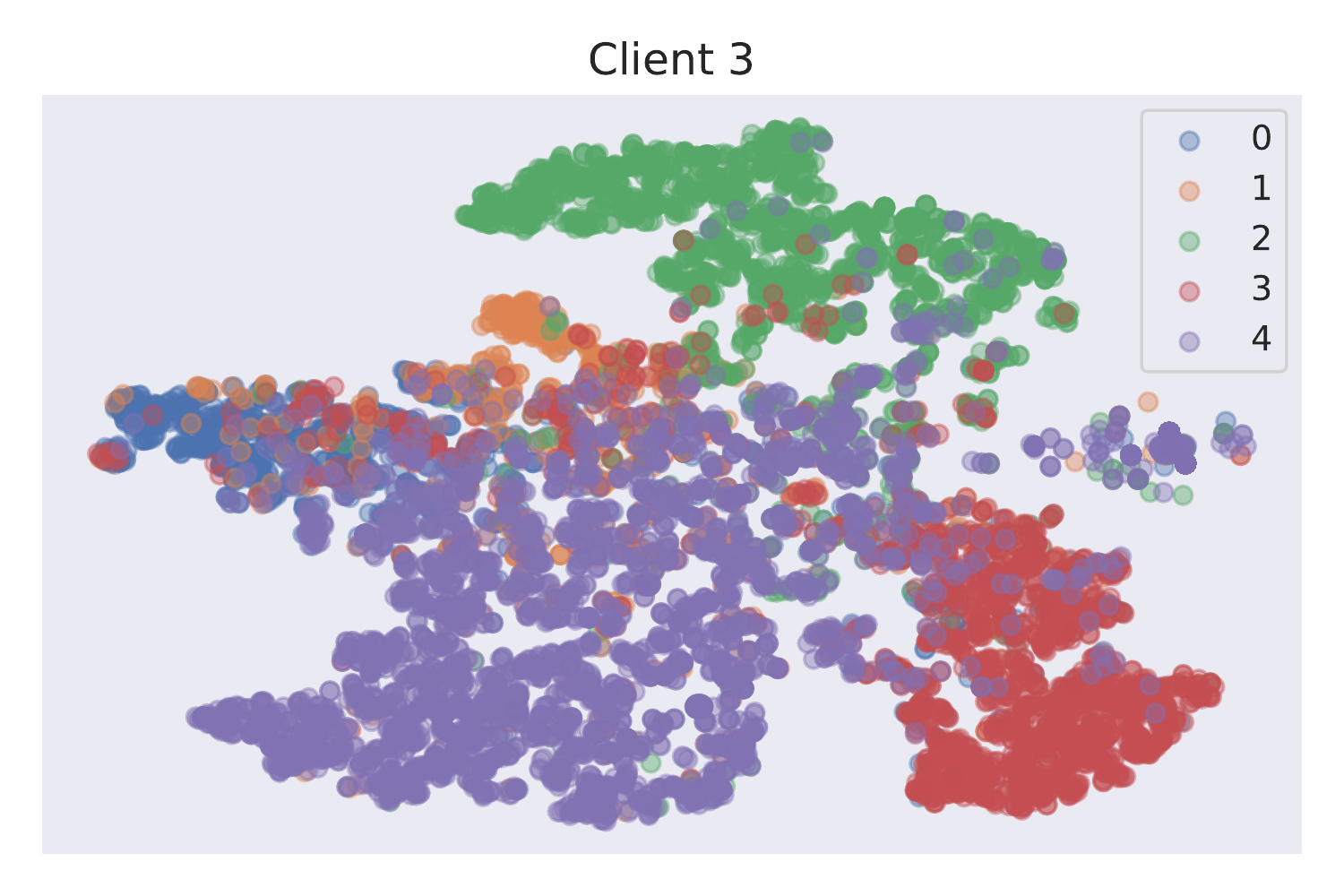}&
\includegraphics[height=\tsneheightc]{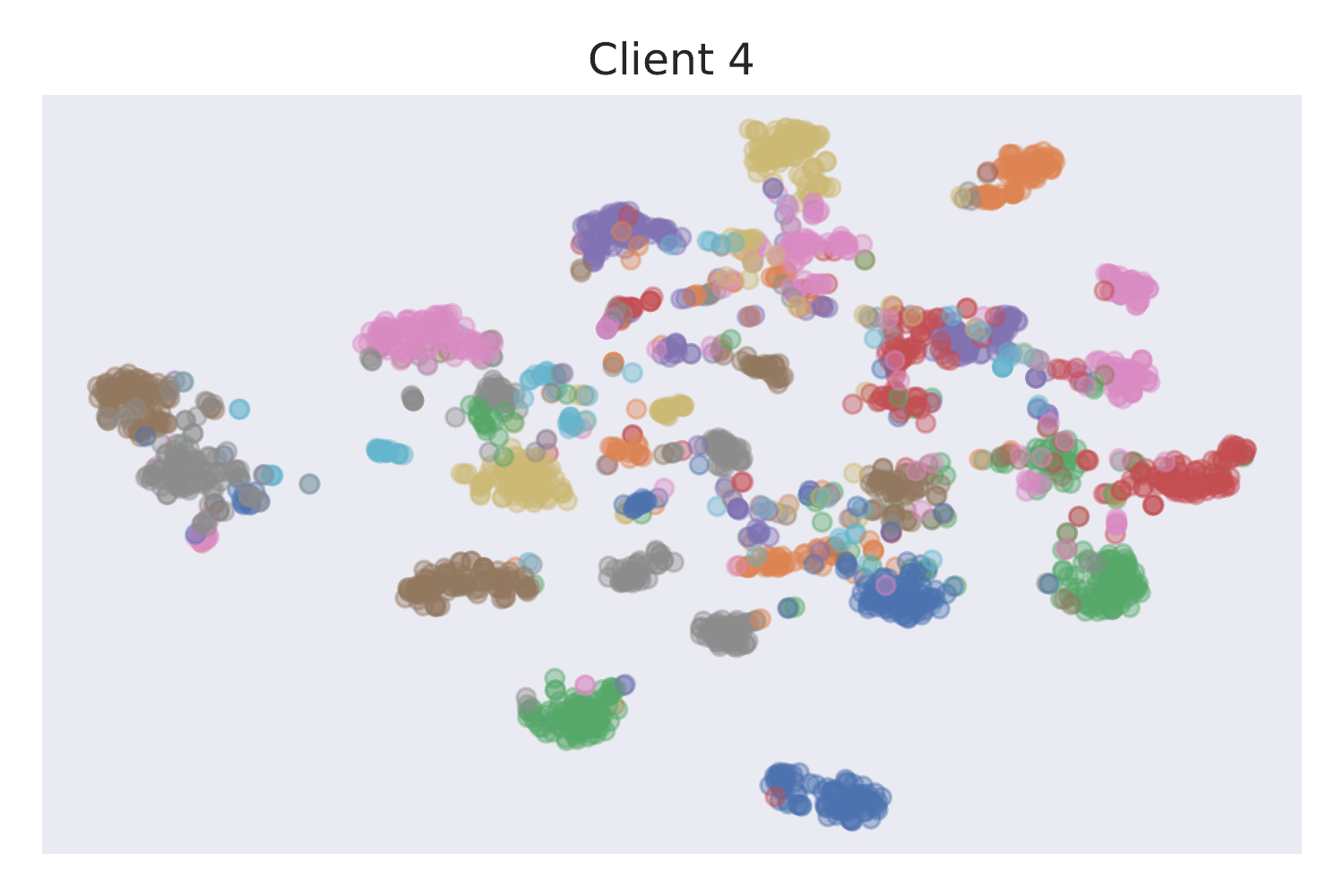}
\\[-1.2ex]
\rowname{\makecell{unimportant}}&
\includegraphics[height=\tsneheightc]{plots/acc/mnist_tsne_cli1.pdf}&
\includegraphics[height=\tsneheightc]{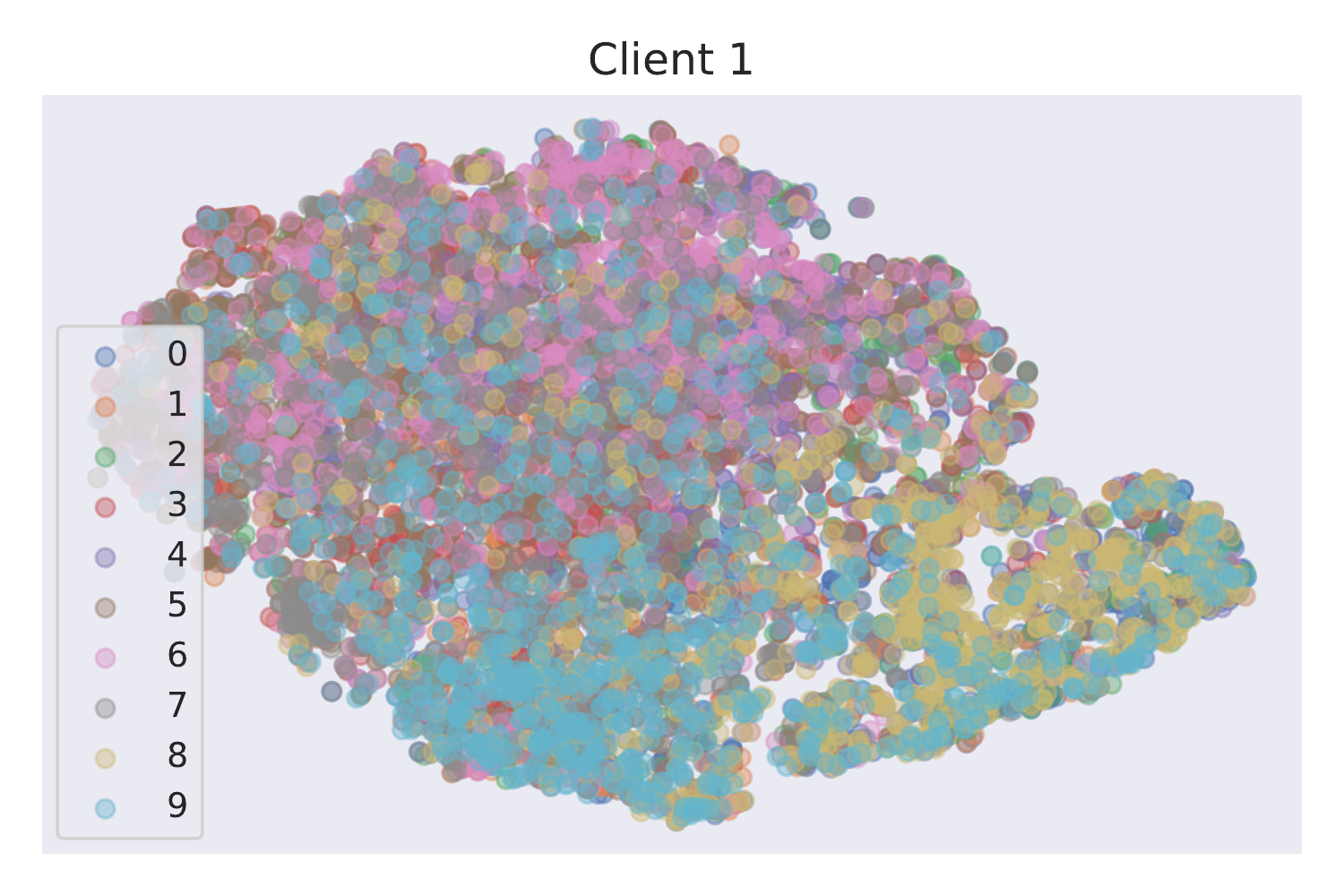}&
\includegraphics[height=\tsneheightc]{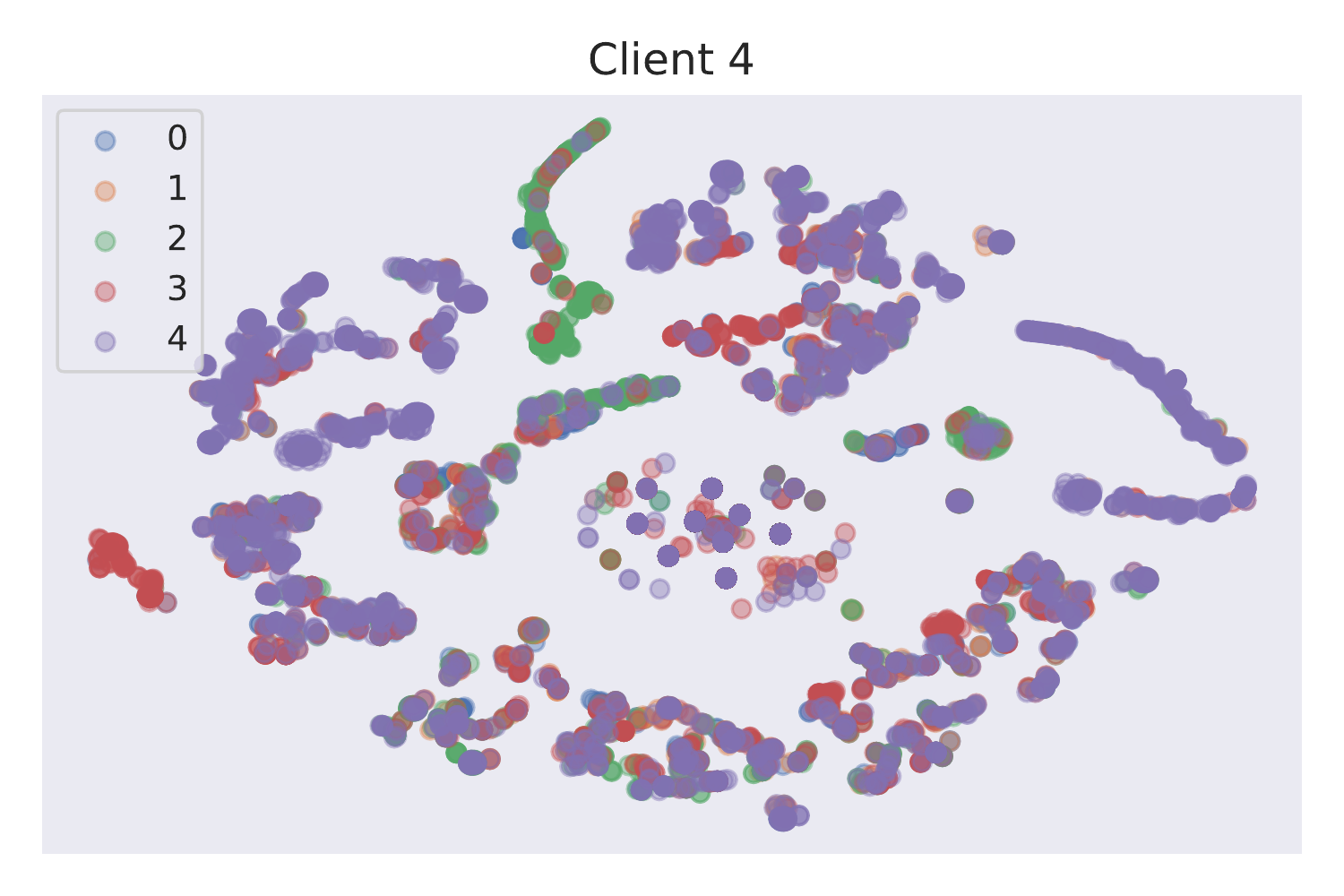}&
\includegraphics[height=\tsneheightc]{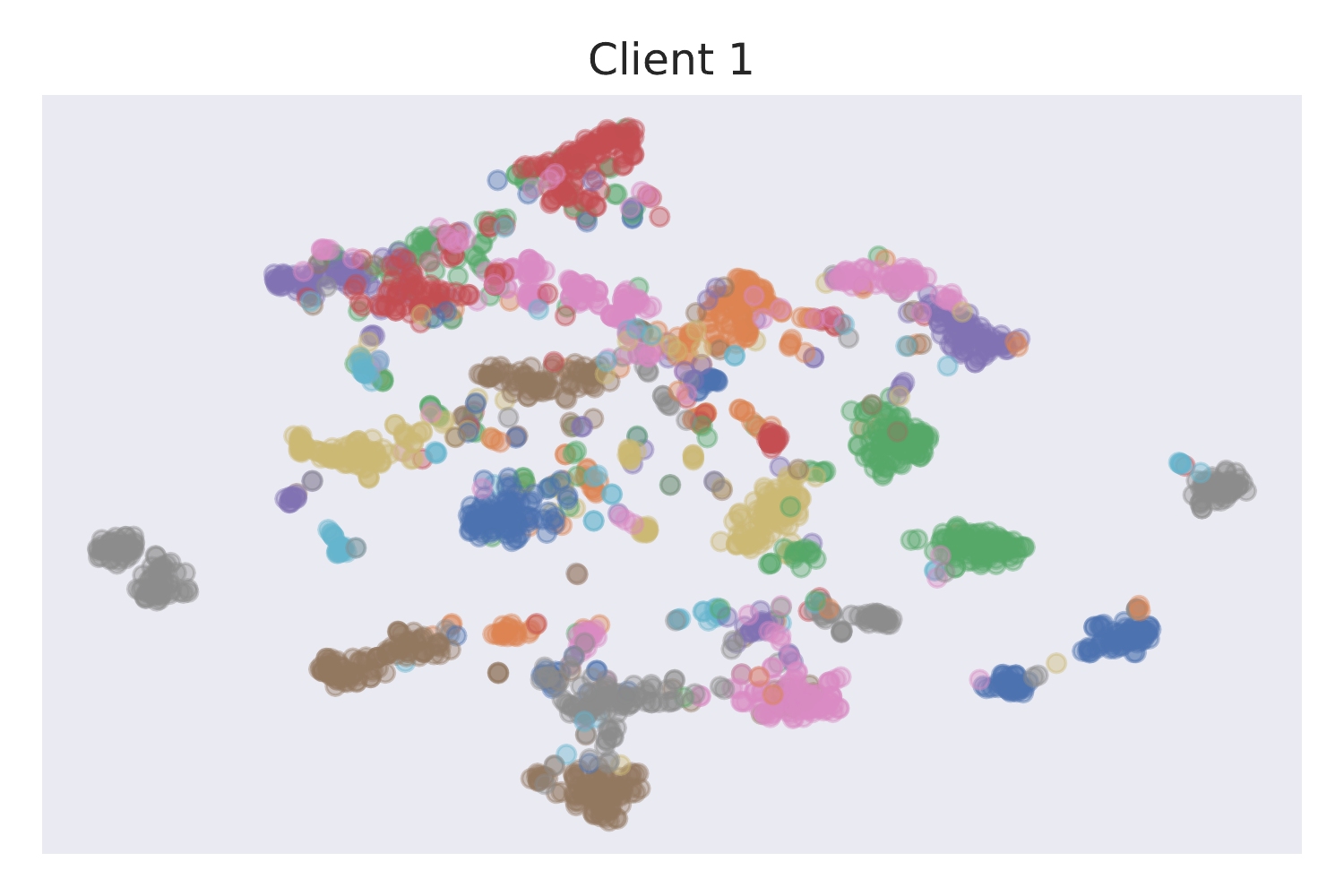}
\\[-1.2ex]
\end{tabular}
\end{subtable}

}

\caption{\small T-SNE visualizations of local embeddings from important client and unimportant client for \ouradmm. }%
\label{fig:tsne_more}
\vspace{-5mm}
\end{figure*}

}

In this section, we first visualize the local embeddings of clients, which are diverse, stemming from the distinct input features of clients. This also justifies the multi-head design of \name that can reweight the embeddings based on their importance.  
Then, we show that the weights norm of learned linear heads can indeed reflect the importance of local clients, which enables functionalities such as test-time noise validation, client denoising, and summarization.

\subsubsection{T-SNE of Local Embeddings}
In row 1 of \cref{fig:histogram}, we show the raw feautures of different clients on four datasets. 
The quality of features can vary among clients. For instance, in \mnist, since the digit always occupies the center, clients hold  black background pixels might not provide useful information for the classification task, and thus are less important.  
The T-SNE~\cite{tsne2008} visualizations in~\cref{fig:tsne_more} reveal that  important clients learn better local embeddings than unimportant clients on \mnist, \cifar and \nus.  
Specifically, in \nus, client \#3  produces linear separable local embeddings (left), which are better than client \#4's embeddings (right) that overlap different classes.
For \modelnet, since clients with multi-view data are of similar importance, their local embeddings exhibit similarities and demonstrate linear separability.
A scrutiny of these local embeddings confirms that the unique characteristic of input features in each client lead to  varied local embeddings. Consequently, we employ multiple heads as the server model, allowing us to account for the diverse feature quality across clients and aptly reweight the local embeddings.

\subsubsection{Client Importance} Given a trained \ouradmm{} model, we plot the weights norm of each client's corresponding linear heads in \cref{fig:histogram} row 2. 
Combining it with row 1, we find that \textit{the client with important local features indeed results in high weights}\footnote{Here the weights of clients refer to the weights of the client's corresponding linear head owned by the server.}. For example, clients \#6, \#7, \#8 in \mnist holding middle rows of images that contain the center of digits, have high weights, while clients \#1, \#14 holding the black background pixels have low weights. 
A similar phenomenon is observed on \cifar for client \#5 (center) and client \#1 (corner). 
On \cifar,  clients \#8, \#9 also have high weights, which is because the objects on \cifar also appear on the right bottom corner. 
On \modelnet, clients have complementary views of the same objects, so their features have similar importance, leading to similar weights norms.
Based on our observation, we conclude that \textit{the weights of linear heads can reflect the importance of local clients.} We use this principle to infer that, for \nus, 
the first 500 dim. of textual features have higher importance than other multimodality features, resulting in the high weights norm of client 3.

\subsubsection{Client Importance Validation via Noisy Test Client}
Given a trained \ouradmm model, we add Gaussian noise to the test local features
to verify the client-level importance indicated by the linear heads. For each time, we only perturb the features of one client and keep other clients' features unchanged.
The results in \cref{fig:histogram} row 3 show that \textit{perturbing the client with high weights affects more for the test accuracy}, which verifies that clients with higher weights are more important.

\subsubsection{Client Denoising}
We study the denoising ability of \name under training-time noisy clients. 
We construct one noisy client (i.e., client \#7, \#5, \#2, \#3 for \mnist, \cifar, \nus, \modelnet respectively) by adding Gaussian noise to its local features and re-train the \ouradmm model.
The obtained weights norm in \cref{fig:histogram} row 4 shows  that \ouradmm can \textit{automatically detect the noisy client and lower its weights} (compared to the clean one in row 2).
\cref{tab:denoise} in Appendix~\ref{app:exp_details} shows that \ouradmm outperforms baselines with faster convergence and higher accuracy under noisy clients. 

\subsubsection{Client Summarization}
Regarding client summarization, \textbf{(1)} we first rank the importance of clients according to their weights norm (\cref{fig:histogram}  row 2), then we select $u\%$ proportion of the most ``important" clients to re-train the \ouradmm model. We find that \textit{its performance is close to the one trained by all clients}.
\cref{tab:clisumm} shows that the test accuracy-drop of training with 50\% of the most important clients is less than 1\% on \mnist and \nus, and less than 4\% on \cifar; the accuracy-drop of training with 20\% of the most important clients is less than 10\% on all datasets. 
\textbf{(2)} We select $u\%$ proportion of the least important clients to re-train the model, and we find that its performance is significantly lower than the one trained with important clients, which indicates the effectiveness of \name for client selection. 
\textbf{(3)} For the multi-view dataset \modelnet, we find that the test accuracy of models trained with 12, 8, and 4 clients are similar, i.e., 91.04\%, 90.69\%, and 90.64\%, suggesting that a few views can already provide sufficient training information and the agents with multiview data are of similar importance which is also reflected by our linear head weights.

\begin{table}[h]
\centering
\vspace{-2mm}
\begin{minipage}[t]{1\linewidth}
  \caption{\small Functionality of client summarization enabled by \ouradmm.}
   \label{tab:clisumm}
   \centering
 \resizebox{0.8\columnwidth}{!}{
 \begin{tabular}{cccccccc}
    \toprule
  Client ratio &   Type &  {\mnist}  & {\cifar} & {\nus}  \\\midrule
{\footnotesize \multirow{1}{*}{$100\%$}}
& all   & 97.13  & 75.25 & 88.51      \\ 
[0.05em]\midrule[0.05em]
{\footnotesize \multirow{2}{*}{\makecell{$50\%$}}}
& important &   96.58  & 70.28  & 87.29    \\
& unimportant &  78.11  & 62.67  & 75.80    \\
[0.05em]\midrule[0.05em]
{\footnotesize \multirow{2}{*}{$20\%$} }
& important&  88.72  & 66.06  & 80.28   \\
& unimportant &  29.11  & 54.99  & 59.34     \\
\bottomrule
\end{tabular}
}
\end{minipage}
\vspace{-3mm}
\end{table}

\section{Conclusions}
We propose a VFL framework with multiple linear heads (\name) and an ADMM-based method (\ouradmm) for efficient communication. We provide the convergence guarantee for \ouradmm. 
We also introduce 
user-level differential privacy mechanism for \name and prove the privacy guarantee. 
Extensive experiments verify the superior performance of our algorithms under vanilla VFL and DP VFL and show that \name enables client-level explainability.

\section*{Acknowledgement}

The authors thank Yunhui Long, Linyi Li, Yangjun Ruan, Weixin Chen, and the anonymous reviewers for their valuable feedback and suggestions.  

This work is partially supported by the National Science Foundation under grant No. 1910100, No. 2046726, No. 2229876, DARPA GARD, the National Aeronautics and Space Administration (NASA) under grant No. 80NSSC20M0229, Alfred P. Sloan Fellowship, the Amazon research award, and the eBay research grant.

\bibliography{reference}
\bibliographystyle{plain}

\newpage

\newcommand\DoToC{%
  \startcontents
  \printcontents{}{1}{\textbf{Contents}\vskip3pt\hrule\vskip5pt}
  \vskip3pt\hrule\vskip5pt
}

\onecolumn

\appendix

The Appendix is organized as follows:
\begin{itemize}[leftmargin=*]
    \item Appendix~\ref{app:algos} provides algorithm details for \oursgd ~\cite{vepakomma2018split}(\cref{algo:splitlearn}) and \ouradmmjoint (\cref{algo:vimadmmjoint});
    \item   Appendix~\ref{app:conv-guatantee} provides the proofs for convergence guarantees in \cref{theorem:main-paper-conv};
    \item   Appendix~\ref{app:dp} provides the proofs for privacy guarantee in \cref{theo:dp_guarantee};
    \item  Appendix~\ref{app:exp_details} provides more details on experimental setups and the additional experimental results;
     \item   Appendix~\ref{app:discussion} provides additional discussion on ADMM and VFL.
\end{itemize}

\vspace{-3mm}
\subsection{Algorithm Details} 
\label{app:algos}
 
\subsubsection{\oursgd~\cite{vepakomma2018split}}
\label{app:algo_splitlearn}
At each communication round $t$,
 the server samples a set of data indices, $B(t)$, with batch size $|B(t)|= b $. Then we describe the key steps \oursgd (\cref{algo:splitlearn}) as follows: 

\textbf{(1) \emph{Communication from client to server.}} Each client $k$ sends a batch of embeddings $\{{ h_j^k}^{(t)}\}_{j \in B(t)}$ to the server, where ${ h_j^k}^{(t)}=  f(x_j^k; \theta_k^{(t)}), \forall j \in B(t)$. 

\textbf{(2) \emph{Sever updates server model $\theta_0$}. } According to VFL  objective in Eq.~\ref{eq:ms-obj}, the server model is updated  as:  
\begin{equation} \label{eq:update_W_sdg}
\theta_0^{(t +1)} \gets \theta_0^{(t)} - \eta \nabla_{\theta_0^{(t)}} \mathcal{L}_{\mathrm{VFL}} (\theta_0^{(t)}), \forall  k\in [M]
\end{equation}
where $\eta$ is the server learning rate, and
\begin{equation}
\nabla_{\theta_0^{(t)}} \mathcal{L}_{\mathrm{VFL}} (\theta_0^{(t)}) = 
\nabla_{\theta_0^{(t)}}  \left (\frac{1}{N} \sum_{j=1}^N  \ell ( [ {h_j^1}^{(t)}, \ldots,{h_j^M}^{(t)} ] , y_j; \theta_0^{(t)})   + \beta  \mathcal{R}(\theta_0^{(t)})  \right).
\end{equation}
Here $ [ {h_j^1}^{(t)}, \ldots,{h_j^M}^{(t)} ]$ denotes the concatenated local embeddings. 

\textbf{(3) \emph{Communication from server to client.} } Server computes gradients w.r.t each local embedding $ \nabla_{{h_j^k}^{(t)}}  \mathcal{L}_{\mathrm{VFL}} (\theta_0^{(t+1)})$ by the VFL objective in Eq.~\ref{eq:ms-obj}, where  
\begin{equation}\label{eq:compute_gradient_sgd}
\nabla_{{h_j^k}^{(t)}}  \mathcal{L}_{\mathrm{VFL}} (\theta_0^{(t+1)})= 
\nabla_{{h_j^k}^{(t)}}  \ell ( [ {h_j^1}^{(t)}, \ldots,{h_j^M}^{(t)} ] , y_j; \theta_0^{(t+1)})   ,  
\forall j\in B(t), k\in [M]
\end{equation}
Server sends gradients $\{  \nabla_{{h_j^k}^{(t)}}  \mathcal{L}_{\mathrm{VFL}} (\theta_0^{(t+1)})  \}_{j\in B(t)}$ to each client $k, \forall k\in [M]$.

\textbf{(4) \emph{Client updates local model parameters $\theta_k$}.} 
Finally, every client $k$  locally updates the model  parameters  $\theta_k$ according to the VFL objective in Eq.~\ref{eq:ms-obj} as follows: 
{ 
\begin{equation}\label{eq:update_theta_sgd}
\theta_{k}^{(t +1)} \gets  \theta_{k}^{(t  )} - \eta^k  \nabla_{\theta_{k}^{(t )}}   \mathcal{L}_{\mathrm{VFL}} (\theta_0^{(t+1)}) , \forall k\in [M]
\end{equation}}%
where $\eta^k$ is the local learning rate for client $k$, and
\begin{equation}
\nabla_{\theta_{k}^{(t )}}   \mathcal{L}_{\mathrm{VFL}} (\theta_0^{(t+1)}) = \frac{1}{N} \sum_{j=1}^N \nabla_{\theta_{k}^{(t )}}   {{h_j^k}^{(t)}}  
 \nabla_{{h_j^k}^{(t)}}  \mathcal{L}_{\mathrm{VFL}} (\theta_0^{(t+1)})   + \beta 
\nabla_{\theta_{k}^{(t )}}  \mathcal{R}(\theta_k^{(t )}) 
\end{equation}

These four steps of \oursgd are summarized in \cref{algo:splitlearn}.

\begin{small}
\begin{algorithm}[h]
\footnotesize
\begin{algorithmic}[1]
 \STATE {\bfseries Input:}{number of communication rounds $T$, number of clients $M$, number of training samples $N$, batch size $b$ , input features $\{\{ x_j^1\}_{j=1}^{N}, \{ x_j^2\}_{j=1}^{N}, \ldots, \{ x_j^M\}_{j=1}^{N}\} $, the labels $\{y_j\}_{j=1}^{N}$, local model $\{\theta_k\}_{k=1}^M $; linear heads $\{W_k\}_{k=1}^M $;
server learning rate $\eta$; client learning rate $\{\eta^k\}_{k=1}^M$;
}
\FOR {communication round $ t \in [T]$}  
     \STATE Server samples a set of data indices $B(t)$ with $|B(t)|= b $  \;
    \FOR {client $ k \in [M]$}   
           \STATE \textbf{generates} a local training batch $ \{x_j^k\}_{j\in B(t) }$ 
        \STATE \textbf{computes} local embeddings  ${h_j^k}^{(t)}\gets  f(x_j^k; \theta_k) ,\forall j\in B(t) $
         \STATE \textbf{sends} local embeddings $ \{{h_j^k}^{(t)}\}_{j\in B(t) }$ to the server
     \ENDFOR
     \STATE Server \textbf{updates} server model $\theta_0^{(t +1)}     $  by Eq.~\ref{eq:update_W_sdg}    \;
     \STATE Server \textbf{computes} gradients w.r.t embeddings $ \nabla_{{h_j^k}^{(t)}}  \mathcal{L}_{\mathrm{VFL}} (\theta_0^{(t+1)})    $ by Eq.~\ref{eq:compute_gradient_sgd}  $,\forall j\in B(t)$  \;
     \STATE Server \textbf{sends}  gradients $\{  \nabla_{{h_j^k}^{(t)}}  \mathcal{L}_{\mathrm{VFL}} (\theta_0^{(t+1)})  \}_{j\in B(t)}$ to each client $k, \forall k\in [M]$ \;
    \FOR {client $ k \in [M]$}  
         \STATE \textbf{updates} local model  $\theta_{k}^{(t +1)} $     by Eq.~\ref{eq:update_theta_sgd}   
      \ENDFOR
 \ENDFOR
 \caption{\small \oursgd{}~\cite{vepakomma2018split}
 }\label{algo:splitlearn}
 \end{algorithmic}
\end{algorithm}
 \end{small}

\vspace{-3mm}
\subsubsection{\ouradmmjoint}
\label{app:algo_admmjoint}
 At each communication round $t$,
 the server samples a set of data indices, $B(t)$, with batch size $|B(t)|= b $. Then we describe the key steps of \ouradmmjoint (Algorithm~\ref{algo:vimadmmjoint}) as follows: 
 
\textbf{(1) \emph{Communication from client to server.}} Each client $k$ sends a batch of local logits $\{{ o_j^k}^{(t)}\}_{j \in B(t)}$ to the server, where ${ o_j^k}^{(t)}=  f(x_j^k; \theta_k^{(t)})  W_{k}^{(t)}, \forall j \in B(t)$ 

\textbf{(2) \emph{Sever updates auxiliary variables $\{z_j\}$}.} 
After receiving the local logits from all clients, the server updates the auxiliary variable for each sample $j$  as:
\begin{equation}\label{eq:update_z_wosplit}
z_{j}^{(t)} =  \underset{z_j}{\operatorname{argmin}}  \quad  \ell (z_j,y_j) -     {  \lambda_{j}^{(t-1)}}^\top   z_{j}  +\frac{\rho}{2}\left\| \sum_{k=1}^M  {o_j^k}^{(t)} - z_j \right\|^{2} , \forall j \in B(t)
\end{equation}
Since the optimization problem in Eq.~\ref{eq:update_z_wosplit} is convex and differentiable with respect to $z_j$, we use the L-BFGS-B algorithm~\cite{zhu1997algorithm} to solve the minimization problem. 

\textbf{(3) \emph{Sever updates dual variables $\{\lambda_j\}$}.} After the updates in Eq.~\ref{eq:update_z_wosplit}, the server updates the dual variable  for each sample $j$    as:
\begin{equation}\label{eq:update_lambda_wosplit}
\lambda^{(t)}_{j} =  \lambda^{(t-1)}_{j}+\rho\left(\sum_{k=1}^M {o_j^k}^{(t)} -z^{(t)}_{j}\right) , \forall j \in B(t)
\end{equation}

\textbf{(4) \emph{Communication from server to client.} }After the updates in Eq.~\ref{eq:update_lambda_wosplit}, we define a residual variable ${s_j^{k}}^{(t+1)} $ for each sample $j$ of $k$-th client, which provides supervision for updating local model:
\begin{equation}\label{eq:def_residual_wosplit}
    {s_j^{k}}^{(t)} \triangleq {z_j}^{(t)} - \sum_{i\in [M] , i \neq k}  {o_j^i}^{(t)}   
\end{equation} 
The server sends the dual variables $  \{ \lambda^{(t)}_{j} \}_{j \in B(t)}$ and the residual variables $ \{  {s_j^{k}}^{(t)}\}_{j \in B(t)} $ of all samples to each client $k$.

\textbf{(5) \emph{Client updates linear head $W_k$ and local model $\theta_k$ alternatively.}} The linear head of each client is locally updated  as:  
\begin{small}
\begin{equation} \label{eq:update_W_wosplit}
W_{k}^{(t +1)}= \underset{W_k}{\operatorname{argmin}}  \quad \beta   \mathcal{R}(W_k)  
 + \frac{1}{b} \sum_{ j \in B(t)} {\lambda_{j}^{(t)}}^\top f(x_{j_k}; \theta_k^{(t)} )  W_{k} +   \sum_{ j \in B(t)}\frac{\rho}{2b} \left\|   {s_j^{k}}^{(t)}  -   f(x_{j_k}; \theta_k^{(t)} )   W_{k}  \right\|_F^2, \forall k\in [M]
\end{equation}
\end{small}
Each client updates the local model parameters  $\theta_k$ as follows:
{ \small
\begin{equation}\label{eq:update_theta_wosplit}
\theta_{k}^{(t +1)}= \underset{\theta_{k}}{\operatorname{argmin}}  \quad \beta  \mathcal{R}(\theta_k) + \frac{1}{b}  \sum_{ j \in B(t)}    {\lambda_{j}^{(t)}}^\top     f(x_{j_{k}}; \theta_k ) { W_{k}^{(t+1)}  } +  \sum_{ j \in B(t)}  \frac{\rho}{2b} \left\|   {s_j^{k}}^{(t)} - f(x_{j_k}; \theta_k ) { W_{k}^{(t+1)} }   \right\|_F^2.
\end{equation}}%

Due to the nonconvexity of the loss function of DNN, 
we use $\tau$ local steps of SGD 
to update $W_k$ and $\theta_k$ alternatively at each round with the objective of Eq.~\ref{eq:update_W_wosplit} and  Eq.~\ref{eq:update_theta_wosplit}. Specifically, at each local step, we first update  $W_k$ and then update $\theta_k$.
 
These five steps of \ouradmmjoint are summarized in Algorithm~\ref{algo:vimadmmjoint}.

\begin{small}
\begin{algorithm}[t!]
 \caption{\small \ouradmmjoint \colorbox[RGB]{239,240,241}{(with differentially privacy) }}\label{algo:vimadmmjoint}
\footnotesize
\begin{algorithmic}[1]
 \STATE {\bfseries Input:}{number of communication rounds $T$, number of clients $M$, number of training samples $N$, batch size $b$ , input features $\{\{ x_j^1\}_{j=1}^{N}, \{ x_j^2\}_{j=1}^{N}, \ldots, \{ x_j^M\}_{j=1}^{N}\} $, the labels $\{y_j\}_{j=1}^{N}$, local model $\{\theta_k\}_{k=1}^M $; linear heads $\{W_k\}_{k=1}^M $;
auxiliary variables $ \{ z_{j} \}_{j=1}^{N} $;
dual variables $ \{ \lambda_{j} \}_{j=1}^{N} $;
\colorbox[RGB]{239,240,241}{ noise parameter $\sigma$, clipping constant $C$ }} 
\FOR {communication round $ t \in [T]$} 
     \STATE Server samples a set of data indices $B(t)$ with $|B(t)|= b_s $  
    \FOR {client $ k \in [M]$}
         \STATE \textbf{generates} a local training batch $ \{x_j^k\}_{j\in B(t) }$ 
         \STATE\textbf{computes} local logits  ${ o_j^k}^{(t)}=  f(x_j^k; \theta_k^{(t)})  W_{k}^{(t)}  , \forall j \in B(t) $
         \STATE   \colorbox[RGB]{239,240,241}{ \textbf{clips and perturbs} local logit matrix $\medmath{ \{{ o_j^k}^{(t)} \}_{j\in B(t) } \gets \mathtt{Clip}\left( \{{ o_j^k}^{(t)} \}_{j\in B(t) }  , C\right) + \mathcal{N}\left(0, \sigma^{2} C^{2}\right)} $  }
         \STATE \textbf{sends} local logits $ \{{o_j^k}^{(t)}\}_{j\in B(t) }$ to the server
    \ENDFOR
     \STATE Server \textbf{updates} auxiliary variables $ z_{j}^{(t)} $ by Eq.~\ref{eq:update_z_wosplit}$, \forall j\in B(t)$
     \STATE Server \textbf{updates} dual variables $ \lambda^{(t)}_{j}$ by Eq.~\ref{eq:update_lambda_wosplit}  $, \forall j\in B(t)$
     \STATE Server \textbf{computes} residual variables $ {s_j^{k}}^{(t)}$ by Eq.~\ref{eq:def_residual_wosplit}  $, \forall j\in B(t), k\in [M]$
     \STATE Server \textbf{sends} $  \{ \lambda^{(t)}_{j} \}_{j \in B(t)}$ , $ \{  {s_j^{k}}^{(t)}\}_{j \in B(t)} $  to each client $k, \forall k\in [M]$ 
    \FOR {client $ k \in [M]$}  
        \FOR {local step  $e \in [\tau]$} 
         \STATE \textbf{updates} local linear head  $W_{k}^{(t +1)}$ by Eq.~\ref{eq:update_W_wosplit} with SGD
         \STATE \textbf{updates} local model  $\theta_{k}^{(t +1)}$ by Eq.~\ref{eq:update_theta_wosplit} with SGD
        \ENDFOR
    \ENDFOR
\ENDFOR
 \end{algorithmic}

\end{algorithm}

 \end{small}

\subsection{Convergence Guarantees}\label{app:conv-guatantee}

\subsubsection{Additional Notations and Supporting Lemmas}

To help theoretical analysis, we denote the objective functions in \cref{eq:update_z},  \cref{eq:update_W} , \cref{eq:update_theta} as 

\begin{small}
\begin{equation}
\begin{aligned}
 h(z_j) & =  \ell (z_j) -     {  \lambda_{j}^{(t)}}^\top   z_{j}  +\frac{\rho}{2}\left\| \sum_{k=1}^M  {f(x_j^k;\theta_k^{(t+1)} )} W_k^{(t+1)}- z_j \right\|_{F}^{2} \\
g_k(W_k) & =  \beta_k  \mathcal{R}_{k}(W_k)  
 +  \frac{1}{N} \smashoperator{\sum_{ \substack{ j \in [N]}}}  {\lambda_{j}^{(t)}}^\top {f(x_j^k;\theta_k^{(t)})}  W_{k} \\
&\quad +  \frac{\rho}{2N}  \sum_{ j \in [N]} \left\|  \sum_{ \substack{i\in [M] ,\\ i \neq k } }  f(x_j^i;\theta_i^{(t)})  {W_{i}}^{(t)} +   {f(x_j^k;\theta_k^{(t)})}  W_{k} - {z_j}^{(t)}\right\|_F^2 \\ 
 q_k( \theta_k) & =  \beta_k  \mathcal{R}_{k}(\theta_k) +   \frac{1}{N} \sum_{ j \in [N]}    {\lambda_{j}^{(t)}}^\top    f(x_j^k; \theta_k )  { W_{k}^{(t+1)}  } \\
&\quad + \frac{\rho}{2N}  \sum_{ j \in [N]}   \left\|   \sum_{ \substack{i\in [M] ,\\ i \neq k}}  {f(x_j^i;\theta_i^{(t)})}{W_{i}}^{(t+1)} +   f(x_j^k; \theta_k ) { W_{k}^{(t+1)} } -  {z_j}^{(t)}   \right\|_F^2
\end{aligned}
\end{equation}
\end{small}

Before delving into the main proofs, we introduce the bellow supporting lemmas.

\begin{lemma}\label{lemma:zgrad_lambda}
\begin{equation}
    \nabla \ell (z_j^{(t)}) = \lambda_j^{(t)}
\end{equation}
\end{lemma}
\begin{proof}
According to the optimality of $z_j^{(t)}$ \cref{eq:update_z}   
\begin{align}
    \nabla \ell( z_j^{(t)}) - 
    {  \lambda_{j}^{(t-1)}}  - \rho \left(   \sum_{k=1}^M  {f(x_j^k;\theta_k^{(t)} )} W_k^{(t)} -  z_j^{(t)} \right)=0 , \forall j \in B(t)
\end{align}
then invoke \cref{eq:update_lambda}
$
\lambda^{(t)}_{j} =  \lambda^{(t-1)}_{j}+\rho\left(\sum_{k=1}^M {f(x_j^k;\theta_k^{(t)})} W_k^{(t)} -z^{(t)}_{j}\right)
$, so we have $ \nabla \ell( z_j^{(t)})  = \lambda^{(t)}_{j} $.
\end{proof}

\begin{lemma}\label{lemma:z_lambda}
\begin{equation}
\| \lambda_j^{(t)} -  \lambda_j^{(t-1)} \| \leq L \|  z_j^{(t)} -  z_j^{(t-1)}\|
\end{equation}
\end{lemma}
\begin{proof}
According to \cref{assume:loss_lip} and \cref{lemma:zgrad_lambda}, we have
\begin{align}
    \| \lambda_j^{(t)} -  \lambda_j^{(t-1)} \| =    \|  \nabla \ell( z_j^{(t)})  -  \nabla \ell( z_j^{(t-1)}) \| 
    \leq L \|  z_j^{(t)} -  z_j^{(t-1)}\|
\end{align}
\end{proof}

\begin{lemma}\cite[Lemma 3]{hu2019learning}\label{lemma:trick}
\begin{align}
& \left(\left\|\sum_{m=1}^M x_m^{t+1}-z\right\|^2-\left\|\sum_{m=1}^M x_m^t-z\right\|^2\right)\\
& \leq \sum_{m=1}^M\left(\left\|\sum_{\substack{k=1 \\
k \neq m}}^M x_k^t+x_m^{t+1}-z\right\|^2-\left\|\sum_{m=1}^M x_m^t-z\right\|^2\right) 
+  \sum_{m=1}^M\left\|x_m^{t+1}-x_m^t\right\|^2
\end{align}
\end{lemma}

\subsubsection{Proofs for \cref{theorem:main-paper-conv}}

We restate our assumptions here in \cref{theorem:main-paper-conv}:

\begin{assumption}\label{assume:loss_lip}
 $\ell(z;\cdot)$ is $L$-Lipschitz smooth w.r.t $z$.
\end{assumption}
\begin{assumption} \label{assume:strong_convex}
$\mathcal{L}_{\mathrm{ADMM}} $ is strongly convex w.r.t $z$, $W$, $\theta$ with  constant $\mu_{z}$, $\mu_{W}$, $\mu_{\theta}$ respectively. 
\end{assumption}

\begin{assumption}\label{assume:w_theta_upper_bound}
The norm of $W_k$ is bounded $\|W_k\| \leq \sigma_{W}$.
The local model $f(\cdot;\theta)$ has bounded gradient $\|\nabla f(\cdot;\theta) \|\leq L_{\theta}$ and  bounded output norm $\|f(\cdot; \theta)\| \leq \sigma_{\theta}$.
\end{assumption}

\begin{assumption}\label{assume:vim_lower_bounded} The original objective function
$ \mathcal{L}_{\mathrm{\name}}$ is bounded from below over $\Theta$ and $\mathcal{W}$, that is 
$\underline{e} := \min_{\{\theta_k\} \in \Theta,  \{W_k\} \in \mathcal{W} }    \mathcal{L}_{\mathrm{\name}} ( \{W_k\}_{k=1}^M,\{\theta_k\}_{k=1}^M ) > - \infty.
$
\end{assumption}

\subsubsection{Proofs for \cref{theorem:main-paper-conv} Part (A)}

We decompose \cref{theorem:main-paper-conv} part (A) into the below two lemmas and prove them one-by-one. 
\begin{lemma}\label{lemma:admm_decrease}
Let  \cref{assume:loss_lip} to \cref{assume:w_theta_upper_bound}
hold,
and there exists a  penalty parameter $\rho$ satisfying 
\begin{equation}\label{eq:rho-condition}
   \max\{L, \frac{2L^2}{ \mu_z }\} < \rho < \min\{\frac{\mu_{\theta}}{   L_{\theta}^{2} \sigma_{W}^2},\frac{\mu_{W}}{ \sigma_{\theta}^2 }  \}
\end{equation} 
then
$\mathcal{L}_{\mathrm{ADMM}} $  is monotonically decreasing:
\begin{equation}
\begin{aligned}
\mathcal{L}_{\mathrm{ADMM}} ( \{W_k^{(t+1)}\}, \{\theta_k^{(t+1)}\},\{z_j^{(t+1)}\},\{\lambda_j^{(t+1)}\} )  - \mathcal{L}_{\mathrm{ADMM}} ( \{W_k^{(t)}\}, \{\theta_k^{(t)}\} ,\{z_j^{(t)}\},\{\lambda_j^{(t)}\}) < 0.
\end{aligned}
\end{equation}
\end{lemma}

\begin{lemma}\label{lemma:admm_lower_bound}
Let  \cref{assume:loss_lip} to  \cref{assume:vim_lower_bounded} hold, then
the  following limit exists and $\mathcal{L}_{\mathrm{ADMM}}$ is lower bounded by \underline{e} defined in \cref{assume:vim_lower_bounded}:
\begin{equation}
\begin{aligned}
\lim_{t\rightarrow \infty} \mathcal{L}_{\mathrm{ADMM}} ( \{W_k^{(t)}\}, \{\theta_k^{(t)}\} ,\{z_j^{(t)}\},\{\lambda_j^{(t)}\}) \geq \underline{e}.
\end{aligned}
\end{equation}
\end{lemma}

We first present the proof for the monotonically decreasing property  of $\mathcal{L}_{\mathrm{ADMM}}$ in \cref{lemma:admm_decrease}.

\begin{proof}[Proof for \cref{lemma:admm_decrease}]
\begin{equation}
\begin{aligned}
    & \mathcal{L}_{\mathrm{ADMM}} ( \{W_k^{(t+1)}\}, \{\theta_k^{(t+1)}\},\{z_j^{(t+1)}\},\{\lambda_j^{(t+1)}\} ) - \mathcal{L}_{\mathrm{ADMM}} ( \{W_k^{(t)}\}, \{\theta_k^{(t)}\} ,\{z_j^{(t)}\},\{\lambda_j^{(t)}\}) \\
     = &   \underbrace{\mathcal{L}_{\mathrm{ADMM}} ( \{W_k^{(t+1)}\}, \{\theta_k^{(t+1)}\},\{z_j^{(t+1)}\},\{\lambda_j^{(t+1)}\} ) - \mathcal{L}_{\mathrm{ADMM}} ( \{W_k^{(t+1)}\}, \{\theta_k^{(t+1)}\},\{z_j^{(t+1)}\},\{\lambda_j^{(t)}\} )}_{T_1} \\
    & +\underbrace{ \mathcal{L}_{\mathrm{ADMM}} ( \{W_k^{(t+1)}\}, \{\theta_k^{(t+1)}\},\{z_j^{(t+1)}\},\{\lambda_j^{(t)}\} ) 
    -  \mathcal{L}_{\mathrm{ADMM}}( \{W_k^{(t+1)}\}, \{\theta_k^{(t+1)}\},\{z_j^{(t)}\},\{\lambda_j^{(t)}\}  ) 
    }_{T_2}\\
    &+ \underbrace{ \mathcal{L}_{\mathrm{ADMM}}( \{W_k^{(t+1)}\}, \{\theta_k^{(t+1)}\},\{z_j^{(t)}\},\{\lambda_j^{(t)}\} 
    -  \mathcal{L}_{\mathrm{ADMM}}( \{W_k^{(t+1)}\}, \{\theta_k^{(t)}\},\{z_j^{(t)}\},\{\lambda_j^{(t)}\}  ) 
    }_{T_3}\\
    & +
    \underbrace{  \mathcal{L}_{\mathrm{ADMM}} ( \{W_k^{(t+1)}\}, \{\theta_k^{(t)}\},\{z_j^{(t)}\},\{\lambda_j^{(t)}\} 
    -  \mathcal{L}_{\mathrm{ADMM}} ( \{W_k^{(t)}\}, \{\theta_k^{(t)}\} ,\{z_j^{(t)}\},\{\lambda_j^{(t)}\}) 
    }_{T_4}
\end{aligned}
\end{equation}

Recall ADMM objective function
\begin{align*}
 & \mathcal{L}_{\mathrm{ADMM}}  = \frac{1}{N}\sum_{j=1}^N  \ell (z_j,y_j) + \sum_{k=1}^M \beta_k   \mathcal{R}_{k}(\theta_k) +  \sum_{k=1}^M \beta_k  \mathcal{R}_{k}(W_k)   \\
 & + \frac{1}{N}\sum_{j=1}^N  \lambda_{j}^\top \left (   \sum_{k=1}^M f(x_j^k;\theta_k)  W_{k} - z_j \right)+ \frac{\rho}{2N}  \sum_{j=1}^N \left\|  \sum_{k=1}^M f(x_j^k;\theta_k)  W_{k} - z_j\right\|_F^2     
\end{align*}

Then we have 
\begin{small}
\begin{align*}
   T_1 
  & = \mathcal{L}_{\mathrm{ADMM}} ( \{W_k^{(t+1)}\}, \{\theta_k^{(t+1)}\},\{z_j^{(t+1)}\},\{\lambda_j^{(t+1)}\} ) - \mathcal{L}_{\mathrm{ADMM}} ( \{W_k^{(t+1)}\}, \{\theta_k^{(t+1)}\},\{z_j^{(t+1)}\},\{\lambda_j^{(t)}\} )\\ 
    & =  \frac{1}{N} \sum_{j=1}^N  (\lambda_{j}^{(t+1)}-\lambda_{j}^{(t)})  ^\top \left (   \sum_{k=1}^M f(x_j^k;\theta_k^{(t+1)})  W_{k}^{(t+1)} - z_j^{(t+1)} \right) \\
    & \overset{(a)}{=}  \frac{1}{N} \sum_{j=1}^N  \frac{1}{\rho}  \|\lambda_{j}^{(t+1)}-\lambda_{j}^{(t)}\|^2\\
     & \overset{(b)}{\leq}   \sum_{j=1}^N  \frac{L^2}{\rho N}  \|z_{j}^{(t+1)}-z_{j}^{(t)}\|^2
\end{align*}
\end{small}
where 
(a) we use the \cref{eq:update_lambda} that 
$\frac{1}{\rho} (\lambda^{(t)}_{j} - \lambda^{(t-1)}_{j}) = \left(\sum_{k=1}^M {f(x_j^k;\theta_k^{(t)})} W_k^{(t)} -z^{(t)}_{j}\right) $; 
(b) we use \cref{lemma:z_lambda}.

\begin{small}
\begin{align*}
T_2 & =\mathcal{L}_{\mathrm{ADMM}} ( \{W_k^{(t+1)}\}, \{\theta_k^{(t+1)}\},\{z_j^{(t+1)}\},\{\lambda_j^{(t)}\} ) 
    -  \mathcal{L}_{\mathrm{ADMM}}( \{W_k^{(t+1)}\}, \{\theta_k^{(t+1)}\},\{z_j^{(t)}\},\{\lambda_j^{(t)}\}  ) \\
& =   \frac{1}{N}  \sum_{j=1}^N  \left(\ell (z_j^{(t+1)}) -  \ell (z_j^{(t)}) \right) 
 - \frac{1}{N} \sum_{j=1}^N  {\lambda_{j}^{(t)}}^\top \left (     z_j^{(t+1)} - z_j^{(t)}\right) \\
 +&  \frac{\rho}{2N}  \sum_{j=1}^N  \left( \left\|  \sum_{k=1}^M f(x_j^k;\theta_k^{(t+1)})  W_{k}^{(t+1)} - z_j^{(t+1)}\right\|_F^2 - \left\|  \sum_{k=1}^M f(x_j^k;\theta_k^{(t+1)})  W_{k}^{(t+1)} - z_j^{(t)}\right\|_F^2  \right) \\
&=  \frac{1}{N} \sum_{j=1}^N  \left( h(z_j^{(t+1)})-  h(z_j^{(t)}) \right) \\
& \overset{(a)}{\leq} \frac{1}{N}  \sum_{j=1}^N  \left( \left \langle  \nabla h (z_j^{(t+1)})  ,  z_j^{(t+1)} - z_j^{(t)} \right \rangle  - \frac{\mu_z}{2} \| z_j^{(t+1)} - z_j^{(t)}\|^2 \right) \\
& \overset{(b)}{=} - \frac{\mu_z}{2N} \sum_{j=1}^N   \| z_j^{(t+1)} - z_j^{(t)}\|^2 \\
\end{align*}
\end{small}
where (a) strong convex of $h$ \cref{assume:strong_convex} , (b) optimality of z update at \cref{eq:update_z}.

\begin{small}
\begin{align*}
 T_3 = &  \mathcal{L}_{\mathrm{ADMM}}( \{W_k^{(t+1)}\}, \{\theta_k^{(t+1)}\},\{z_j^{(t)}\},\{\lambda_j^{(t)}\}
    -  \mathcal{L}_{\mathrm{ADMM}}( \{W_k^{(t+1)}\}, \{\theta_k^{(t)}\},\{z_j^{(t)}\},\{\lambda_j^{(t)}\}  ) \\
  = & \sum_{k=1}^M \beta_k  \left( \mathcal{R}_{k}(\theta_k^{(t+1)})  - \mathcal{R}_{k}(\theta_k^{(t)})  \right) + 
  \frac{1}{N} \sum_{j=1}^N  {\lambda_{j}^{(t)}}^\top \left (   \sum_{k=1}^M \left( f(x_j^k;\theta_k^{(t+1)})W_k^{(t+1)}  - f(x_j^k;\theta_k^{(t)})W_k^{(t+1)}\right)     \right)  \\
  & + \frac{\rho}{2N}  \sum_{j=1}^N \left(  \left\|  \sum_{k=1}^M f(x_j^k;\theta_k^{(t+1)})  W_{k}^{(t+1)} - z_j^{(t)}\right\|_F^2 - 
    \left\|  \sum_{k=1}^M f(x_j^k;\theta_k^{(t)})  W_{k}^{(t+1)} - z_j^{(t)}\right\|_F^2 \right) \\
  \overset{(a)}{\leq} & \sum_{k=1}^M \beta_k  \left( \mathcal{R}_{k}(\theta_k^{(t+1)})  - \mathcal{R}_{k}(\theta_k^{(t)})  \right) + 
 \frac{1}{N}   \sum_{j=1}^N  {\lambda_{j}^{(t)}}^\top \left (   \sum_{k=1}^M \left( f(x_j^k;\theta_k^{(t+1)})W_k^{(t+1)}  - f(x_j^k;\theta_k^{(t)})W_k^{(t+1)}\right)     \right)  \\
  & + \frac{\rho}{2N}  \sum_{j=1}^N \sum_{k=1}^M \left( 
 \left\|   \sum_{ \substack{i\in [M] ,\\ i \neq k}}  {f(x_j^i;\theta_i^{(t)})}{W_{i}}^{(t+1)} +   f(x_j^k; \theta_k^{(t+1)} ) { W_{k}^{(t+1)} } -  {z_j}^{(t)}   \right\|_F^2 -   \left\|  \sum_{k=1}^M f(x_j^k;\theta_k^{(t)})  W_{k}^{(t+1)} - z_j^{(t)}\right\|_F^2  \right) \\
 & +\frac{\rho}{2N}  \sum_{j=1}^N \sum_{k=1}^M   \left\|   f(x_j^k;\theta_k^{(t+1)})  W_{k}^{(t+1)} -   f(x_j^k;\theta_k^{(t)})  W_{k}^{(t+1)} \right\|_F^2 \\
 =& \sum_{k=1}^M  \left( q_k(\theta_k^{(t+1)})- q_k(\theta_k^{(t)}) \right) + \frac{\rho}{2N}  \sum_{j=1}^N \sum_{k=1}^M   \left\|   f(x_j^k;\theta_k^{(t+1)})  W_{k}^{(t+1)} -   f(x_j^k;\theta_k^{(t)})  W_{k}^{(t+1)} \right\|_F^2 \\
  \overset{(b)}{\leq} &\sum_{k=1}^M  \left(  \left \langle  \nabla q_k (\theta_k^{(t+1)})  ,  \theta_k^{(t+1)} - \theta_k^{(t)} \right \rangle  - \frac{\mu_{\theta}}{2} \| \theta_k^{(t+1)} - \theta_k^{(t)}\|^2 \right) \\
  & \quad + \frac{\rho}{2N}  \sum_{j=1}^N \sum_{k=1}^M   \left\|   f(x_j^k;\theta_k^{(t+1)})  W_{k}^{(t+1)} -   f(x_j^k;\theta_k^{(t)})  W_{k}^{(t+1)} \right\|_F^2 \\
 \overset{(c)}{=} & \sum_{k=1}^M    - \frac{\mu_{\theta}}{2}   \| \theta_k^{(t+1)} - \theta_k^{(t)}\|^2 +  \frac{\rho}{2N}  \sum_{j=1}^N \sum_{k=1}^M   \left\|   f(x_j^k;\theta_k^{(t+1)})  -   f(x_j^k;\theta_k^{(t)}) \right\|^2  \| W_{k}^{(t+1)}\|^2  \\
\overset{(d)}{\leq} &  \sum_{k=1}^M    -\frac{ \mu_{\theta}}{2}\left\|   \theta_k^{(t+1)}  -   \theta_k^{(t)}\right\|^2 +  \frac{\rho   L_{\theta}^{2}}{2} \sum_{k=1}^M  \left\|   \theta_k^{(t+1)}  -   \theta_k^{(t)}\right\|^2   \| W_{k}^{(t+1)}\|^2   \\
\overset{(e)}{\leq} &   \sum_{k=1}^M \frac{-\mu_{\theta} +\rho   L_{\theta}^{2} \sigma_{W}^2  }{2}   \left\|   \theta_k^{(t+1)}  -   \theta_k^{(t)}\right\|^2   \\
\end{align*}
\end{small}
where (a) is due to \cref{lemma:trick}, (b) is due to the strong convex of $q_k$ \cref{assume:strong_convex},  (c) is due to the  optimality of $\theta_k$ in \cref{eq:update_theta}, (d) is due to the Lipschitz continuity of local model $f(\theta)$ \cref{assume:w_theta_upper_bound} ( bounded gradient implies Lipschitz continuity), and  (e) is due to the upper bound of the linear weights  in \cref{assume:w_theta_upper_bound}.

\begin{small}
\begin{align*}
   T_4 =& \mathcal{L}_{\mathrm{ADMM}} ( \{W_k^{(t+1)}\}, \{\theta_k^{(t)}\},\{z_j^{(t)}\},\{\lambda_j^{(t)}\}
    -  \mathcal{L}_{\mathrm{ADMM}} ( \{W_k^{(t)}\}, \{\theta_k^{(t)}\} ,\{z_j^{(t)}\},\{\lambda_j^{(t)}\}) \\
 = &\sum_{k=1}^M \beta_k  \left( \mathcal{R}_{k}(W_k^{(t+1)})  - \mathcal{R}_{k}(W_k^{(t)})  \right) + \frac{1}{N}
  \sum_{j=1}^N  {\lambda_{j}^{(t)}}^\top \left (   \sum_{k=1}^M f(x_j^k;\theta_k^{(t)})  \left ( W_k^{(t+1)} -  W_k^{(t)} \right)  \right)  \\
  & + \frac{\rho}{2N}  \sum_{j=1}^N \left(  \left\|  \sum_{k=1}^M f(x_j^k;\theta_k^{(t)})  W_{k}^{(t+1)} - z_j^{(t)}\right\|_F^2 - 
    \left\|  \sum_{k=1}^M f(x_j^k;\theta_k^{(t)})  W_{k}^{(t)} - z_j^{(t)}\right\|_F^2 \right)  \\
  \overset{(a)}{\leq} & \sum_{k=1}^M \beta_k  \left( \mathcal{R}_{k}(W_k^{(t+1)})  - \mathcal{R}_{k}(W_k^{(t)})  \right) + 
  \frac{1}{N} \sum_{j=1}^N  {\lambda_{j}^{(t)}}^\top \left (   \sum_{k=1}^M f(x_j^k;\theta_k^{(t)})  \left ( W_k^{(t+1)} -  W_k^{(t)} \right)  \right)  \\
  & + \frac{\rho}{2N}  \sum_{j=1}^N \sum_{k=1}^M \left(  \left\|  \sum_{ \substack{i\in [M] ,\\ i \neq k } }  f(x_j^i;\theta_i^{(t)})  {W_{i}}^{(t)} +   {f(x_j^k;\theta_k^{(t)})}  W_{k}^{(t+1)} - {z}_j^{(t)} \right\|_F^2 - 
    \left\|  \sum_{k=1}^M f(x_j^k;\theta_k^{(t)})  W_{k}^{(t)} - z_j^{(t)}\right\|_F^2 \right)  \\ 
&+    \frac{\rho}{2N}  \sum_{j=1}^N \sum_{k=1}^M \left \|    {f(x_j^k;\theta_k^{(t)})}  W_{k}^{(t+1)} -  {f(x_j^k;\theta_k^{(t)})}  W_{k}^{(t)}   \right\|_F^2 \\
= &   \sum_{k=1}^M  \left( g_k( W_{k}^{(t+1)}) - g_k( W_{k}^{(t)})\right) +    \frac{\rho}{2N}  \sum_{j=1}^N \sum_{k=1}^M \left \|    {f(x_j^k;\theta_k^{(t)})}  W_{k}^{(t+1)} -  {f(x_j^k;\theta_k^{(t)})}  W_{k}^{(t)}   \right\|_F^2 \\
\overset{(b)}{\leq} & \sum_{k=1}^M  \left(  \left \langle  \nabla g_k (W_k^{(t+1)})  ,  W_k^{(t+1)} - W_k^{(t)} \right \rangle  - \frac{\mu_{W}}{2} \| W_k^{(t+1)} - W_k^{(t)}\|^2 \right)   \\
& + \frac{\rho}{2N}  \sum_{j=1}^N \sum_{k=1}^M \left \|    {f(x_j^k;\theta_k^{(t)})}  W_{k}^{(t+1)} -  {f(x_j^k;\theta_k^{(t)})}  W_{k}^{(t)}   \right\|_F^2 \\
\overset{(c)}{\leq} & \sum_{k=1}^M  - \frac{\mu_{W}}{2}   \left\| W_k^{(t+1)} - W_k^{(t)}\right\|^2 +    \frac{\rho N}{2}  \sum_{k=1}^M \left \|   W_{k}^{(t+1)} - W_{k}^{(t)}   \right\|^2  \| {f(x_j^k;\theta_k^{(t)})}   \|^2\\
\overset{(d)}{\leq} &  \sum_{k=1}^M    \frac{-\mu_{W} + \rho N \sigma_{\theta}^2  }{2}  \left \|   W_{k}^{(t+1)} - W_{k}^{(t)}   \right\|^2  \\
\end{align*}
\end{small}
where (a) is due to \cref{lemma:trick}, (b) is due to strong convex of $g_k$ \cref{assume:strong_convex}, (c) is because of the optimality of $W_k$ in \cref{eq:update_W} and (d) is due to upper bound of the model outputs in \cref{assume:w_theta_upper_bound}.

Combining the above bounds for $T_1, T_2, T_3, T_4$ together and recall the condition for $\rho$ in \cref{eq:rho-condition}, we have
\begin{small}
\begin{align*}
    & \mathcal{L}_{\mathrm{ADMM}} ( \{W_k^{(t+1)}\}, \{\theta_k^{(t+1)}\},\{z_j^{(t+1)}\},\{\lambda_j^{(t+1)}\} ) - \mathcal{L}_{\mathrm{ADMM}} ( \{W_k^{(t)}\}, \{\theta_k^{(t)}\} ,\{z_j^{(t)}\},\{\lambda_j^{(t)}\}) \\
     = &T_1+ T_2  + T_3 +T_4 \\
    \leq  & \frac{1}{N}\sum_{j=1}^N \left(- \frac{\mu_z}{2}  + \frac{L^2}{\rho} \right) \|z_{j}^{(t+1)}-z_{j}^{(t)}\|^2 + \sum_{k=1}^M \frac{-\mu_{\theta} +\rho  L_{\theta}^{2} \sigma_{W}^2  }{2}   \left\|   \theta_k^{(t+1)}  -   \theta_k^{(t)}\right\|^2    + \sum_{k=1}^M    \frac{-\mu_{W} + \rho  \sigma_{\theta}^2}{2}  \left \|   W_{k}^{(t+1)} - W_{k}^{(t)}   \right\|^2  \\
    < & 0 
\end{align*}
\end{small}
Thus, proved.
\end{proof}

Then we provide the proof for the lower-bounded property of $\mathcal{L}_{\mathrm{ADMM}}$ for  \cref{lemma:admm_lower_bound}. 

\begin{proof}[Proof for \cref{lemma:admm_lower_bound}]

\begin{align*}
 &\mathcal{L}_{\mathrm{ADMM}} ( \{W_k^{(t)}\}, \{\theta_k^{(t)}\} ,\{z_j^{(t)}\},\{\lambda_j^{(t)}\})\\  
 = & \frac{1}{N} \sum_{j=1}^N  \ell (z_j^{(t)},y_j) + \sum_{k=1}^M \beta_k   \mathcal{R}_{k}(\theta_k^{(t)}) +  \sum_{k=1}^M \beta_k  \mathcal{R}_{k}(W_k^{(t)})   \\
 &+ \frac{1}{N} \sum_{j=1}^N  {\lambda_{j}^{(t)}}^\top \left (   \sum_{k=1}^M f(x_j^k;\theta_k^{(t)})  W_{k}^{(t)} - z_j^{(t)} \right)+ \frac{\rho}{2N}  \sum_{j=1}^N \left\|  \sum_{k=1}^M f(x_j^k;\theta_k^{(t)})  W_{k}^{(t)} - z_j^{(t)}\right\|_F^2 \\
 \overset{(a)}{=}& \frac{1}{N} \sum_{j=1}^N  \ell (z_j^{(t)},y_j) + \sum_{k=1}^M \beta_k   \mathcal{R}_{k}(\theta_k^{(t)}) +  \sum_{k=1}^M \beta_k  \mathcal{R}_{k}(W_k^{(t)})   \\
 &+ \frac{1}{N} \sum_{j=1}^N  {\nabla \ell (z_j^{(t)})}^\top \left (   \sum_{k=1}^M f(x_j^k;\theta_k^{(t)})  W_{k}^{(t)} - z_j^{(t)} \right)+ \frac{\rho}{2N}  \sum_{j=1}^N \left\|  \sum_{k=1}^M f(x_j^k;\theta_k^{(t)})  W_{k}^{(t)} - z_j^{(t)}\right\|_F^2 \\
 \overset{(b)}{\geq}& \frac{1}{N}\sum_{j=1}^N  \ell \left(\sum_{k=1}^M f(x_j^k;\theta_k^{(t)})  W_{k}^{(t)} , y_j \right)  + \sum_{k=1}^M \beta_k   \mathcal{R}_{k}(\theta_k^{(t)}) +  \sum_{k=1}^M \beta_k  \mathcal{R}_{k}(W_k^{(t)})     \\
  & + \frac{\rho - L}{2N}  \sum_{j=1}^N \left\|  \sum_{k=1}^M f(x_j^k;\theta_k^{(t)})  W_{k}^{(t)} - z_j^{(t)}\right\|_F^2\\ 
 \overset{(c)}{\geq}& \frac{1}{N}\sum_{j=1}^N  \ell \left(\sum_{k=1}^M f(x_j^k;\theta_k^{(t)})  W_{k}^{(t)} , y_j \right)  + \sum_{k=1}^M \beta_k   \mathcal{R}_{k}(\theta_k^{(t)}) +  \sum_{k=1}^M \beta_k  \mathcal{R}_{k}(W_k^{(t)})     \\
  =&   \mathcal{L}_{\mathrm{\name}} ( \{W_k^{(t)}\}, \{\theta_k^{(t)}\})\\
  \geq & \underline{e}
\end{align*}
where (a) is due to \cref{lemma:zgrad_lambda}; (b) is due to Lipschitz continuity of  gradient of $\ell$ in  \cref{assume:loss_lip}  that 
\begin{small}
\begin{align*}
\ell \left(\sum_{k=1}^M f(x_j^k;\theta_k^{(t)})  W_{k}^{(t)} \right) - \ell \left( z_j^{(t)} \right) - \frac{L}{2}\left\| \sum_{k=1}^M f(x_j^k;\theta_k^{(t)})  W_{k}^{(t)}  - z_j^{(t)}  \right\| \leq 
\left \langle \nabla \ell (z_j^{(t)}) , \left (   \sum_{k=1}^M f(x_j^k;\theta_k^{(t)})  W_{k}^{(t)} - z_j^{(t)} \right) \right \rangle
\end{align*} 
\end{small}
and (c) is due to $\rho \geq L$  from the condition \cref{eq:rho-condition}.

The result show that $\mathcal{L}_{\mathrm{ADMM}} ( \{W_k^{(t)}\}, \{\theta_k^{(t)}\} ,\{z_j^{(t)}\},\{\lambda_j^{(t)}\})$ is lower bounded. Thus, proved. 
\end{proof}

\begin{proof}[Proof for \cref{theorem:main-paper-conv} (A)]
Combining \cref{lemma:admm_decrease} and \cref{lemma:admm_lower_bound}, we show that $\mathcal{L}_{\mathrm{ADMM}} ( \{W_k^{(t)}\}, \{\theta_k^{(t)}\} ,\{z_j^{(t)}\},\{\lambda_j^{(t)}\})$ is monotonically decreasing and is convergent. This completes the proof.
\end{proof}

\subsubsection{Proofs for \cref{theorem:main-paper-conv} Part (B)}

\begin{proof}[Proofs for \cref{theorem:main-paper-conv} (B)]
	
\cref{lemma:admm_decrease} implies that 
\begin{align*}
	 & \mathcal{L}_{\mathrm{ADMM}} ( \{W_k^{(t+1)}\}, \{\theta_k^{(t+1)}\},\{z_j^{(t+1)}\},\{\lambda_j^{(t+1)}\} ) - \mathcal{L}_{\mathrm{ADMM}} ( \{W_k^{(t)}\}, \{\theta_k^{(t)}\} ,\{z_j^{(t)}\},\{\lambda_j^{(t)}\}) \\
    \leq  & \frac{1}{N}\sum_{j=1}^N \left(- \frac{\mu_z}{2}  + \frac{L^2}{\rho} \right) \|z_{j}^{(t+1)}-z_{j}^{(t)}\|^2 + \sum_{k=1}^M \frac{-\mu_{\theta} +\rho  L_{\theta}^{2} \sigma_{W}^2  }{2}   \left\|   \theta_k^{(t+1)}  -   \theta_k^{(t)}\right\|^2   \\
    & \quad+   \sum_{k=1}^M    \frac{-\mu_{W} + \rho  \sigma_{\theta}^2}{2}  \left \|   W_{k}^{(t+1)} - W_{k}^{(t)}   \right\|^2 
\end{align*}

Using the fact that $\mathcal{L}_{\mathrm{ADMM}}$ is  monotonically decreasing and  lower-bounded (in \cref{lemma:admm_lower_bound}) as well as the bounds for $\rho$ in \cref{eq:rho-condition},  we have  $\forall j\in [N], k \in [M]$, 
\begin{align}\label{eq:limit_z_theta_w_update}
	\lim_{t\rightarrow \infty} \left\|  z_{j}^{(t+1)}-z_{j}^{(t)} \right\|^2  \rightarrow 0, 
	\lim_{t\rightarrow \infty} \left\|   \theta_k^{(t+1)}  -   \theta_k^{(t)}\right\|^2  \rightarrow 0,
	\lim_{t\rightarrow \infty} \left\|   W_{k}^{(t+1)} - W_{k}^{(t)}   \right\|^2  \rightarrow 0.
\end{align}

By \cref{lemma:z_lambda} that $\| \lambda_j^{(t+1)} -  \lambda_j^{(t)} \| \leq L \|  z_j^{(t+1)} -  z_j^{(t)}\|$, we further obtain 
\begin{align}\label{eq:limit_lambda_update}
	\lim_{t\rightarrow \infty} \left\| \lambda_j^{(t+1)} -  \lambda_j^{(t)}  \right\|^2  \rightarrow 0, \forall j\in [N]
\end{align}
In light of the dual update step of \cref{algo:vimadmm}, \cref{eq:limit_lambda_update} implies that 
\begin{align}\label{eq:limit_constaint_update}
	\lim_{t\rightarrow \infty} \left\| \sum_{k=1}^M f(x_j^k;\theta_k^{(t+1)})  W_{k}^{(t+1)} - z_j^{(t+1)}  \right\|^2  \rightarrow 0, \forall j\in [N]
\end{align}

Using the limit points, we have 
$W_{k}^{(t+1)} \rightarrow W_{k}^{*},
\theta_k^{(t+1)}  \rightarrow \theta_k^{*} , 
z_j^{(t+1)}  \rightarrow z_j^{*} , 
\lambda_j^{(t+1)}  \rightarrow \lambda_j^{*} 
$.

Based on \cref{eq:limit_constaint_update}, 
we have 
\begin{align}
	 \sum_{k=1}^M f(x_j^k;\theta_k^{*})  W_{k}^{*} = z_j^{*} 
\end{align}

Then, we examine the optimality condition for the 
$\{W_k^{(t+1)}\}$  subproblems at iteration $t + 1$:
\begin{align} 
 0= & \beta_k   \nabla \mathcal{R}_k(W_k^{(t+1)} ) +   \frac{1}{N} \sum_{j=1}^N {\lambda_{j}^{(t)}}^\top  f(x_j^k;\theta_k^{(t)})  \nonumber \\
  & + \sum\limits_{j=1}^N \frac{\rho}{N} \left(  \sum\limits_{i\in [M] , i \neq k} f(x_j^i;\theta_i^{(t)})  {W_{i}}^{(t)} +    f(x_j^k;\theta_k^{(t)})  W_{k}^{(t+1)} - {z_j}^{(t)}\right) f(x_j^k;\theta_k^{(t)}) 
\end{align}

According to \cref{eq:limit_z_theta_w_update} and \cref{eq:limit_constaint_update}, we have 

\begin{align}
 0= & \beta_k   \nabla \mathcal{R}_k(W_k^{*} ) +   \frac{1}{N} \sum_{j=1}^N {\lambda_{j}^{*}}^\top  f(x_j^k;\theta_k^{*}) 
\end{align}

Similarly, the optimality condition for the 
$\{\theta_k^{(t+1)}\}$  subproblems at iteration $t + 1$ indicates that:
\begin{align} 
0 = & \beta_k   \nabla \mathcal{R}_k(\theta_k^{(t+1)} ) +   \frac{1}{N} \sum_{j=1}^N {\lambda_{j}^{(t)}}^\top  \nabla f(x_j^k;\theta_k^{(t+1)})  W_k^{(t+1)}  \nonumber \\
 &   + \sum\limits_{j=1}^N \frac{\rho}{N} \left(  \sum\limits_{i\in [M] , i \neq k} f(x_j^i;\theta_i^{(t)})  {W_{i}}^{(t+1)} +    f(x_j^k;\theta_k^{(t)})  W_{k}^{(t+1)} - {z_j}^{(t)}\right) \nabla f(x_j^k;\theta_k^{(t+1)})  W_k^{(t+1)} 
\end{align}
According to \cref{eq:limit_z_theta_w_update} and \cref{eq:limit_constaint_update}, we have 
\begin{align} 
0 = & \beta_k   \nabla \mathcal{R}_k(\theta_k^{*} ) +   \frac{1}{N} \sum_{j=1}^N {\lambda_{j}^{*}}^\top  \nabla f(x_j^k;\theta_k^{*})  W_k^{*}  
\end{align}

Based on the optimality condition for the 
$\{z_j^{(t+1)}\}$  subproblems at iteration $t + 1$, we have 
\begin{align}\label{eq:op_cond_z}
   0 =  \nabla \ell( z_j^{(t)}) - 
    {  \lambda_{j}^{(t-1)}}  - \rho \left(   \sum_{k=1}^M  {f(x_j^k;\theta_k^{(t)} )} W_k^{(t)} -  z_j^{(t)} \right), \forall j \in B(t)
\end{align}

Based on the strongly convexity w.r.t $z_j$  in \cref{assume:strong_convex},
there exists a subgradient $\eta^{(t)} \in \partial \ell\left(z_j^{(t)}\right)$ such that 
\begin{align}
\left \langle z - z_j^{(t)}, 	\eta^{(t)} - \left(  \lambda_j^{(t-1)} + \rho \left(   \sum_{k=1}^M  {f(x_j^k;\theta_k^{(t)} )} W_k^{(t)} -  z_j^{(t)} \right)\right)  \right \rangle \geq 0, \quad \forall z 
\end{align}
It implies that 
\begin{align}\label{eq:subgrad_conv}
 \ell \left( z; y_j \right) - \ell \left( z_j^{(t)}; y_j \right)   + \left \langle  z- z_j^{(t)},  -  \left(  \lambda_j^{(t-1)} + \rho \left(   \sum_{k=1}^M  {f(x_j^k;\theta_k^{(t)} )} W_k^{(t)} -  z_j^{(t)} \right)\right)  \right \rangle  \geq 0	, \forall z 
\end{align}
Taking the limits for \cref{eq:subgrad_conv} and using the results in \cref{eq:limit_z_theta_w_update}~\cref{eq:limit_lambda_update}~\cref{eq:limit_constaint_update},  we have

\begin{align} 
 \ell \left( z; y_j \right) - \ell \left( z_j^{*}; y_j \right)   + \left \langle  z- z_j^{*},  -  \lambda_j^{*}   \right \rangle  \geq 0	, \forall z 
\end{align}

That is: 
\begin{align*}
 \ell \left( z; y_j \right) + {\lambda_j^*}^\top \left(  \sum_{k=1}^M f(x_j^k;\theta_k^{*})  W_{k}^{*} - z \right) 	 \geq  \ell \left( z_j^*; y_j \right) + {\lambda_j^*}^\top \left(  \sum_{k=1}^M f(x_j^k;\theta_k^{*})  W_{k}^{*} - z_j^* \right) 	
\end{align*}

It implies that
\begin{align*}
z_{j}^{*} \in \arg \min_{z_j}   \ell \left( z_j; y_j \right) + {\lambda_j^*}^\top \left(  \sum_{k=1}^M f(x_j^k;\theta_k^{*})  W_{k}^{*} - z_j \right) 
\end{align*}

This completes the proof.

\end{proof}

\subsection{Privacy Guarantees} \label{app:dp}

\subsubsection{Preliminaries}
We utilize R\'enyi Differential Privacy (RDP) to perform the privacy analysis since it supports a tighter composition of privacy budget~\cite{mironov2017renyi} than the moments accounting technique~\cite{abadi2016deep} for Differential Privacy (DP). 

We start by introducing the definition of RDP as a generalization of DP, which leverages the $\alpha$-R\'enyi divergence between the output distributions of two neighboring datasets.  The definition of the neighboring dataset in this work follows the \textit{\revise{client-level}} differentially private FL framework~\cite{mcmahan2018learning}. \textit{The neighboring datasets would differ in all samples associated with a single \revise{client}, that is, one user is present or absent in the VFL global dataset.} (\cref{def:user-dp}) 

\begin{definition}(R\'enyi Differential Privacy~\cite{mironov2017renyi})\label{def:rdp}
We say that a mechanism $\mathcal{M}$ is $(\alpha, \epsilon)$-RDP with order $\alpha \in  (1, \infty)$ if for all neighboring datasets $D, D'$ 
\begin{equation}
    D_{\alpha}\left(\mathcal{M}(D) \| \mathcal{M}\left(D^{\prime}\right)\right):=\frac{1}{\alpha-1} \log \mathbb{E}_{\theta \sim \mathcal{M}\left(D^{\prime}\right)}\left[\left(\frac{\mathcal{M}(D)(\theta)}{\mathcal{M}\left(D^{\prime}\right)(\theta)}\right)^{\alpha}\right] \leq \epsilon
\end{equation}
\end{definition}

RDP guarantee can be converted to DP guarantee as follows:
\begin{theorem}(RDP to $(\epsilon, \delta)$-DP Conversion ~\cite{balle2020hypothesis})\label{thm:rdp2dp}
\footnote{This theorem is tighter than the original RDP paper~\cite{mironov2017renyi}, and it is adopted in the official implementation of the PyTorch Opacus library.}
If $f$ is an $(\alpha, \epsilon)$-RDP mechanism, it also satisfies $(\epsilon + \log \frac{\alpha -1 }{\alpha} - \frac{\log\delta + \log \alpha}{ \alpha -1}    , \delta)$-differential privacy for any $0 < \delta < 1$. 
\end{theorem}

Here, we highlight three key properties that are relevant to our analyses.
\begin{theorem}(RDP Composition~\cite{mironov2017renyi})\label{thm:rdpcomposition}
Let $f: \mathcal{D} \mapsto \mathcal{R}_{1} $ be $(\alpha, \epsilon_1)$-RDP and $g : \mathcal{R}_{1}  \times  \mathcal{D}  \mapsto \mathcal{R}_{2} $ be $(\alpha,  \epsilon_2)$-RDP, then the mechanism defined as $(X, Y)$, where $X \sim f(D)$ and $Y \sim g(X, D)$, satisfies $(\alpha, \epsilon_1+ \epsilon_2)$-RDP.
\end{theorem}

\begin{theorem}(RDP Guarantee for Gaussian Mechanism~\cite{mironov2017renyi})\label{thm:rdpguassian}
If $f$ is a real-valued function, the Gaussian Mechanism for approximating $f$ is defined as $ \mathbf{G}_{\sigma} f(D) = f(D) +  \mathcal{N}\left(0, \sigma^{2}\right)  $. If $f$ has $\ell_2$ sensitivity 1, then the Gaussian Mechanism  $ \mathbf{G}_{\sigma} f $ satisfies $(\alpha, \alpha/(2\sigma^2))$-RDP. 
\end{theorem}

\subsubsection{Proof of \cref{theo:dp_guarantee}}
We aim to protect the local training data of each client under \revise{client-level} $(\epsilon,\delta)$-DP guarantee (\cref{def:user-dp}) during VFL training. 
Let $X$ be the VFL global dataset, i.e., the union of local feature sets $X_1,\ldots, X_M$ from all $M$ clients.
We denote the output of client $k$ as a matrix $\mathcal{A}_k$, where each row is the embedding or logit of one local training sample.   With a loss of generality, we consider the embedding matrix $\mathcal{A}_k=[h_1^k, \ldots,  h_N^k]^\top$ as local output, and our analysis directly applies to the logit matrix $\mathcal{A}_k=[o_1^k, \ldots,  o_N^k]^\top$. The local outputs from all clients can be concatenated as a global embedding matrix $\mathcal{A}$:
\begin{equation}
    \mathcal{A}= [\mathcal{A}_1, \mathcal{A}_2,  \dots, \mathcal{A}_M]
\end{equation}

For our algorithms (\cref{algo:vimadmm}, \cref{algo:vimadmmjoint}) that sample a mini-batch of data with data indices $B(t)$ at each round $t$ for clients to compute their embeddings,  we view the corresponding $\mathcal{A}_k$ for each client $k$ as:
\begin{align}
    &\mathcal{A}_k^{(t)} [j]=  {h_j^k}^{(t)} \quad   \text{if} \quad   j\in B(t), \\
    &\mathcal{A}_k^{(t)} [j]=   0  \quad \quad\quad  \text{if} \quad   j\notin B(t).
\end{align}
Here we can fill in the rows of the output matrix for non-sampled indices (i.e., $j \notin B(t)$) as all zeros for privacy analysis. 

We will first analyze the privacy cost for one communication round (omitting the superscript $t$ here) and then accumulate the privacy costs over $T$ rounds via the DP composition theorem. 

We define a function $\mathcal{H}$ that outputs a global embedding matrix consisting of  clipped local embedding matrices for FL global dataset $X$ as:   
\begin{equation}
     \mathcal{H} (X)  = [\hat {\mathcal{A}_1},\ldots, \hat {\mathcal{A}_M}]  , \text{where\quad} \hat {\mathcal{A}_k}  =  \mathtt{Clip} ( \mathcal{A}_k  , C) , \forall k \in [M]. 
\end{equation} 

\begin{lemma}\label{lemma:l2sensi}
For any neighboring datasets $X, X$ differing by all samples associated by one single \revise{client}, the $\ell_2$ sensitivity for $\mathcal{H}$ is $C$.
\end{lemma}
\begin{proof}
WLOG, the neighboring dataset  $X'$ differs the first \revise{client} from $X$, i.e.,  $X' = \{{X_1}', X_2, \ldots, X_M\}$. Therefore the global embedding matrix $\mathcal{H}(X)$ and $\mathcal{H} (X')$ only differ by the clipped local embedding matrix from the first \revise{client} ($\hat {\mathcal{A}_1}$ and $\hat {\mathcal{A}_1}'$). Then, the $\ell_2$ sensitivity for $\mathcal{H}$ is bounded as follows:
\begin{equation}
  \max_{X, X'}  \| \mathcal{H}(X) - \mathcal{H} (X')\|_2 = \sqrt{ \revise{\| \hat {\mathcal{A}_1}  - \hat {\mathcal{A}_1}' \|_F^2} } \leq C .
\end{equation}
\end{proof}
Then, we define our Gaussian mechanism $\mathbf{G}_{\sigma C} \mathcal{H}$, which outputs a global matrix consisting of noise-perturbed local embedding matrices for VFL global dataset $X$:
\begin{equation}
     \mathbf{G}_{\sigma C} \mathcal{H} (X) = [ \widetilde {\mathcal{A}_1}, \ldots, \widetilde {\mathcal{A}_M} ], \text{where\quad} \widetilde {\mathcal{A}_k} =  \mathtt{Clip} (   {\mathcal{A}_k} , C) + \mathcal{N}\left(0, \sigma^{2} C^{2}\right) , \forall k \in [M]. 
\end{equation}

\begin{lemma}\label{lemma:ourgaussian}
Given the function $\mathcal{H} $ with $\ell_2$ sensitivity $C$,  Gaussian standard deviation $\sigma^2 C^2$ , 
the Gaussian mechanism with  $ \mathbf{G}_{\sigma C} \mathcal{H}  $ satisfies \revise{client-level} 
 $(\alpha, \alpha/(2\sigma^2) )$-RDP.
\end{lemma}
\begin{proof}
The $\ell_2$ sensitivity for the function $\mathcal{H} $ is $C$ by  \cref{lemma:l2sensi}. The Gaussian standard deviation for the noise-perturbed embedding is $\sigma C$, which is proportional to the clipping constant $C$. Combining it with \cref{thm:rdpguassian} yields the conclusion that $ \mathbf{G}_{\sigma C} \mathcal{H}  $ guarantees \revise{client-level} $(\alpha, \alpha/(2\sigma^2))$-RDP.
\end{proof}

We note that the training process in the server does not access the raw data $X_k$, thus it does not increase the privacy budget and the whole algorithm in one round satisfies RDP
by the post-processing property of RDP. 
For algorithms with $T$ communication rounds, we use the RDP Composition theorem (\cref{thm:rdpcomposition}) to accumulate the privacy budge over $T$ rounds, and convert the RDP guarantee into DP guarantee (\cref{thm:rdp2dp}).

Finally, we recall \cref{theo:dp_guarantee} and provide the formal proof. 
\thmdp* 
\begin{proof}
At each communication round, according to \cref{lemma:ourgaussian}, $ \mathbf{G}_{\sigma C}\mathcal{H} $ satisfies \revise{client-level} $(\alpha, \frac{\alpha}{2\sigma^2}  )$-RDP.
Due to the  post-processing property of RDP, after server training, our DP algorithms (i.e., DP versions of Algorithm~\ref{algo:vimadmm}, ~\ref{algo:vimadmmjoint}) with one round still satisfy \revise{client-level}  $(\alpha, \epsilon' (\alpha))$-RDP. 
Based on RDP Composition theorem (\cref{thm:rdpcomposition}), our DP algorithms  with $T$ communication rounds satisfy \revise{client-level}  $(\alpha, \frac{T\alpha}{2\sigma^2} )$-RDP. 
Based on the connection between RDP and DP in \cref{thm:rdp2dp},  our DP algorithms  with $T$ communication rounds also satisfy \revise{client-level}  $
( \frac{T \alpha}{2\sigma^2} + \log \frac{\alpha -1 }{\alpha} - \frac{\log\delta + \log \alpha}{ \alpha -1}  , \delta)$-DP.

\end{proof}

\subsection{Experimental Details and Additional Results}\label{app:exp_details}

\subsubsection{Platform}
We simulate the vertical federated learning setup (1 server and N clients) on a Linux machine with AMD Ryzen Threadripper 3990X 64-Core CPUs and 4 NVIDIA GeForce RTX 3090 GPUs. 
The algorithms are implemented by PyTorch~\cite{NEURIPS2019_pytorch}. Please see the submitted code for full details. We run each experiment 3 times with different random seeds.

\subsubsection{Hyperparameters}
We detail our hyperparameter tuning protocol and the hyperparameter values here.
For all VFL training experiments, 
we use the SGD optimizer with learning rate $\eta$  for the server's model, and the SGD optimizer with momentum 0.9 and learning rate $\eta$  for client $k$'s local model. 
The regularization weight $\beta$ is set to 0.005.
The embedding dimension $d_f$ is set to 60, and batch size $b$ is set to 1024 for all datasets. 

\paragraph{Vanilla VFL training}
For Vanilla VFL training experiments, we tune learning rates by performing a grid search separately for all methods over 
$\{0.05, 0.1,0.3,0.5, 0.8 \}$  on \mnist,
$\{0.003,0.005, 0.008,0.01,0.05, 0.1\}$ on \cifar,    
$\{0.05 , 0.1,0.5\}$ on \nus, 
$\{0.0005, 0.005, 0.01, 0.05,0.1\}$ on \modelnet. 
For ADMM-based methods, we tune penalty factor $\rho$ with a search grid $\{0.5,1,2\}$  on all datasets.

\paragraph{Differentially private VFL training}
We leverage the PyTorch Differential Privacy library Opacus~\footnote{\href{https://github.com/pytorch/opacus}{https://github.com/pytorch/opacus}}  to calculate the privacy budgets $\epsilon$. In all experiments, $\delta=1e-5$.
For each privacy budget $\epsilon$,  we perform a grid search for the combination of hyperparameters (including noise scale $\sigma$,  clipping threshold $C$, and learning rate $\eta$) for all methods for a fair comparison. 
The noise scale is tuned from $\sigma$ $\{2,3, 5,8, 10,30,50,70\}$ on all datasets. 
 $C$ is tuned from  $\{0.0005, 0.001, 0.005,0.01, 0.1, 1\}$ 
 and $\eta$ is tuned from  $\{0.05, 0.3, 0.5, 1\}$ for \mnist; 
 $C$ is tuned from  $\{0.01, 0.05, 0.1, 0.5, 1 \}$
 and $\eta$ is tuned from  $\{ 0.005, 0.01, 0.05,  0.1, 0.5, 1\}$ for \cifar; 
 $C$ is tuned from  $\{0.001, 0.005, 0.01,  0.05, 0.1 \}$ 
 and $\eta$ is tuned from  $\{0.05, 0.1, 0.3, 0.5, 1\}$ for \nus; 
 $C$ is tuned from  $\{0.01, 0.05, 0.1, 0.5, 1\}$
 and $\eta$ is tuned from  $\{0.05, 0.1, 0.5\}$ for \modelnet.
We use the best hyper-parameters to start {3 runs} with different random seeds and report the average results for each method.

\paragraph{Client-level explainability}
In the experiments of \textit{client importance validation via noisy test client}, for each time, we perturb the features of all test samples at one client by adding Gaussian noise sampled from   $\mathcal{N}\left(0, {\bar \sigma}^{2}  \right) $ to its features. In order to observe the difference in test accuracy between important clients and unimportant clients,  we set $\bar \sigma$ to  10 for \mnist, 1 for \cifar and \nus, and 3 for \modelnet.

In the experiments of \textit{client denoising}, we construct one noisy client (i.e., client 7, 5, 2, 3 for \mnist, 
\cifar, \nus , \modelnet respectively) by adding Gaussian noise sampled from   $\mathcal{N}\left(0, {\widetilde \sigma}^{2}  \right) $ to all its training samples and test samples. We set $\widetilde \sigma$ to  1 for \mnist,  \nus and \modelnet, and 3 for \cifar.

\subsubsection{Additional Evaluation Results}

\paragraph{\revise{Effect of $\rho$}}
\revise{Here we report the test accuracy of \ouradmm with different penalty factor $\rho$ in \cref{fig:rho}, which show that \ouradmm is not sensitive to $\rho$ on four datasets.}
{
\renewcommand{\thesubfigure}{\alph{subfigure}}

\begin{figure*}[t]

\newlength{\rhoheightc}
\settoheight{\rhoheightc}{\includegraphics[width=.25\linewidth]{ICLRplots/acc/mnist_rho.pdf}}%

\newcommand{\rowname}[1]%
{\rotatebox{90}{\makebox[\rhoheightc][c]{\tiny #1}}}
\renewcommand{\tabcolsep}{10pt}

\centering
{

\begin{subtable}{1\linewidth}
\centering
\begin{tabular}{c@{}c@{}c@{}c@{}}
         \makecell{{\mnist}}
        & \makecell{{\cifar}}
         & \makecell{{ \nus}}  
         & \makecell{{ \modelnet}}
        \vspace{-3pt}\\
\includegraphics[height=\rhoheightc]{ICLRplots/acc/mnist_rho.pdf}&
\includegraphics[height=\rhoheightc]{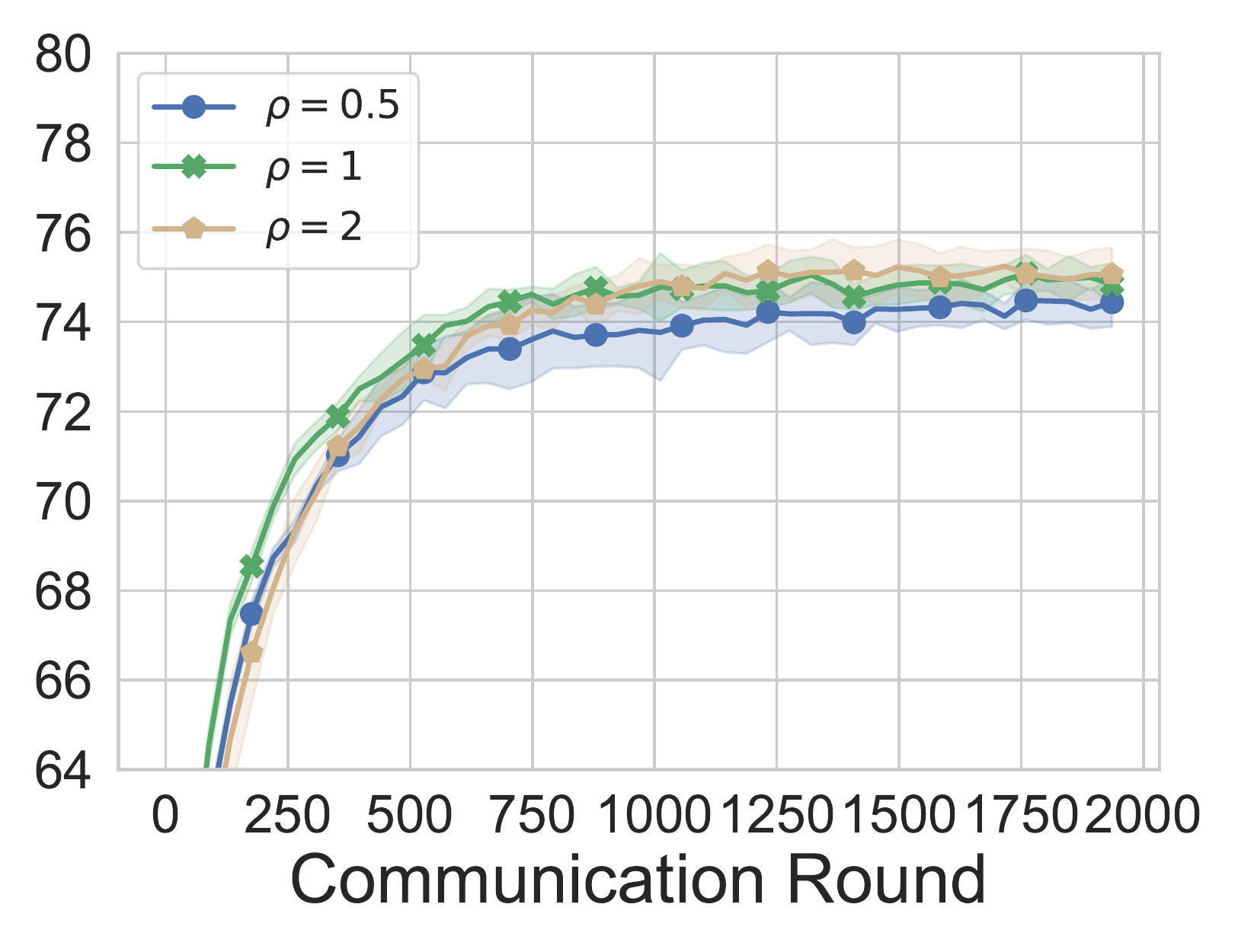}&
\includegraphics[height=\rhoheightc]{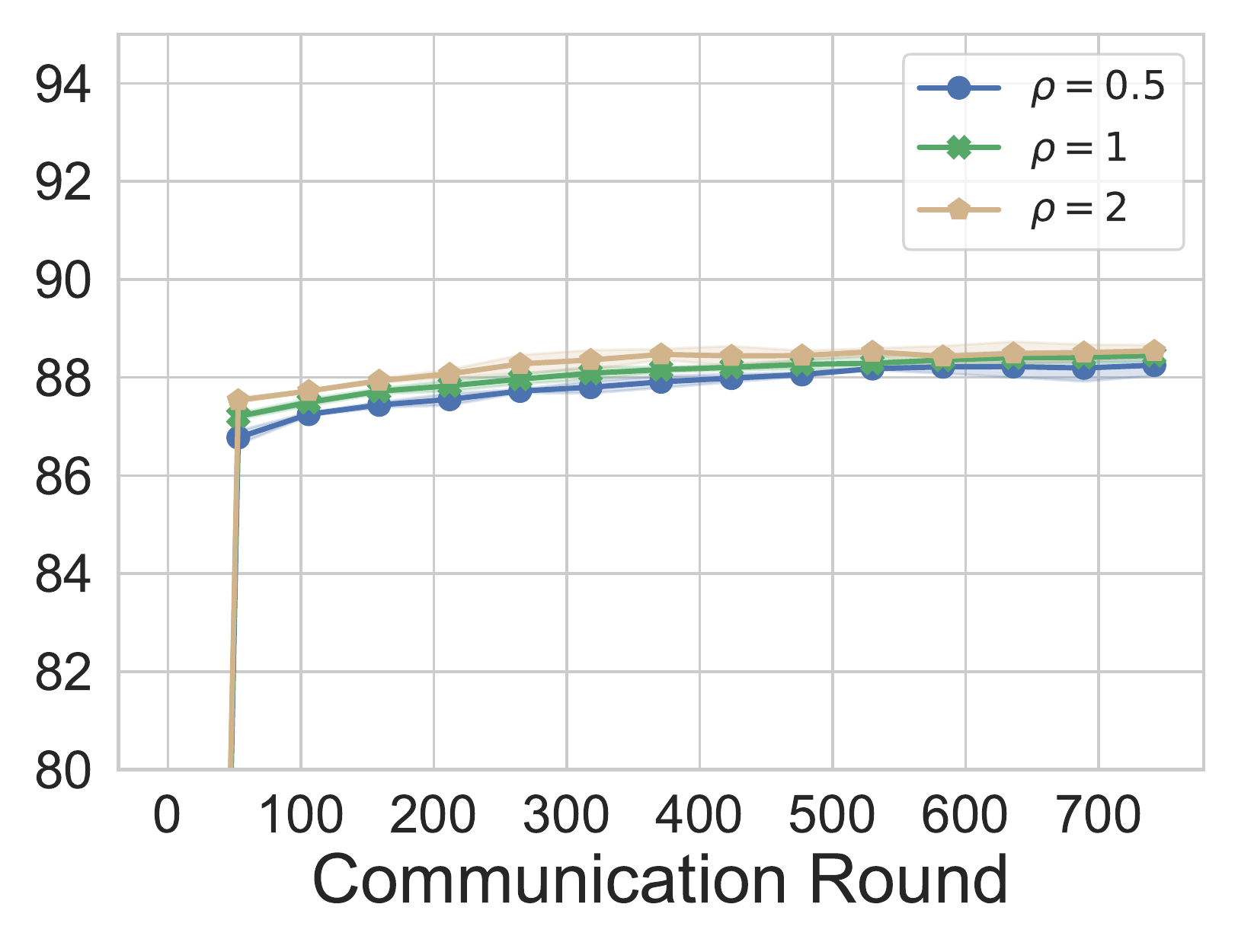}&
\includegraphics[height=\rhoheightc]{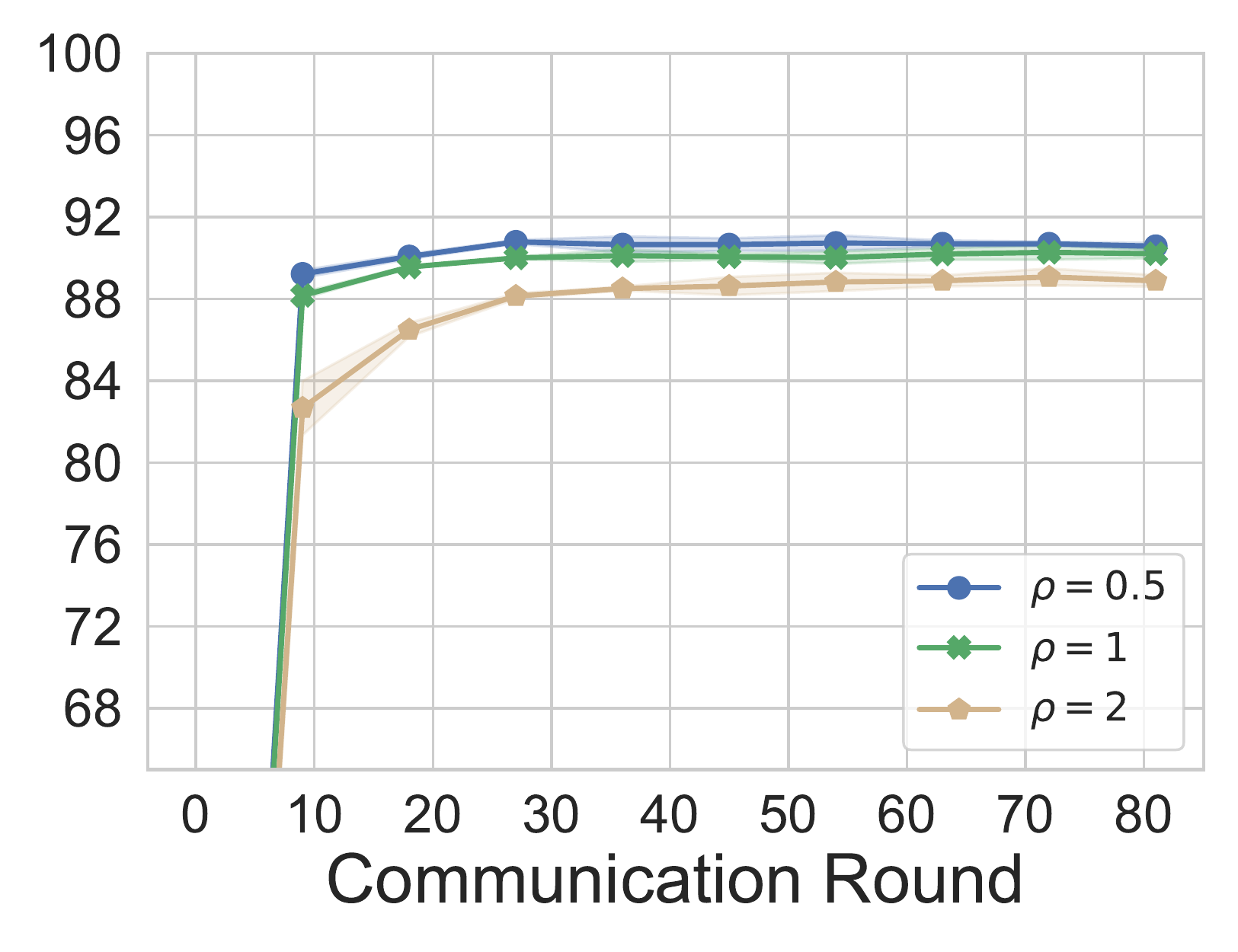}\\[-1.ex]
\end{tabular}
\end{subtable}
}
\vspace{-2mm}
\caption{Performance of \ouradmm with different penalty factor $\rho$ on four datasets. \ouradmm is not sensitive to $\rho$ from 0.5 to 2. }%
\label{fig:rho}
\vspace{-7mm}
\end{figure*}

}

\paragraph{Results for a large number of clients.}
We evaluate baselines and our methods under 100 clients on  \mnist by allowing the agents to obtain overlapped features, and the results show that our methods still outperform baselines. Specifically, we divide the features into 100 overlapped subsets for 100 clients so that each client has 14 pixels.
The results in Table~\ref{tab:100cli} show that VIM methods (i.e.,  \ouradmm, \ouradmmjoint) have higher accuracy than baselines in both w/ and w/o model splitting settings.

{
\begin{table}[h]
  \captionof{table}{Performance of Vanilla VFL when $M=100$ on \mnist}
   \label{tab:100cli}
\centering
 \resizebox{0.5\columnwidth}{!}{%
 \begin{tabular}{ccccccccccc}
\toprule
 \multicolumn{3}{c}{{W/ model splitting}}  &    \multicolumn{2}{c}{{W/o model splitting}} \\ \cmidrule(lr){1-3} \cmidrule(lr){4-5}
\vafl & \oursgd & \ouradmm & \fdml & \ouradmmjoint \\  
 95.38 & 95.45 & \textbf{95.77} &  95.85 & \textbf{95.96}\\ \bottomrule
\end{tabular}%
}
\end{table}

}

\paragraph{Results for client subsampling under \revise{client}-level DP}

We extended our study under \revise{client}-level DP by incorporating a subsampling mechanism into the \ouradmm framework to save the privacy budget. Specifically, during each communication round, the server receives the local embeddings from $p\%$ clients, corresponding to a participation rate of $p\%$. To address missing local embeddings, the server leverages historical local embeddings from other clients to complete their absent local embeddings.

We calculate the privacy budget $\epsilon_i$  for client following our \cref{theo:dp_guarantee}, where $T$  denotes the number of communication rounds that client  $i$
 uploads local embeddings, instead of the total number of communication rounds as in our original algorithm. 
 Due to the non-overlapping nature of local data among clients, the concatenated output matrix from all clients satisfies the $\max_i \epsilon_i$ \revise{client}-level DP guarantee according to DP parallel composition.

As shown in \cref{tab:mask_userdp},  there is a utility loss when $p\%=25\%,50\%$ compared to $p\%=100\%$ under DP and non-DP settings on  MNIST and CIFAR. This discrepancy demonstrates the necessity of aggregating local outputs from all clients during training to achieve optimal utility in vertical federated learning.

\begin{table}[h]
  \centering
   \caption{The utility of \ouradmm with different client subsampling ratio under \revise{client}-level DP.}
  \label{tab:mask_userdp}
\resizebox{0.5\linewidth}{!}{%
\begin{tabular}{ccccccccccccccc}\toprule
& \multicolumn{3}{c}{\mnist} & \multicolumn{3}{c}{\cifar}   \\\cmidrule(lr){2-4} \cmidrule(lr){5-7}  
& 25\%&50\%&100\% &25\%&50\%& 100\%\\\midrule
  $\epsilon=\infty$ &94.20&94.52&97.13 &65.78&66.10&75.25\\
  $\epsilon=8$ &85.98&87.22&92.35 &62.01&62.86&73.83\\
 $\epsilon=1$ &85.51&86.62&91.09&46.53&46.69&61.65
  \\\bottomrule  
\end{tabular}
}
\end{table}

\paragraph{Additional results on client denoising}
\cref{tab:denoise} presents the test accuracy of \vafl, \oursgd, and \ouradmm at different epochs (communication rounds)  on different datasets under one noisy client. Note that each epoch consists of $N/b$ communication rounds.
\cref{tab:denoise} shows that under the noisy training scenario, \ouradmm consistently outperform \oursgd  and \vafl with faster convergence and higher test accuracy, which indicates the effectiveness of \name's multiple linear heads in client denoising. 
{
\setlength{\tabcolsep}{4pt} %

\begin{table}[ht]
    \centering   
    \caption{Test accuracy under one noisy client whose training local features and test local features are perturbed by Gaussian noise.}
    \label{tab:denoise}
\begin{subtable}{\columnwidth}
    \centering
 \resizebox{\columnwidth}{!}{%
    \begin{tabular}{c
  S[table-format=2.2]
  @{\tiny${}\pm{}$}
  >{\tiny}S[table-format=2.2]<{\endcollectcell}
  S[table-format=2.2]
  @{\tiny${}\pm{}$}
  >{\tiny}S[table-format=2.2]<{\endcollectcell}
  S[table-format=2.2]
  @{\tiny${}\pm{}$}
  >{\tiny}S[table-format=2.2]<{\endcollectcell}
  S[table-format=2.2]
  @{\tiny${}\pm{}$}
  >{\tiny}S[table-format=2.2]<{\endcollectcell}
  S[table-format=2.2]
  @{\tiny${}\pm{}$}
  >{\tiny}S[table-format=2.2]<{\endcollectcell}
  S[table-format=2.2]
  @{\tiny${}\pm{}$}
  >{\tiny}S[table-format=2.2]<{\endcollectcell}
   S[table-format=2.2]
  @{\tiny${}\pm{}$}
  >{\tiny}S[table-format=2.2]<{\endcollectcell}
  S[table-format=2.2]
  @{\tiny${}\pm{}$}
  >{\tiny}S[table-format=2.2]<{\endcollectcell}
  S[table-format=2.2]
  @{\tiny${}\pm{}$}
  >{\tiny}S[table-format=2.2]<{\endcollectcell}
  S[table-format=2.2]
  @{\tiny${}\pm{}$}
  >{\tiny}S[table-format=2.2]<{\endcollectcell}
  S[table-format=2.2]
  @{\tiny${}\pm{}$}
  >{\tiny}S[table-format=2.2]<{\endcollectcell}
  S[table-format=2.2]
  @{\tiny${}\pm{}$}
  >{\tiny}S[table-format=2.2]<{\endcollectcell}
}
    \toprule
\multirow{3}{*}{Method}   & \multicolumn{24}{c}{Test accuracy @ epoch (communication round)}\\
    \cmidrule(lr){2-25}
& \multicolumn{6}{c}{\mnist}  & \multicolumn{6}{c}{\cifar} & \multicolumn{6}{c}{\nus} & \multicolumn{6}{c}{\modelnet} \\ \cmidrule(lr){2-7} \cmidrule(lr){8-13} \cmidrule(lr){14-19}  \cmidrule(lr){20-25}
& \multicolumn{2}{c}{2 (106)}  & \multicolumn{2}{c}{5 (265)}  & \multicolumn{2 }{c}{10 (530)}          
& \multicolumn{2}{c}{2 (88)}  & \multicolumn{2}{c}{5 (220)}  & \multicolumn{2 }{c}{10 (440)}          
& \multicolumn{2}{c}{2 (106)}  & \multicolumn{2}{c}{5 (265)}  & \multicolumn{2 }{c}{10 (530)}  
& \multicolumn{2}{c}{2 (18)}  & \multicolumn{2}{c}{5 (45)}  & \multicolumn{2 }{c}{10 (90)}    \\\midrule
{\footnotesize \multirow{1}{*}{\vafl}}
& 91.07 & 0.17 & 94.36 & 0.16 & 95.59 & 0.11 & 28.83 & 1.04 & 38.77 & 0.39 & 46.98 & 0.70 & 51.88 & 0.72 & 77.68 & 0.74 & 85.31 & 0.15 & 43.23 & 3.07 & 80.13 & 1.10 & 89.56 & 0.41 
\\ [0.05em]\midrule[0.05em]
{\footnotesize \multirow{1}{*}{\oursgd}}
& 95.04 & 0.14 & 96.01 & 0.03 & 96.43 & 0.08 & 42.75 & 0.13 & 50.06 & 0.18 & 55.53 & 0.37 & 85.35 & 0.24 & 86.42 & 0.24 & 87.14 & 0.29 & 77.94 & 1.00 & 88.74 & 0.07 & 89.69 & 0.42 
\\  [0.05em]\midrule[0.05em]
{\footnotesize \multirow{1}{*}{\ouradmm}}
& 96.22 & 0.07 & 96.60 & 0.04 & 96.82 & 0.07 & 67.08 & 0.43 & 70.70 & 0.34 & 71.76 & 0.14 & 86.38 & 0.20 & 87.00 & 0.27 & 87.18 & 0.14 & 90.05 & 0.38 & 90.71 & 0.31 & 90.59 & 0.05 
 \\ \bottomrule
\end{tabular}%
}
\end{subtable}
\end{table}
}

\paragraph{Reference accuracy for SOTA model and reference model in the centralized setting}

we included the comparisons in \cref{tab:reference_acc}, which shows the reference accuracy for SoTA models and a simple reference model in a centralized setting on four datasets. Specifically,
\begin{itemize}
    \item \textit{SoTA models}: These models may employ different model architectures and training methods compared to our approach. They serve as a virtual upper bound as the highest achievable accuracy on each dataset.
\item \textit{Reference model in the centralized setting}. This reference model has the same model size as one local model coupled with a server model. 
\end{itemize}

For instance, on the MNIST dataset, the latest SoTA method in a centralized setting achieves an accuracy of 99.87\%, while a basic reference model (comprising one local model followed by a server model) reaches 98.19\%. In comparison, our VFL model (consisting of M local models followed by a server model) demonstrates a comparable accuracy of 97.13\%.

Furthermore, on datasets such as NUS-WIDE and ModelNet, our VFL model even surpasses the accuracy of the reference model in a centralized setting. This is attributable to the significantly higher number of model parameters in VFL. For example, in ModelNet, we utilized four ResNet-18 feature extractors as local models for four different clients, allowing for a more nuanced understanding and representation of the data.

\begin{table}[h]
  \centering
   \caption{{Accuracy for SOTA model and reference model in the centralized setting.}}
  \label{tab:reference_acc}
\resizebox{\linewidth}{!}{%
\begin{tabular}{ccccccccccccccc}\toprule
& \mnist & \cifar & \nus (5 classes) & \modelnet (2D multi-views) \\\midrule
SOTA method in centralized setting (virtual upper bound) &   99.87~\cite{byerly2021no} & 99.50~\cite{dosovitskiy2020image} & 88.7~\cite{liu2020backdoor}& 96.6~\cite{wang2022multi}\\
Reference model  in  centralized setting (e.g., one local model followed by server model)  & 98.19
& 77.61
&87.71
&88.96\\
\ouradmm Model 
(e.g., $M$ local models followed by server model)  
&97.13
&75.25 
&88.51
&91.32
  \\\bottomrule  
\end{tabular}
}
\end{table}

\paragraph{\ouradmm on larger models}

\ouradmm can scale well to large model such as ResNet-18 as shown in the experiments on ModelNet40. Leveraging the larger models as feature extractors for clients, \ouradmm can produce higher-quality local embeddings, which are also crucial for learning accurate linear heads on the server side.  Here we also report the results of \ouradmm on CIFAR with CNN, ResNet-18, and ResNet-34, which are 
75.25\%,
81.35\%,
and 82.58\%, respectively. 
It shows that a larger model can lead to higher accuracy for \ouradmm, validating its scalability and efficiency.

\paragraph{\ouradmm on non-image tasks}

\ouradmm  can be adapted to non-image tasks, such as datasets with both text and image modalities. For example, in the NUS-WIDE dataset, which encompasses both text and image features as local datasets, \ouradmm achieves state-of-the-art results as shown in \cref{fig:vanilla_vfl}. This adaptability is due to \ouradmm's flexible design, which can handle heterogeneous input data types via different feature extractors (e.g., local models) in the clients, and then aggregate heterogeneous local embeddings via multiple linear heads in the server. We believe these results underscore \ouradmm's potential in a broader range of applications beyond image tasks.

\paragraph{\ouradmm under long-tail datasets}

Long-tail datasets are characterized by a significant imbalance, where minority classes have far fewer samples than majority ones. \textit{This horizontal imbalance is distinct from the challenges addressed by vertical federated learning, where the same sample (whether it belongs to a majority or minority class)  is vertically split across multiple clients.} This means that \textit{in vertical federated learning, minority class samples are still be evenly distributed among clients}. We conduct additional experiments on long-tail data.

We create long-tail training datasets following~\cite{vigneswaran2021feature} with an imbalance factor of 10 (i.e., the ratio of samples in the head to tail class). Specifically, for MNIST, this resulted in class sample sizes of [6000, 4645, 3596, 2784, 2156, 1669, 1292, 1000, 774, 600] for 10 classes. For CIFAR, the class sample sizes are [5000, 3871, 2997, 2320, 1796, 1391, 1077, 834, 645, 500] across the 10 classes.

We compared the \ouradmm model, which consists $M$ local models followed by a server model, with a reference model in a centralized setting. This reference model has the same model size as one local model coupled with a server model. 

We show the results in \cref{tab:longtail}, demonstrating that our \ouradmm is still effective on challenging long-tail training datasets, yielding results comparable to those of the reference model in a centralized setting. Moreover, the long-tail version of MNIST dataset does not significantly impact the accuracy of \ouradmm compared to the original MNIST dataset.

\subsection{Discussion}\label{app:discussion}

\paragraph{DP generative model in VFL}

    An alternative way to achieve DP in VFL is to locally train a DP generative model for each client, and then send the DP generative models to the server, which takes only one communication round. However, we identified several key challenges that make it less suitable for our VFL context:
\begin{itemize}
    \item \textit{Mismatch of Synthetic Partial Features Across Clients.} In VFL, there are $M$ clients holding different subsets of features for the same training samples (denoted as  $x_j^1, x_j^2, \ldots, x_j^M$ for sample $j$).  A generative model, due to its stochastic nature, would generate synthetic partial features without correspondence to a specific training sample (e.g., the parietal features generated from a local generative model would adhere to the local data distribution, but do not correspond to a particular original sample $j$).  This lack of correspondence means that the server cannot effectively concatenate the synthetic partial features into a cohesive ``global'' dataset for training. This is a limitation of the generative model in VFL.
   \item \textit{Quality Concerns Due to Partial Features.} Given that each client only has partial features (see \cref{fig:histogram} row 1 for the visualization of raw local features), a generative model trained locally (without FL) might yield lower-quality data. This is particularly problematic in cases where partial features are not informative (e.g., clients having only background pixels in image datasets like \mnist/\cifar). 
    The state-of-the-art accuracy for synthetic data in centralized learning with sample-level DP on \mnist is around 97.6\% under $\epsilon =1$ and  98.2\% under  $\epsilon =10$~\cite{hu2023sok}. Since the partial features in VFL are less informative than the full features in centralized learning, the DP genetive model in VFL would lead to lower accuracy. On the other hand, our \ouradmm DP learning algorithm already achieves a promising accuracy that is close to the state-of-the-art:  91.35\% under $\epsilon =1$ and 92.35\% under $\epsilon =8$  in VFL under \revise{client}-level DP, which is a stricter privacy notion than sample-level DP. 
    \item \textit{Scalability Issues of DP Generative Model with High-Dimensional Data.}
DP generative models often struggle with high-dimensional datasets~\cite{hu2023sok,wang2021datalens}. For instance, their performance on datasets like CIFAR10 is limited~\cite{wang2021datalens}, posing a challenge for more complex datasets like ModelNet40. Additionally, the generation of multi-modal data (e.g., text and image features in NUS-WIDE) remains an unresolved challenge. In our method, as we are not training generative models, the high dimensionality of data will not pose a significant challenge to \ouradmm.  
   \item \textit{Communication Overhead with Generative Model. }
The model size of a standard DCGAN~\cite{radford2015unsupervised} implemented in PyTorch \footnote{\href{https://pytorch.org/tutorials/beginner/dcgan_faces_tutorial.html}{https://pytorch.org/tutorials/beginner/dcgan\_faces\_tutorial.html}} is 13.65 MB. As it only takes one round for communication, the communication costs for each client would be 13.65 MB.  In comparison, \ouradmm demonstrates similar communication costs on certain datasets. For example, on the ModelNet40 dataset, \ouradmm achieves an accuracy of 89\% with a total communication cost of 11.32 MB (See \cref{fig:vanilla_vfl} and \cref{tab:communication} for details). This efficiency stems from the transmission of local embeddings and ADMM-related variables, which collectively have a smaller size than the number of parameters within a deep neural network like a DCGAN generator.  This trend becomes even more evident if we use a larger generative model than DCGAN, which is a common direction in current generative AI advancements.  Moreover, the high-quality local embeddings (e.g., from a pretrained ResNet-18 as local model) and multiple local updates at each communication round (enabled by ADMM) significantly aid convergence. Consequently, a relatively small number of communication rounds (approximately 10) is required to reach an accuracy of 89\% on ModelNet40.

We remain open to future research exploring the feasibility and optimization of the local training of DP generative models in VFL settings.

\end{itemize}

\paragraph{Linear head and client importance}

In \cref{sec:exp_explainability}, we utilize the norm of weights in the linear head, learned by the server from local embeddings, to determine the importance of the corresponding client (and its local features).
Our approach is in line with existing methods such as LIME~\cite{ribeiro2016should}, SHAP~\cite{lundberg2017unified}, and others~\cite{gohel2021explainable} that utilize model weights to determine feature importance.  In our model, we follow existing work by assuming feature independence ~\cite{ribeiro2016should,lundberg2017unified,gohel2021explainable} to simplify the interpretation of weights in terms of feature importance.

\paragraph{Challenges of ADMM algorithm design in VFL}
There are several key challenges of designing ADMM algorithm in VFL for distributed optimization: 
\begin{itemize}
    \item   how to ensure the consensus among clients and form it as a constrained optimization problem (e.g., from Eq.~\ref{eq:multihead-obj} to Eq.~\ref{eq:admm-obj}).
    \item  how to decompose the optimization problem into small sub-problems that can be solved in parallel by ADMM.
\end{itemize}

For the first challenge, although ADMM is flexible to introduce auxiliary variables and thus formulate a constrained optimization problem in HFL, it raises new challenges in VFL.
For example, the ADMM-based methods in HFL~\cite{elgabli2020gadmm, elgabli2020fgadmm, huang2019dp, yue2021inexact} usually use the global model as the auxiliary variable and enforce the consistency between the global model and each local model. However, VFL communicates embeddings, and it is not feasible to enforce local embeddings from different clients to be the same as they provide unique information from different aspects.
Therefore, in this paper, we introduce the auxiliary variable $z_j$ for each sample $j$ and construct the constraint between $z_j$ and server's output $ \sum_{k=1}^M h_{j}^k W_{k}$  (i.e., the logits), which enables the optimization for each $W_{k}$  by ADMM.

For the second challenge, we propose the bi-level optimization for server's model and clients' models to train DNNs for VFL with model splitting, while the existing ADMM-based method in VFL~\cite{hu2019learning} only considers logistic regression with linear models in client-side, which does not apply to DNNs.
The initial attempt we made is to decompose the optimization for server's linear heads by ADMM while still using chain rule of SGD to update local models, which does not exhibit much superiority over pure SGD-based methods. Later, we decompose the optimization for both server's linear heads and local models by ADMM, leading to our current algorithm \ouradmm that enables multiple local updates for clients at each communication round and achieves significantly better performance, as we show in Sec.~\ref{sec:evaluation}.

\paragraph{Limitations}
Directly deploying VFL algorithms without stopping criteria or regularization techniques may lead to the over-fitting phenomenon, as in many other algorithms. 
Based on our experiments, we find that over-fitting is a common problem of VFL algorithms due to a large number of model parameters from all clients in the whole VFL system. Compared to centralized learning or horizontal FL, the prediction for one data sample in VFL involves $M$ times model parameters, which corresponds to $M$ partitions of input features. 
To prevent over-fitting, we use regularizers to constrain the complexity of models and  adopt standard stopping criteria, i.e., stop training when the model converges or the validation accuracy starts to drop more than $2\%$.

\end{document}